\documentclass[10pt,twocolumn,letterpaper]{article}
\usepackage[top=0.7in,bottom=0.7in,left=0.6in,right=0.7in,columnsep=0.35in]{geometry}

\usepackage{microtype}
\usepackage{graphicx}
\usepackage{subfigure}
\usepackage{booktabs}

\usepackage{hyperref}

\usepackage[round]{natbib}
\usepackage[font=small]{caption}

\usepackage{amsfonts}
\usepackage{amsthm}
\usepackage{amsmath}

\newtheorem{theorem}{Theorem}
\newtheorem{corollary}{Corollary}
\newtheorem{lemma}{Lemma}
\newtheorem{remark}{Remark}
\newtheorem{fact}{Fact}

\usepackage{color}

\title{On the Fundamental Limits of Exact Inference in Structured Prediction}
\author{
    Hanbyul Lee\\
    Department of Statistics\\
    Purdue University\\
    West Lafayette, IN 47906, USA\\
    {\tt lee3078@purdue.edu}
  \and
    Kevin Bello\\
    Department of Computer Science\\
    Purdue University\\
    West Lafayette, IN 47906, USA\\
    {\tt kbellome@purdue.edu}
  \and
    Jean Honorio\\
    Department of Computer Science\\
    Purdue University\\
    West Lafayette, IN 47906, USA\\
    {\tt jhonorio@purdue.edu}
}
\date{}

\begin{document}

\maketitle

\begin{abstract}
Inference is a main task in structured prediction and it is naturally modeled with a graph.
In the context of Markov random fields, noisy observations corresponding to nodes and edges are usually involved, and the goal of exact inference is to recover the unknown true label for each node precisely.
The focus of this paper is on the fundamental limits of exact recovery irrespective of computational efficiency, assuming the generative process proposed by \citet{globerson2015hard}.
We derive the necessary condition for any algorithm and the sufficient condition for maximum likelihood estimation to achieve exact recovery with high probability,
and reveal that the sufficient and necessary conditions are tight up to a logarithmic factor for a wide range of graphs.
Finally, we show that there exists a gap between the fundamental limits and the performance of the computationally tractable method of \citet{bello2019exact}, which implies the need for further development of algorithms for exact inference.
\end{abstract}

\section{Introduction}
\label{sec:introduction}

Structured prediction, which is a supervised machine learning that involves structured objects such as sequences and trees, has been utilized in a wide range of domains including sociology, computer vision, natural language processing, and bioinformatics.
Examples of the structured prediction problem are community detection, part-of-speech tagging, protein folding, and image segmentation.
In various scenarios, one wishes to learn a model representing the interrelation of predicted variables implied by the structure, or is interested on inference after a model has been already learned.
Classical methods for learning include conditional random fields \citep{lafferty2001conditional} and structured support vector machines \citep{taskar2004max, tsochantaridis2005large}.
In the present paper, we focus on the inference problem.

A great deal of work has addressed the statistical inference over graphs in recent years \citep{chandrasekaran2008complexity, fortunato2010community, abbe2014decoding}.
Typical inference problems in the context of Markov random fields (MRFs) aim to infer the unknown true label corresponding to each node of a given graph, where noisy observations are provided for the labels assigned to edges, or to both edges and nodes. 
One concrete example is to recover individual opinions in a social network, where one receives noisy measurements of whether two connected individuals have the same opinion or not, and obtains noisy estimates of individual opinions \citep{foster2018inference}.
Assuming a simple generative model suggested in \citet{globerson2015hard},
we consider two regimes, one in which only noisy edge observations are given, and one regime in which noisy edge and node observations are collectively provided.

To solve the above inference problems, computationally efficient algorithms have been studied over the past few years.
\citet{globerson2015hard} presented tight upper and lower bounds of minimum-achievable Hamming error which can be attained by a polynomial-time algorithm for 2D grid graphs.
\citet{foster2018inference} also developed a polynomial-time solvable algorithm which is based on tree decompositions and can be applied to more general graphs.
While the aforementioned works focused on approximate inference, \citet{bello2019exact} studied the sufficient conditions for exact inference in polynomial time and provided high probability results for general families of graphs.
However, the current research works are mainly motivated by computational considerations,
and little analysis has been devoted to the statistical complexity of the inference problem irrespective of computational efficiency.

Analyzing the information-theoretic limits associated with the performance of any algorithm is crucial to understand the statistical complexity of the inference problem, as discussed in different contexts \citep{chen2014statistical, banks2016information, abbe2017community}.
In particular, establishing the information-theoretic lower and upper bounds is instrumental in the development of algorithms.
If an existent method is computationally tractable and achieves the fundamental limits, then there is little point in suggesting a new algorithm,
and if there currently exists a gap between the performance of computationally tractable methods and the fundamental limits, then this situation encourages further elaboration of algorithms.

With this motivation, we develop the fundamental limits of the aforementioned exact inference problems,
and compare the limit bounds to those of the currently existent polynomial-time algorithm for exact inference \citep{bello2019exact}.
The information-theoretic limits of a similar problem were studied in \citet{chen2016information}, where only a pairwise difference measurement corresponding to each edge was assumed to be given.
In contrast to \citet{chen2016information}, we consider the node estimate as well as the edge measurement.

The main contribution of this paper consists of providing the necessary condition for any algorithm and the sufficient condition for the optimal strategy (the MLE algorithm) to exactly recover the true unknown labels with probability $1-o(n)$, where $n$ is the number of nodes of a given graph.
We find that the conditions involve multiple graphical factors such as the number of edges, the maximum degree of the graph, and the Cheeger constant.
Our results apply to general graphs, and especially to complete graphs, regular expanders and star graphs, we show that the sufficient and necessary conditions are tight up to a logarithmic factor.
Furthermore, we reveal that the error bound of the MLE algorithm decays much faster than that of the polynomial-time algorithm in \citet{bello2019exact},
eliciting a gap between the optimal method and the currently existent tractable algorithm.

The remainder of this paper is organized as follows. 
In Section \ref{sec:preliminaries}, we describe the formal problem setup and introduce terminology and notation.
We develop sufficient and necessary conditions for exact recovery where only edge observations are given in Section \ref{sec:case_that_only_edge}, and where edge and node observations are both given in Section \ref{sec:case_that_edge_and_node}. 
In both sections, we provide an illustration of the fundamental limit bounds for a few examples of graphs.
Section \ref{sec:concluding_remarks} concludes the paper with a summary of our findings and a discussion of future works. The proofs of the theorems are deferred to the appendices.

\section{Preliminaries}
\label{sec:preliminaries}

We first introduce the inference problem, graph terminology and notation used throughout the paper.

\subsection{Exact Inference Problem}
\label{subsec:exact_inference_problem}

We assume that there is a known undirected connected graph $\mathcal{G}=(\mathcal{V}, \mathcal{E})$ with $n$ nodes, where $\mathcal{V}=\{1,\dots,n\}$. 
Each node has an unknown true label $y_i^* \in \{-1, 1\}$, $i\in\mathcal{V}$, and we denote the true label vector by $\pmb{y}^* = (y_1^*, \cdots, y_n^*)^\top$. 
We suppose that nature picks $\pmb{y}^*$ from a uniform distribution with support $\mathcal{Y} = \{-1, 1\}^n$, $|\mathcal{Y}| = 2^n$.
A set of noisy observations $X\in \mathcal{X}$ and $\pmb{c} \in \mathcal{C} = \{-1, 1\}^n$, where $X$ and $\pmb{c}$ correspond to edges and nodes respectively, is assumed to be generated from $\pmb{y}^*$ by the following process:
\begin{gather*}
\mathbb{P}\big(X_{ij}=y_i^*y_j^*\big|\pmb{y}^*\big) 
= 1-\mathbb{P}\big(X_{ij}=-y_i^*y_j^*\big|\pmb{y}^*\big) = 1-p 
\\
\mathbb{P}\big(c_{k}=y_k^*\big|\pmb{y}^*\big) 
= 1-\mathbb{P}\big(c_{k}=-y_k^*\big|\pmb{y}^*\big) = 1-q
\end{gather*}
for $(i, j)\in \mathcal{E}$ and $k\in\mathcal{V}$.
The parameters $p$ and $q$ are fixed and between $0$ and $\frac{1}{2}$.
Note that we set $X$ as an $n\times n$ upper triangular matrix with $X_{ij}=0$ for $(i,j) \notin \mathcal{E}$, so that the cardinality of $\mathcal{X}$ is $|\mathcal{X}| = 2^{|\mathcal{E}|}$.
We consider two regimes where (1) only noisy edge observation $X$ is given and (2) noisy edge and node observations $X$ and $\pmb{c}$ are both given.
The goal is to exactly recover the true label vector $\pmb{y}^*$ (i.e., with zero Hamming error) in each regime.

\subsection{Graph Terminology}
\label{subsec:graph_terminology}

We denote the {\it degree of $i$th node} by $\Delta_i$, and the {\it maximum degree of nodes} by $\Delta_{\max}:=\underset{i\in\mathcal{V}}{\max}~\Delta_i$.
For any subset $S\subseteq\mathcal{V}$, 
$S^c$ indicates its complement and
we let $\mathcal{E}(S, S^c)$ denote the collection of all edges going from a node in $S$ to a node outside of $S$, i.e., $\mathcal{E}(S, S^c) := \{(i,j)\in\mathcal{E} : i\in S, j\in S^c ~\text{or}~ i\in S^c, j\in S\}$. 
$|\mathcal{E}(S, S^c)|$ represents the number of edges between $S$ and $S^c$.
The {\it Cheeger constant} of a graph $\mathcal{G}$ is defined as
$
\phi_{\mathcal{G}} := \underset{S\subseteq \mathcal{V}, 1\leq|S|\leq {_\lfloor}\frac{n}{2}{_\rfloor}}{\min} \frac{|\mathcal{E}(S, S^c)|}{|S|}
$, which is also called {\it edge expansion}.

In what follows, we introduce several classes of graphs which are taken as examples in this paper.
\begin{enumerate}
    \item {\it Complete graph}: A graph is said to be complete if every pair of distinct nodes is connected by an edge.
    \item {\it Chain graph}: A chain graph is a sequence of nodes connected by edges. The length of a chain is the number of edges, which is the number of nodes minus one.
    \item {\it Star graph}: A star graph has one internal node which is connected to all the others. The other nodes are not connected to each other.
    \item {\it Regular expander}: A $d$-regular expander with constant $C>0$ is a graph whose nodes have the same degree $d$, and which satisfies that for every subset $S\subseteq\mathcal{V}$ with $|S|\leq \frac{n}{2}$, $|\mathcal{E}(S, S^c)| \geq C d |S|$.
\end{enumerate}
A chain graph and a star graph are examples of \emph{tree graphs}, in which any two nodes are connected by exactly one path of edges. 
Note that a graph is a tree graph if and only if a graph has exactly $n-1$ edges. 
Key graphical metrics of the above example graphs are summarized in Table \ref{tab:graph_examples}.

\begin{table}[t]
\caption{Graphical Metrics of Example Graphs.}
\label{tab:graph_examples}
\vskip 0.15in
\begin{center}
\begin{small}
\begin{sc}
\begin{tabular}{lccc}
\toprule
Type of Graph & $|\mathcal{E}|$ & $\Delta_{\max}$ & $\phi_{\mathcal{G}}$\\
\midrule
Complete graph& $\binom{n}{2}$ & $n-1$ & $n/2$ \\
Chain graph& $n-1$ & $2$ & $2/n$ \\
Star graph& $n-1$ & $n-1$ & $1$ \\
$d$-regular expander & $nd/2$ & $d$ & $\geq Cd$ \\ 
\bottomrule
\end{tabular}
\end{sc}
\end{small}
\end{center}
\vskip -0.1in
\end{table}

\subsection{Other Notation}

The number $e$ and $\log(\cdot)$ indicate Euler's number and the natural logarithm.
The function $H^*(p):= -p\log p - (1-p)\log(1-p)$ stands for the binary entropy function.
$\mathbb{E}_W[\cdot]$ represents the expectation with respect to a random variable $W$.
The indicator function is denoted by $\mathbb{I}[\cdot]$.
We denote by $a\wedge b$ and $a \vee b$ the minimum and maximum between $a$ and $b$, respectively, where $a$ and $b$ are two scalars.
We use $C$, $C'$, $C'', \dots$ to denote universal constants independent of ($n, p, q, |\mathcal{E}|, \phi_{\mathcal{G}}$).
The notation $f(n) = o(g(n))$ means $\lim_{n\rightarrow\infty}f(n)/g(n) = 0$;
$f(n) = \Omega(g(n))$ means that there exists a constant $C$ such that $f(n)\geq C g(n)$ asymptotically;
$f(n) = O(g(n))$ means that there exists a constant $C$ such that $f(n)\leq C g(n)$ asymptotically;
$f(n) = \Theta(g(n))$ means that there exists constants $C$ and $C'$ such that $C g(n) \leq f(n)\leq C' g(n)$ asymptotically.

\section{Regime I: Edge Observations Only}
\label{sec:case_that_only_edge}

Now we derive the information-theoretic limits of exact recovery where only noisy edge observations are given.
We first introduce the lower bound of the minimax error probability and the lower bound of the probability of success of the maximum likelihood estimator which is a minimax optimal estimator.
From the two results, we derive the necessary and the sufficient conditions of exact recovery regardless of its computational complexity.
Then we illustrate the limit bounds for several examples of graphs.

\subsection{Minimax Lower Bound}
\label{subsec:minimax1}

We first develop the lower bound of the probability for any algorithm to fail to exactly recover the true label vector $\pmb{y}^*$.

\begin{theorem}
\label{theorem1}
Let $\mathcal{G} = (\mathcal{V}, \mathcal{E})$ be an undirected connected graph with $n$ nodes, Cheeger constant $\phi_{\mathcal{G}}$, and maximum node degree $\Delta_{\max}$.
Consider a family of distributions $\mathcal{P}$ over $\mathcal{Y} \times \mathcal{X}$.
Then, for the inference problem in Section \ref{subsec:exact_inference_problem} in which only the noisy edge observation $X\in\mathcal{X}$ is given,
the minimax probability of error is bounded below as follows:
$$
\inf_{ \hat{\pmb{y}} : \mathcal{X} \to \mathcal{Y} }
\sup_{ P \in\mathcal{P} }
\mathbb{P}_{ (\pmb{y}^*, X) \sim P } \big[ \hat{\pmb{y}} (X) \neq \pmb{y}^* \big]
\geq
\max\big\{ f_1, g_1, g^*_1 \big\}
$$
where
\begin{align*}
f_1
&:=
\frac{1}{2} 
\sum_{m=0}^{\Delta_{\max}}
\binom{\Delta_{\max}}{m}
\bigg[\Big\{ (1-p)^m p^{\Delta_{\max}-m} \Big\}
\\&~~~~~~~~~~~~~~~~~~~~~~~~~~~\wedge
\Big\{ p^m (1-p)^{\Delta_{\max}-m} \Big\}\bigg],
\\[0.5em]
g_1
&:=\frac{n-1}{n} - \frac{|\mathcal{E}|}{n}\cdot 
\bigg(1 - \frac{H^*(p)}{\log 2} \bigg),
\\[0.5em]
g^*_1
&:= g_1 +\frac{\big(|\mathcal{E}|\log 2-\kappa_1\big)\vee 0 }{n\log 2}\cdot \mathbb{I}\big[(1-p)^{|\mathcal{E}|} \leq e^{-1}\big].
\end{align*}
Here,
$\kappa_1 := -|\mathcal{E}|\log(1-p) \{2(1-p)\}^{|\mathcal{E}|}  \mathbb{E}_B \big[\frac{\tau(B)}{2^n}\big]
+ \{2(1-p)\}^{|\mathcal{E}|}
\mathbb{E}_B \big[-\frac{\tau(B)}{2^n} \log\big(\frac{\tau(B)}{2^n}\big)\big]$.
In addition, $\tau(B) := 2\big(\frac{p}{1-p}\big)^B
+
\sum_{k=1}^{n-1} 
\binom{n}{k}
\big(\frac{p}{1-p}\big)^{( \phi_{\mathcal{G}} \cdot\{k\wedge (n-k)\}
- B )\vee 0}$
and $B\sim \text{Bin}(|\mathcal{E}|-n+1,~ \frac{1}{2})$.
\end{theorem}

The bound of $f_1$ is derived from Assouad's lemma, and $g_1$ and $g_1^*$ are obtained from Fano's inequality.
When applying Fano's inequality, we use two different strategies to bound the mutual information and then derive $g_1$ and $g_1^*$ from them respectively.
See Appendix \ref{proof_theorem1} for the details.

\begin{figure}[t]
\vskip 0.2in
\begin{center}
\centerline{\includegraphics[width=2.9in]{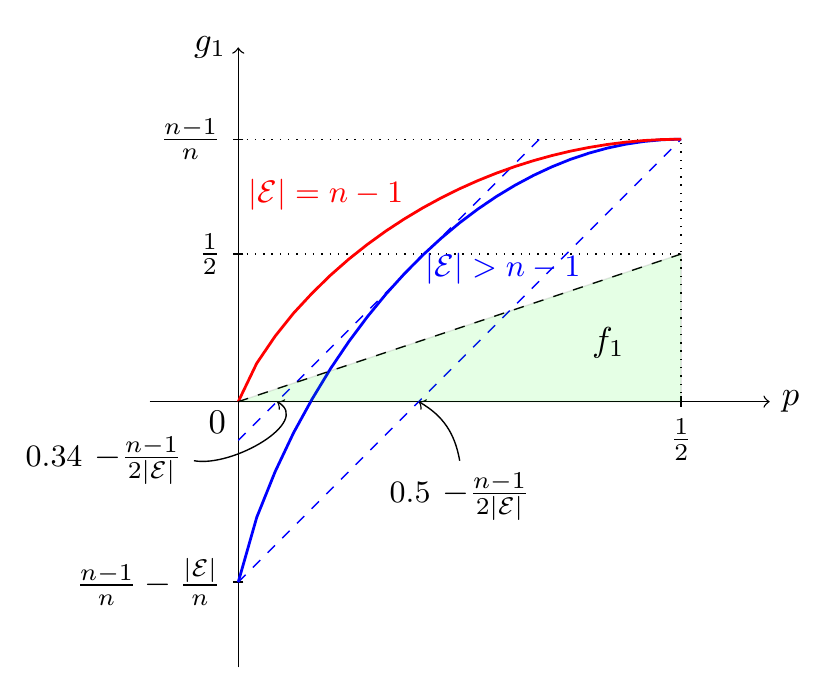}}
\caption{
Graphical illustration of the functions $f_1$ and $g_1$ with respect to the parameter $p\in[0,\frac{1}{2}]$ in Theorem \ref{theorem1}.
The green triangle indicates the area where the function $f_1$ has its function value.
The red and blue lines represent the function $g_1$ when $|\mathcal{E}| = n-1$ and $|\mathcal{E}| > n-1$, respectively.
The function $g_1$ always has a positive value when $|\mathcal{E}| = n-1$,
and when $|\mathcal{E}| > n-1$, $g_1$ has its $y$-intercept at $\frac{n-1}{n}-\frac{|\mathcal{E}|}{n}$ and its $x$-intercept between $0.34 - \frac{n-1}{2|\mathcal{E}|}$ and $0.5 - \frac{n-1}{2|\mathcal{E}|}$.
}
\label{fig:minimax1}
\end{center}
\vskip -0.2in
\end{figure}

From Figure \ref{fig:minimax1}, we can verify which of the functions $f_1$ or $g_1$ achieves the maximum in different situations.
For instance, when $|\mathcal{E}| = n-1$, i.e., the graph $\mathcal{G}$ has a tree structure, the function $g_1$ is always greater than $f_1$.
If $|\mathcal{E}|$ gets larger than $n-1$, then the function $g_1$ has a negative value when $p$ is small, therefore $f_1$ achieves the maximum for small $p$.
As $\frac{n-1}{|\mathcal{E}|}$ decreases, the range of $p$ where $f_1 > g_1$ holds gets larger.

The function $g_1^*$ is always greater than or equal to $g_1$ since the added value to $g_1$ in $g_1^*$ is non-negative.
In most cases, $g_1$ and $g_1^*$ are the same.
Exceptionally, when $|\mathcal{E}| > n-1$ (i.e., $\mathcal{G}$ is not a tree graph), $n$ is small and $p$ is large, $g_1^*$ becomes greater than $g_1$, which can be also verified from the illustration in Section \ref{subsec:illustration1}.

We add a remark about the possibility to improve the limit bound in Theorem \ref{theorem1}.

\begin{remark}
We can observe that $\kappa_1$ decreases as $\phi_{\mathcal{G}}$ increases when $(1-p)^{|\mathcal{E}|} \leq e^{-1}$ holds.
Hence, $g_1^*$ increases as $\phi_{\mathcal{G}}$ increases, which does not match with the intuition that 
exact recovery is easier for a graph $\mathcal{G}$ with larger Cheeger constant.
This implies that there is room for improvement on the bound in Theorem \ref{theorem1}.
\end{remark}

Now, we derive a corollary from Theorem \ref{theorem1}, which induces a necessary condition for any algorithm to exactly recover the true label $\pmb{y}^*$.

\begin{corollary}[Necessary Condition]
\label{corollary1}
Under the same assumptions as in Theorem \ref{theorem1}, if the following condition holds:
\begin{equation}
\label{eq:corollary1}
|\mathcal{E}|
\cdot \bigg(1 - \frac{H^*(p)}{\log 2} \bigg) \leq \frac{n}{2}-1,    
\end{equation}
then 
$\underset{\hat{\pmb{y}}:\mathcal{X} \to \mathcal{Y}}{\inf}
\underset{ P \in\mathcal{P} }{\sup}
\mathbb{P}_{ (\pmb{y}^*, X) \sim P } \big[ \hat{\pmb{y}} (X) \neq \pmb{y}^* \big] \geq \frac{1}{2}$,
that is, any algorithm fails to exactly recover the true node labels with probability greater than half.
\end{corollary}

\begin{proof}
Since $g_1^*$ is always greater than or equal to $g_1$ and $f_1$ is smaller than half,
it is sufficient to derive the condition in which $g_1$ becomes larger than half.
\end{proof}

\begin{figure}[t]
\vskip 0.2in
\begin{center}
\centerline{\includegraphics[width=2.6in]{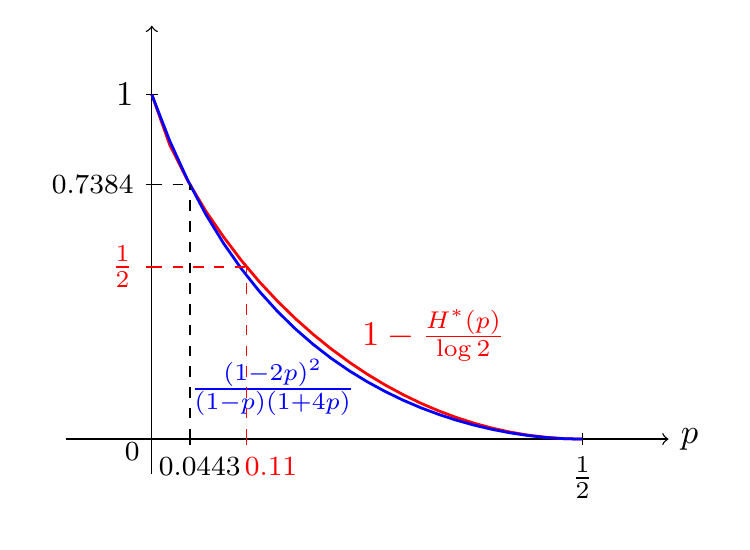}}
\caption{
Graphical illustration of the functions $1 - \frac{H^*(p)}{\log 2}$ (red line) and $\frac{(1-2p)^2}{(1-p) (1 + 4p)}$ (blue line) in Corollary \ref{corollary1} and \ref{corollary2}, respectively. 
Two functions are close to each other on the domain $p\in[0,\frac{1}{2}]$. 
When $p \leq 0.0443$, $\frac{(1-2p)^2}{(1-p) (1 + 4p)}$ is slightly greater than $1 - \frac{H^*(p)}{\log 2}$, and the reverse holds when $p > 0.0443$.
}
\label{fig:match}
\end{center}
\vskip -0.2in
\end{figure}

Note that the opposite of the condition (\ref{eq:corollary1}) provides the necessary condition for exact recovery.
The red line in Figure \ref{fig:match} graphically illustrates the function $1 - \frac{H^*(p)}{\log 2}$, which decreases from $1$ to $0$ as $p$ increases from $0$ to $1/2$.
This implies that it becomes harder to satisfy the necessary condition as $p$ gets close to $1/2$.
On the other hand, we can observe that the necessary condition holds for any graph if $1 - \frac{H^*(p)}{\log 2} > \frac{1}{2}$ (i.e., $p$ is smaller than $0.11$), 
because $|\mathcal{E}|\cdot\big(1 - \frac{H^*(p)}{\log 2} \big)
> \frac{n-1}{2} > \frac{n}{2}-1$.

\subsection{Maximum Likelihood Estimator}
\label{subsec:mle1}

Given the noisy edge observations only, the Maximum Likelihood Estimator (MLE) of the model presented in Section \ref{subsec:exact_inference_problem} is derived as follows:
\begin{equation}
\hat{\pmb{y}}_{mle}(X) := \underset{\pmb{y}\in \mathcal{Y}}{\arg\max}~~~ \pmb{y}^T X\pmb{y}. 
\label{eq:mle1}
\end{equation}
The optimization problem (\ref{eq:mle1}) is known to be NP-hard, so it cannot be used in practice.
However, the MLE method is a minimax-optimal strategy in our problem
\footnote{The Bayes estimator for the zero-one loss is the Maximum A-Posteriori (MAP) estimator. Since we assume that $\pmb{y}^*$ is uniformly distributed, the Bayes estimator is minimax optimal and the MLE is equivalent to the MAP estimator. Hence, the MLE is minimax optimal. See \citet{bickel2015mathematical} for details.}, 
so it serves as a benchmark for tractable algorithms.

We first derive the lower bound of the probability for the MLE algorithm to exactly recover the true label $\pmb{y}^*$ up to a global flip of the labels.

\begin{theorem}
\label{theorem2}
Let $\mathcal{G} = (\mathcal{V}, \mathcal{E})$ be an undirected connected graph with $n$ nodes and Cheeger constant $\phi_{\mathcal{G}}$.
Let $h_1(p, z) := 
\exp \Big[ - \frac{ (1-2p)^2z }{ \frac{4}{3}(1-p) (1 + 4p)} \Big]$.
Then, for the inference problem in Section \ref{subsec:exact_inference_problem} in which only the noisy edge observation $X$ is given,
the MLE algorithm in (\ref{eq:mle1})
returns a solution $\hat{\pmb{y}}_{mle}(X) \in \{\pmb{y}^*, -\pmb{y}^*\}$
with probability at least
$$
1 - \sum_{k=1}^{\lfloor \frac{n}{2}\rfloor} \binom{n}{k} h_1(p, \phi_{\mathcal{G}}k).
$$
\end{theorem}

The Bernstein inequality and the union bound are mainly used in the proof (See Appendix \ref{proof_theorem2}.)
Note that $h_1(p, \phi_{\mathcal{G}}k)$ decreases as $p$ decreases or $\phi_{\mathcal{G}}$ increases,
that is, the lower bound of the probability of success becomes large for small $p$ and large $\phi_{\mathcal{G}}$ (e.g., large $d$ in regular expanders.)
We will illustrate the bounds for several examples in Section \ref{subsec:illustration1}.

\textbf{Comparison to Tractable Algorithm.}
We compare the above lower bound for the MLE algorithm to that of a polynomial-time solvable algorithm introduced in \citet{bello2019exact}.
The upper bound of the error probability of the tractable algorithm is as follows:
\begin{equation}
\label{eq:epsilon1}
\epsilon_1 := 2n\cdot e^{\frac{-3(1-2p)^2 \phi_{\mathcal{G}}^4}{1536\Delta_{\max}^3 p(1-p)+32(1-2p)(1-p)\phi_{\mathcal{G}}^2\Delta_{\max}}}.    
\end{equation}
If $\phi_{\mathcal{G}}=\Omega(n)$ (e.g., complete graph, regular expander with $d=\Omega(n)$), we can derive $\frac{\sum_{k=1}^{ _\lfloor \frac{n}{2} _\rfloor} \binom{n}{k} h_1(p, \phi_{\mathcal{G}}k)}{\epsilon_1} = O\big(\exp\big(-\frac{C(1-2p)^2n}{(1-p)(1+4p)}\big)\big)$ 
for some positive constant $C$ (see Appendix \ref{proof_comparison1}.)
Hence, $\sum_{k=1}^{ _\lfloor \frac{n}{2} _\rfloor} \binom{n}{k} h_1(p, \phi_{\mathcal{G}}k)$ decays much faster than $\epsilon_1$,
and this implies that the currently existent and tractable method does not achieve the fundamental limits.

Next, we derive from Theorem \ref{theorem2} the sufficient condition for the MLE algorithm to exactly recover the true label with high probability.

\begin{corollary}[Sufficient Condition]
\label{corollary2}
Under the same assumptions as in Theorem \ref{theorem2}, if the following condition holds:
\begin{equation}
\label{eq:corollary2}
\phi_{\mathcal{G}}\cdot \frac{(1-2p)^2}{(1-p) (1 + 4p)} \geq \frac{8}{3} \log n,    
\end{equation}
then the optimal solution to the MLE algorithm in (\ref{eq:mle1}) fulfills $\hat{\pmb{y}}_{mle}(X) \in \{\pmb{y}^*, -\pmb{y}^*\}$ with probability $=1-2n^{-1}$.
\end{corollary}

\textbf{Tightness of Sufficient and Necessary Conditions.}
One important question is how tight the necessary and the sufficient conditions derived in Corollary \ref{corollary1} and \ref{corollary2} are. 
In Figure \ref{fig:match}, we observe that the functions $1 - \frac{H^*(p)}{\log 2}$ and $\frac{(1-2p)^2}{(1-p) (1 + 4p)}$ (red and blue lines, respectively) are almost the same.
Hence, the condition (\ref{eq:corollary2}) can be written as
$\phi_{\mathcal{G}}\cdot \big(1 - \frac{H^*(p)}{\log 2}\big) \geq C \log n$
for some constant $C$.
Then, if $|\mathcal{E}|/\phi_{\mathcal{G}} = \Theta(n)$ holds (e.g., complete graph, star graph, regular expander),
the condition (\ref{eq:corollary2}) reduces to
$|\mathcal{E}|\cdot \big(1 - \frac{H^*(p)}{\log 2}\big) \geq C' n\log n$,
and the sufficient and necessary conditions are thus tight up to at most a logarithmic factor.

\subsection{Illustration}
\label{subsec:illustration1}

\begin{figure}[t]
\vskip 0.2in
\begin{center}
\centerline{\includegraphics[width=\columnwidth]{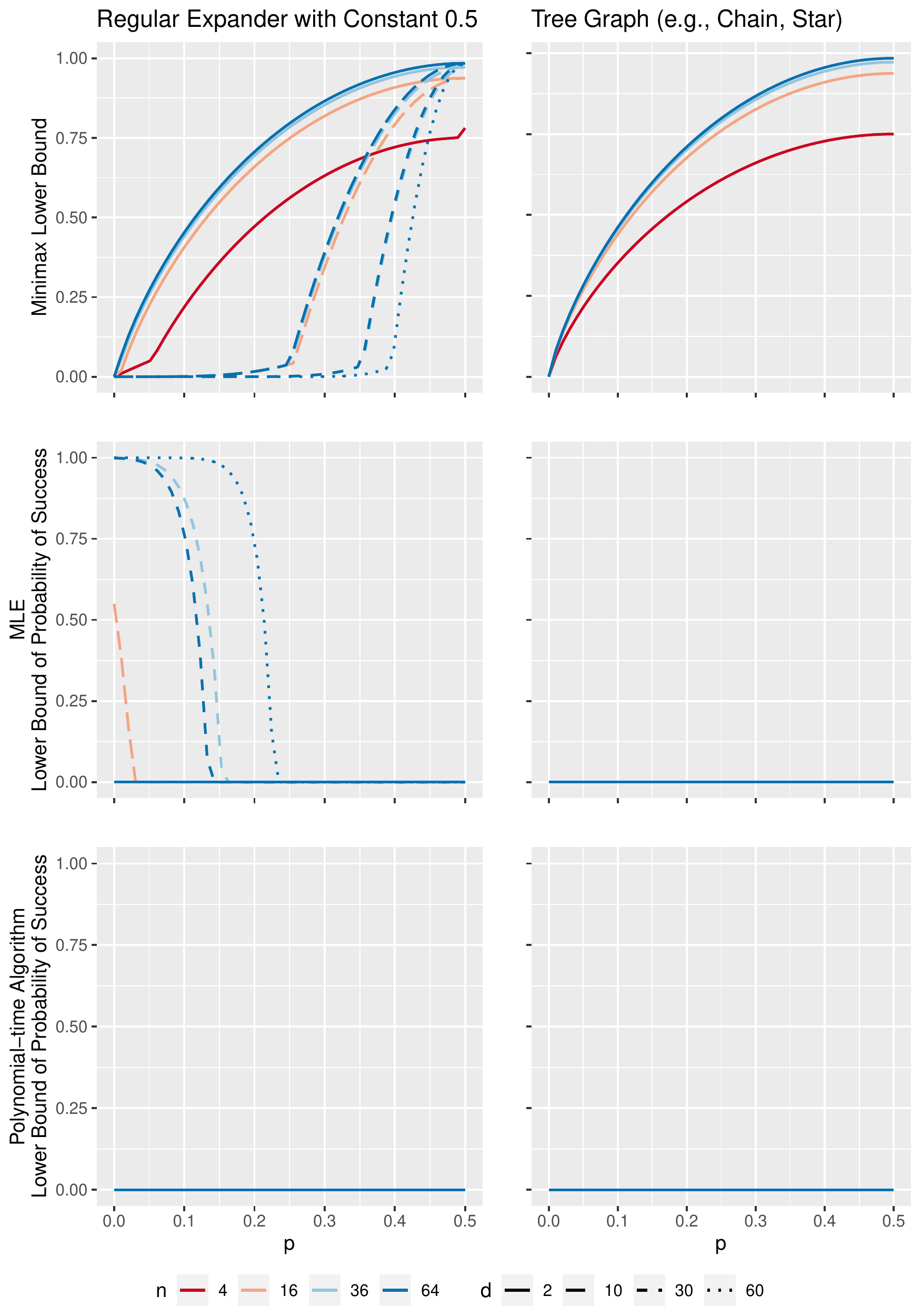}}
\caption{
Minimax lower bound in Theorem \ref{theorem1} (first row), lower bound of probability of success of the MLE algorithm in Theorem \ref{theorem2} (second row), and lower bound of probability of success of the tractable algorithm in \citet{bello2019exact} (third row).
The first and second columns correspond to regular expanders and tree graphs, respectively.
The x-axis indicates the parameter $p\in[0,\frac{1}{2}]$.
Different colors and line-types represent different numbers of nodes ($n$) and the degrees of regular expanders ($d$).
}
\label{fig:illustration1}
\end{center}
\vskip -0.2in
\end{figure}

Figure \ref{fig:illustration1} illustrates the bounds of the probabilities in Theorem \ref{theorem1} and \ref{theorem2}, and the lower bound of the probability of success of the polynomial-time algorithm (which is $1-\epsilon_1$) for regular expanders and tree graphs.
Results for more types of graphs are provided in Appendix \ref{appendix_illustration1}.

The first thing we can observe is the gap between the fundamental limits and the performance of the polynomial-time algorithm.
In regular expanders, we observe that the success of the MLE algorithm is guaranteed with higher probability as the degree $d$ increases,
while the lower bound of the probability of success of the polynomial-time algorithm is zero even with $d=60$.
In fact, $d$ must be about $2000$ so that when $p$ is close to $0$, the lower bound gets close to $1$ in the case of polynomial-time algorithm (See Appendix \ref{appendix_illustration1}.)

Next, we can observe that as $n$ increases, the minimax lower bound increases and the lower bound of the probability of success of the MLE algorithm decreases for tree graphs and regular expanders.
That is, the exact recovery problem becomes difficult as the number of nodes $n$ gets larger for those graph classes.
On the other hand, when $n$ is fixed, the minimax lower bound decreases and the lower bound of the probability of success of the MLE algorithm increases as $d$ increases in regular expanders.
Since the degree $d$ in regular expanders has the same effect on the bounds as the number of nodes $n$ in the complete graph in our results,
the exact recovery problem becomes easier as $n$ increases for complete graphs.
Thus, the effect of $n$ on the fundamental limits is different between complete graphs and the aforementioned graphs.

Lastly, looking in detail at the minimax lower bounds,
we can first observe that 
there exist two non-smooth points when $n=4$ and $d=2$ in the regular expander.
They are the points in which the function achieving the maximum changes.
The maximum lower bound is achieved by $f_1$, $g_1$ and $g_1^*$ in sequence as $p$ increases from $0$ to $1/2$. 
When $n$ and $d$ become greater, $g_1$ and $g_1^*$ become the same, so either $f_1$ or $g_1$ achieves the maximum.
In the case of tree graphs, the function $g_1$ always achieves the maximum, so that there is no non-smooth point.

\section{Regime II: Edge and Node Observations}
\label{sec:case_that_edge_and_node}

Now we consider the extended regime in which noisy edge and node observations are collectively provided.
Recall that the noisy node observation is denoted by the vector $\pmb{c}$.
The $i$th element of $\pmb{c}$ has the same value as the true label $y_i^*$, $i=1,\cdots,n$, with probability $1-q$, $q\in[0,\frac{1}{2}]$.
In addition to $(n,p,|\mathcal{E}|,\Delta_{\max}, \phi_{\mathcal{G}})$, the parameter $q$ will be involved in the results of this section.

\subsection{Minimax Lower Bound}

We first derive the minimax lower bound of the probability of failure as follows.

\begin{theorem}
\label{theorem3}
Let $\mathcal{G} = (\mathcal{V}, \mathcal{E})$ be an undirected connected graph with $n$ nodes, Cheeger constant $\phi_{\mathcal{G}}$, and maximum node degree $\Delta_{\max}$.
Consider a family of distributions $\mathcal{P}$ over $\mathcal{Y} \times (\mathcal{X} \times \mathcal{C})$.
Then, for the inference problem in Section \ref{subsec:exact_inference_problem} in which the noisy edge and node observations $X\in\mathcal{X}$ and $\pmb{c}\in\mathcal{C}$ are given,
the minimax probability of error is bounded below as follows:
\begin{multline*}
\underset{\hat{\pmb{y}}:\mathcal{X}\times\mathcal{C} \to \mathcal{Y}}{\inf}
~\underset{ P \in\mathcal{P} }{\sup}
~\mathbb{P}_{ (\pmb{y}^*, (X,\pmb{c})) \sim P } \big[ \hat{\pmb{y}} (X,\pmb{c}) \neq \pmb{y}^* \big] 
\\ \geq
\max\big\{ f_2, g_2, g^*_2 \big\}
\end{multline*}
where
\begin{align*}
&f_2
:=
\frac{1}{2} \sum_{m=0}^{\Delta_{\max}} \binom{\Delta_{\max}}{m} 
\\&\times\bigg(  \big\{ p^{m} (1-p)^{\Delta_{\max} - m} q \big\}
\wedge 
\big\{ (1-p)^{m}  p^{\Delta_{\max} - m}  (1-q) \big\}
\\&+
\big\{ p^{m} (1-p)^{\Delta_{\max} -m} (1-q) \big\}
\wedge 
\big\{ (1-p)^{m}  p^{\Delta_{\max} - m}  q \big\}  \bigg),
\\[1em]
&g_2
:= \frac{n-1}{n}-\frac{|\mathcal{E}|}{n}\bigg(1 - \frac{H^*(p)}{\log 2}\bigg) -\bigg(1 - \frac{H^*(q)}{\log 2}\bigg),
\\[1em]
&g^*_2
:= 
g_2 +\frac{\big\{(|\mathcal{E}|+n)\log 2-\kappa_2\big\}\vee 0 }{n\log 2}
\\&\hspace{11em}\times\mathbb{I}\big[(1-p)^{|\mathcal{E}|} (1-q)^n \leq e^{-1}\big].
\end{align*}
Here, $
\kappa_2
=
\{-|\mathcal{E}|\log(1-p) -n\log(1-q)\} \cdot \frac{1}{2}\{2(1-q)\}^{n}  \{2(1-p)\}^{|\mathcal{E}|} 
\mathbb{E}_B\big[\frac{\tau(B)}{2^n}\big]
+ \{-|\mathcal{E}|\log(1-p) +n\log 2\} \cdot\frac{1}{2}\{2(1-p)\}^{|\mathcal{E}|}\cdot \frac{1}{2^n}
+ \frac{1}{2}\{2(1-q)\}^{n}\{2(1-p)\}^{|\mathcal{E}|}
\mathbb{E}_B\big[-\frac{\tau(B)}{2^n}\log\big(\frac{\tau(B)}{2^n}\big)\big]$.
In addition,
$\tau(B) := 2\big(\frac{p}{1-p}\big)^B
+
\sum_{k=1}^{n-1} 
\binom{n}{k}
\big(\frac{p}{1-p}\big)^{( \phi_{\mathcal{G}} \cdot\{k\wedge (n-k)\}
- B )\vee 0}$
and $B\sim \text{Bin}(|\mathcal{E}|-n+1,~ \frac{1}{2})$.
\end{theorem}

As in Theorem \ref{theorem1}, $f_2$ is derived from Assouad's lemma and $g_2$ and $g_2^*$ are obtained from Fano's inequality.
See Appendix \ref{proof_theorem3} for the detailed proof.

\begin{figure}[t]
\vskip 0.2in
\begin{center}
\centerline{\includegraphics[width=\columnwidth]{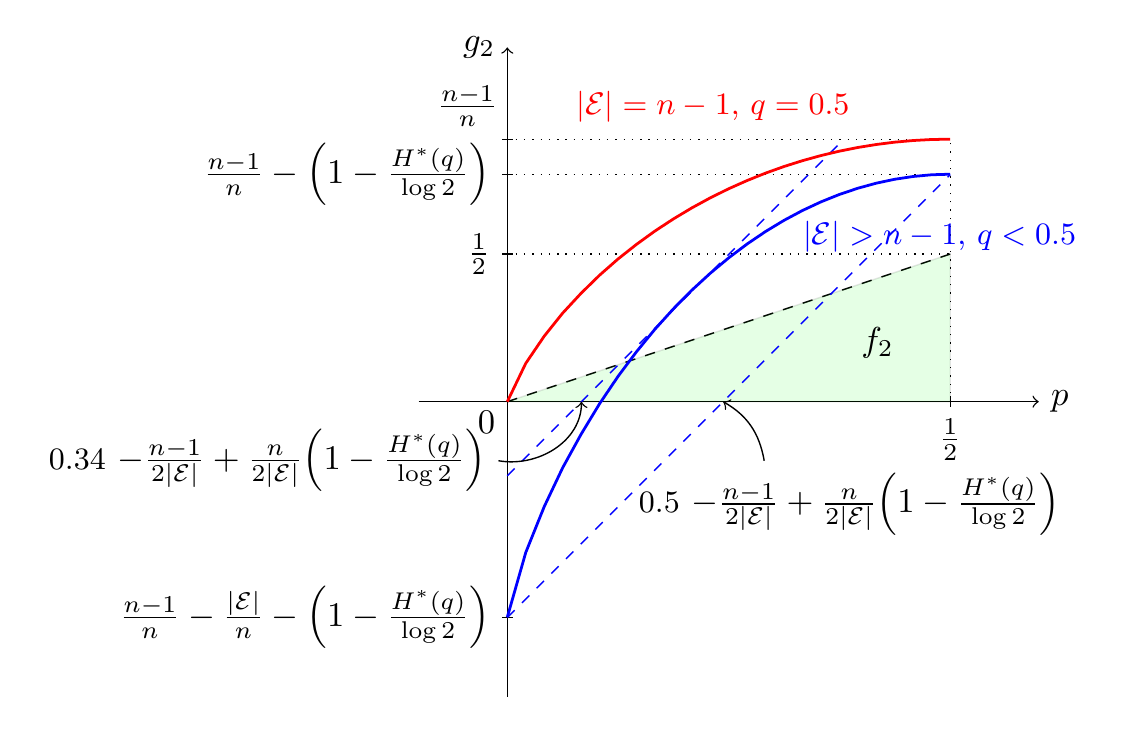}}
\caption{
Graphical illustration of the functions $f_2$ and $g_2$ with respect to the parameter $p\in[0,\frac{1}{2}]$ in Theorem \ref{theorem3}.
The green triangle indicates the area where the function $f_2$ has its function value.
The red and blue lines represent the function $g_2$ when $|\mathcal{E}| = n-1, q=0.5$ and $|\mathcal{E}| > n-1, q<0.5$, respectively.
When $|\mathcal{E}| = n-1$ and $q=0.5$, the function $g_2$ always has a positive value,
and the plot goes downward as $|\mathcal{E}|$ increases or $q$ decreases.
When $|\mathcal{E}| > n-1$ and $q<0.5$, $g_2$ has its $y$-intercept at $\frac{n-1}{n}-\frac{|\mathcal{E}|}{n}-\Big(1 - \frac{H^*(q)}{\log 2}\Big)$ and its $x$-intercept between $0.34 - \frac{n-1}{2|\mathcal{E}|}+ \frac{n}{2|\mathcal{E}|}\Big(1 - \frac{H^*(q)}{\log 2}\Big)$ and $0.5 - \frac{n-1}{2|\mathcal{E}|}+ \frac{n}{2|\mathcal{E}|}\Big(1 - \frac{H^*(q)}{\log 2}\Big)$.
}
\label{fig:minimax2}
\end{center}
\vskip -0.2in
\end{figure}

We can verify from Figure \ref{fig:minimax2} which of the functions $f_2$ or $g_2$ achieves the maximum in different situations.
The figure is similar to Figure \ref{fig:minimax1}, but the difference is that the parameter $q \in [0, \frac{1}{2}]$ is involved in this case.
First, the function $g_2$ is greater than the function $f_2$ for any $p$ only when $|\mathcal{E}| = n-1$ and $q=0.5$.
As $|\mathcal{E}|$ increases or $q$ decreases, the $y$-intercept of function $g_2$ becomes negative and smaller, thus $f_2$ has a greater value than $g_2$ for small $p$.
Especially, when $q$ is close to zero, the maximal function value of $g_2$ becomes negative, therefore $f_2$ is greater than $g_2$ for any $p$ in this case.

The function $g_2^*$ is always greater than or equal to $g_2$, and in most cases, they are equivalent.
An exceptional case is when $p$ and $q$ are both close to $\frac{1}{2}$.
If $p, q \approx \frac{1}{2}$, we can derive that $\kappa_2 \approx (|\mathcal{E}|+n)\log 2 \cdot \frac{1}{2} \big( 1 + \frac{1}{2^{n}} \big)$, that is, $(|\mathcal{E}|+n)\log 2 - \kappa_2 \approx (|\mathcal{E}|+n)\log 2 \cdot \big( \frac{1}{2} - \frac{1}{2^{n+1}} \big)$.
Since this is always positive, $g_2^*$ becomes strictly greater than $g_2$.

Now we derive from Theorem \ref{theorem3} the necessary condition for any algorithm to exactly recover the true label $\pmb{y}^*$.

\begin{corollary}[Necessary Condition]
\label{corollary3}
Under the same assumptions as in Theorem \ref{theorem3}, if the following condition holds:
\begin{equation}
\label{eq:corollary3}
|\mathcal{E}|\bigg(1 - \frac{H^*(p)}{\log 2}\bigg)
+n\bigg(1 - \frac{H^*(q)}{\log 2}\bigg)
\leq \frac{n}{2} - 1,
\end{equation}
then $\underset{\hat{\pmb{y}}:\mathcal{X}\times\mathcal{C} \to \mathcal{Y}}{\inf}
~\underset{ P \in\mathcal{P} }{\sup}
~\mathbb{P}_{ (\pmb{y}^*, (X,\pmb{c})) \sim P } \big[ \hat{\pmb{y}} (X,\pmb{c})\neq \pmb{y}^* \big] \geq \frac{1}{2}$, that is, any algorithm fails to exactly recover the true node labels with probability greater than half.
\end{corollary}

\begin{proof}
Similar to Corollary \ref{corollary1}, we only need to derive the condition in which $g_2$ becomes greater than half.
\end{proof}

The opposite of the condition (\ref{eq:corollary3}) provides the necessary condition for exact recovery.
When $p$ and $q$ are close to $\frac{1}{2}$, $1 - \frac{H^*(p)}{\log 2}$ and $1 - \frac{H^*(q)}{\log 2}$ are close to $0$ as shown in Figure \ref{fig:match}, therefore it becomes impossible for any algorithm to satisfy the necessary condition.
On the other hand, if $1 - \frac{H^*(p)}{\log 2} > \frac{1}{2}$ or $1 - \frac{H^*(q)}{\log 2} > \frac{1}{2}$ holds, i.e., $p$ or $q$ is smaller than $0.11$, then the necessary condition always holds
because $|\mathcal{E}|\big(1 - \frac{H^*(p)}{\log 2}\big)
+n\big(1 - \frac{H^*(q)}{\log 2}\big)
> \frac{n-1}{2} > \frac{n}{2}-1$.

\subsection{Maximum Likelihood Estimator}
\label{subsec:mle2}

Given the noisy edge and node observations, the Maximum Likelihood Estimator (MLE) of the model presented in Section \ref{subsec:exact_inference_problem} is derived as follows:
\begin{equation}
\hat{\pmb{y}}_{mle}(X,\pmb{c}) := \underset{\pmb{y}\in \mathcal{Y}}{\arg\max}~~~ \pmb{y}^T X\pmb{y} + \alpha\pmb{c}^T\pmb{y}. 
\label{eq:mle2}
\end{equation}
Similar to (\ref{eq:mle1}), the optimization problem (\ref{eq:mle2}) is NP-hard but minimax optimal, so it can serve as a benchmark for other algorithms.
Now we derive the lower bound of the probability for the MLE algorithm to exactly recover the true label $\pmb{y}^*$.

\begin{theorem}
\label{theorem4}
Let $\mathcal{G} = (\mathcal{V}, \mathcal{E})$ be an undirected connected graph with $n$ nodes and Cheeger constant $\phi_{\mathcal{G}}$.
Let $h_1(p, z) := e^{- \frac{ (1-2p)^2z }{ \frac{4}{3}(1-p) (1 + 4p)}}$
and
$h_2(z, w) := e^{-\frac{\{(1-2p)z + \alpha(1-2q) w) \}^2}
{ 8p(1-p)z + 8q(1-q)\alpha^2w + \frac{4}{3}[(1-p)\vee (1-q)\alpha]\cdot\{(1-2p)z+\alpha(1-2q)w\}}}$.
Then, for the inference problem in Section \ref{subsec:exact_inference_problem} in which the noisy edge and node observations $X$ and $\pmb{c}$ are given,
the MLE algorithm in (\ref{eq:mle2})
returns a solution $\hat{\pmb{y}}_{mle}(X,\pmb{c}) = \pmb{y}^*$
with probability at least
\begin{align*}
1 - 
&\bigg\{
\sum_{k=1}^{\lfloor \frac{n}{2}\rfloor} \binom{n}{k} 
h_1(p, \phi_{\mathcal{G}}k) + \sum_{k=1}^{n} \binom{n}{k} h_1(q, k)\bigg\}
\\ &\wedge\bigg\{ 
\sum_{k=1}^{\lfloor \frac{n}{2}\rfloor} 
\binom{n}{k} 
h_2(\phi_{\mathcal{G}}k, k)
+
\sum_{k=0}^{\lfloor \frac{n}{2}\rfloor} 
\binom{n}{k} 
h_2(\phi_{\mathcal{G}}k, n-k)
\bigg\},    
\end{align*}
where $\alpha = \frac{\log\frac{1-q}{q}}{\log\frac{1-p}{p}}$.
\end{theorem}

As in Theorem \ref{theorem2}, the Bernstein inequality and the union bound are the main strategies for the proof (See Appendix \ref{proof_theorem4}.)
Note that $h_1(p, \phi_{\mathcal{G}}k)$ and $h_2( \phi_{\mathcal{G}}k, \cdot)$ decrease as $\phi_{\mathcal{G}}$ increases,
that is, the lower bound of the probability of success becomes large when $\phi_{\mathcal{G}}$ increases (e.g., $d$ increases in regular expanders.)
The function $h_2$ is not monotone with respect to $p$ or $q$, so it does not necessarily hold that the bound increases as $p$ or $q$ decreases. 

\textbf{Comparison to Tractable Algorithm.}
We compare the above lower bound to that of a polynomial-time solvable algorithm introduced in \citet{bello2019exact}.
The upper bound of the error probability of the polynomial-time algorithm is $\epsilon_1 + \epsilon_2$ where
$\epsilon_1$ is as in (\ref{eq:epsilon1}) and 
$\epsilon_2 := e^{-\frac{n}{2}(1-2q)^2}$.
We can show that 
if $\phi_{\mathcal{G}}=\Omega(n)$ (e.g., complete graph, regular expander with $d=\Omega(n)$) and $p\approx q \leq 0.25$, then
$\sum_{k=1}^{\lfloor \frac{n}{2}\rfloor} 
\binom{n}{k} 
h_2(\phi_{\mathcal{G}}k, k)
+
\sum_{k=0}^{\lfloor \frac{n}{2}\rfloor} 
\binom{n}{k} 
h_2(\phi_{\mathcal{G}}k, n-k)$ decays much faster than $\epsilon_1+ \epsilon_2$
(see Appendix \ref{proof_comparison2} for the details.)
This implies that the currently existent and tractable method does not achieve the fundamental limits.

Lastly, from Theorem \ref{theorem4}, we derive the sufficient condition for the MLE algorithm to exactly recover the true label with high probability.

\begin{corollary}[Sufficient Condition]
\label{corollary4}
Under the same assumptions as in Theorem \ref{theorem3}, if the following conditions hold:\\
$
\frac{\big\{(1-2p)\phi_{\mathcal{G}} + \alpha(1-2q) \big\}^2}
{ 6p(1-p)\phi_{\mathcal{G}} + 6q(1-q)\alpha^2 + \big[(1-p)\vee (1-q)\alpha\big]\cdot\big\{(1-2p)\phi_{\mathcal{G}}+\alpha(1-2q)\big\}}
$
$\geq \frac{8}{3}\log n$ and
$\frac{\alpha\{(1-2q)\}^2n}
{ 6q(1-q)\alpha + \big[(1-p)\vee (1-q)\alpha\big]\cdot(1-2q)}
$
$\geq \frac{4}{3}\log n$,
then the optimal solution to the MLE algorithm in (\ref{eq:mle2}) fulfills $\hat{\pmb{y}}_{mle}(X,\pmb{c}) = \pmb{y}^*$ with probability $1-5n^{-1}$.
\end{corollary}

\textbf{Tightness of Sufficient and Necessary Conditions.}
When $p$ and $q$ have similar values, the sufficient condition in Corollary \ref{corollary4} can be approximately written as
$(\phi_{\mathcal{G}}+1)\cdot\frac{(1-2p)^2}{(1-p)(1 + 4p)} \geq \frac{8}{3}\log n$ and
$\frac{(1-2q)^2n}
{ (1-q)(1+4q)}
\geq \frac{4}{3}\log n$,
and the necessary condition (\ref{eq:corollary3}) can be written as
$(|\mathcal{E}|+n)\big(1 - \frac{H^*(p)}{\log 2}\big)
> \frac{n}{2} - 1$.
If $|\mathcal{E}|/\phi_{\mathcal{G}} = \Theta(n)$ holds (e.g., complete graph, star graph, regular expander), then
$n(\phi_{\mathcal{G}}+1)\frac{(1-2p)^2}{(1-p)(1 + 4p)} \geq Cn\log n$
and
$n(\phi_{\mathcal{G}}+1)\big(1 - \frac{H^*(p)}{\log 2}\big)
> C'(\frac{n}{2} - 1)$
can be considered the sufficient and necessary conditions
for some constants $C$ and $C'$.
Since $1 - \frac{H^*(p)}{\log 2}$ and $\frac{(1-2p)^2}{(1-p) (1 + 4p)}$ (red and blue lines in Figure \ref{fig:match}, respectively) are almost the same,
the sufficient and necessary conditions are tight up to at most a logarithmic factor.

\subsection{Illustration}
\label{subsec:illustration2}

\begin{figure}[t]
\vskip 0.2in
\begin{center}
\centerline{\includegraphics[width=\columnwidth]{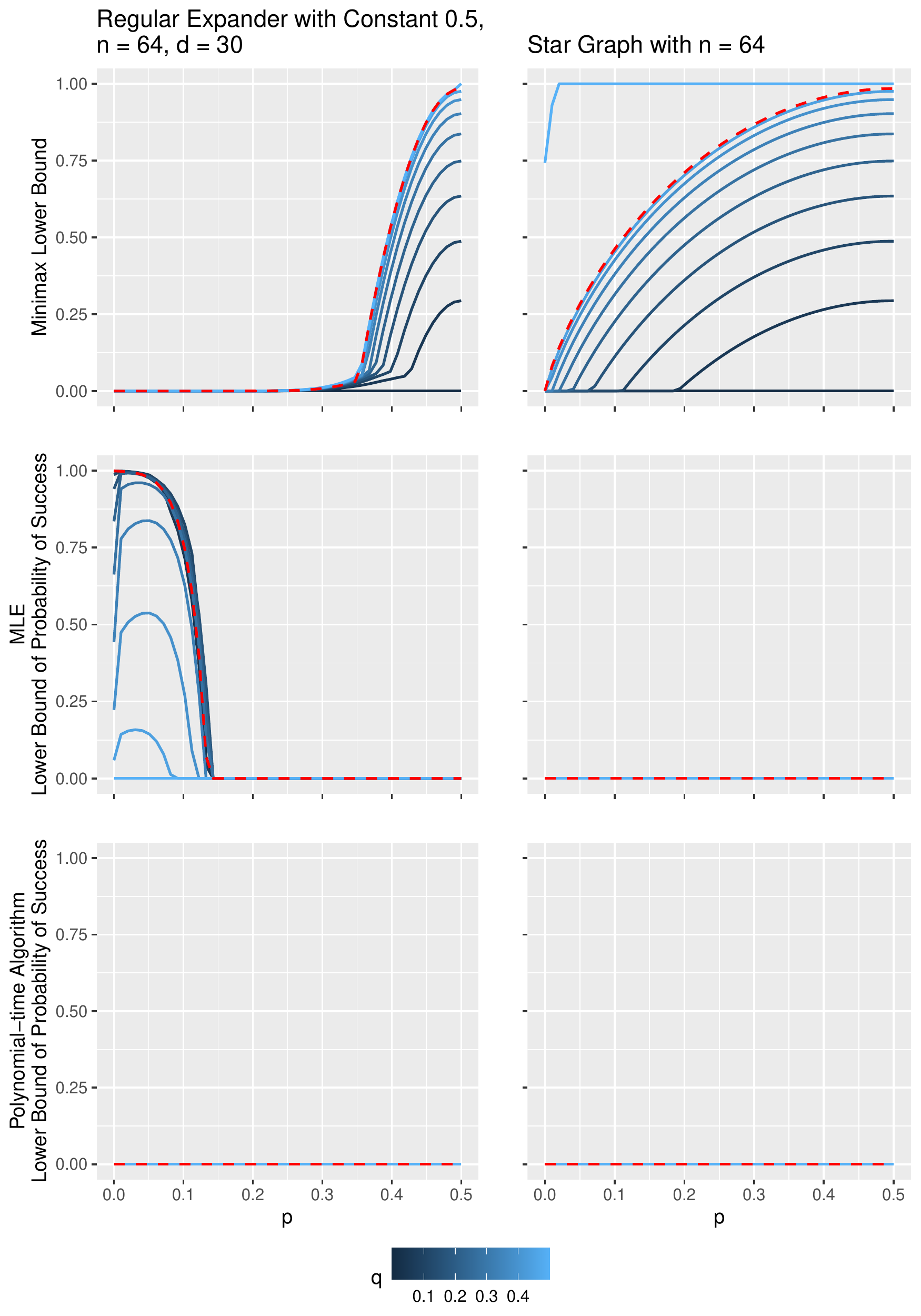}}
\caption{
Minimax lower bound in Theorem \ref{theorem3} (first row), lower bound of probability of success of the MLE algorithm in Theorem \ref{theorem4} (second row), and lower bound of probability of success of the polynomial-time algorithm (third row),
for regular expander with $n=64$, $d=30$ (first column)
and star graph with $n=64$ (second column.)
The x-axis indicates $p\in[0,\frac{1}{2}]$.
Different blue colors represent different $q\in[0,\frac{1}{2}]$.
Red dashed lines denote the bounds provided in Section \ref{sec:case_that_only_edge} for corresponding graphs.
}
\label{fig:illustration2}
\end{center}
\vskip -0.2in
\end{figure}

Figure \ref{fig:illustration2} graphically illustrates the bounds of the probabilities in Theorem \ref{theorem3} and \ref{theorem4}, and the lower bound of the probability of success of the polynomial-time algorithm (which is $1-\epsilon_1 -\epsilon_2$) for the regular expanders and star graphs.
Results for more classes of graphs are given in Appendix \ref{appendix_illustration2}.

As in Section \ref{subsec:illustration1}, we can observe the gap between the fundamental limits and the performance of the polynomial-time algorithm in regular expanders.
We can observe that the lower bound of the probability of success of the MLE algorithm increases as $q$ decreases,
while the lower bound of the probability of success of the polynomial-time algorithm is always zero even with $q=0$.

Next, we can observe that the value of $q$ has a substantial influence on the probability of exact recovery.
When $q$ is close to $\frac{1}{2}$, exact recovery is impossible for any algorithm in star graphs for most $p$ values, as shown in the plot of minimax lower bounds.
On the other hand, when $q$ is close to zero,
the minimax lower bounds become zero in both of regular expanders and star graphs, that is, the necessary condition for exact recovery holds for any algorithm.
Also, when $q$ is close to zero in regular expanders, the lower bound of the probability of success of the MLE algorithm becomes slightly higher than that in the regime where only the noisy edge observation $X$ is given.

Lastly, we remark that there is room for improvement in the lower bound of the probability of success of the MLE algorithm,
because the bound decreases when $p$ decreases around zero in regular expanders.
This does not match with the intuition that the probability of success must increase as $p$ decreases.

\section{Concluding Remarks}
\label{sec:concluding_remarks}

We discussed the fundamental limits of the exact inference problem under a simple generative model assumption \citep{globerson2015hard}, and considered two regimes, one in which only noisy edge observations are given, and one regime in which noisy edge and node observations are both provided.
From the limit bounds, we derived the sufficient and necessary conditions for exact recovery irrespective of computational efficiency,
and we further revealed that those conditions are tight up to at most a logarithmic factor.
Moreover, by comparing the limit bounds of the MLE algorithm and a polynomial-time algorithm \citep{bello2019exact}, we showed that there is a gap between the fundamental limits and the performance of a currently existent and tractable algorithm for exact inference, and thus there is still a need to develop better algorithms.

We assumed a plain generative model in this paper.
Our strategies of the proofs will be able to provide a guide to future works dealing with more complex models.
For instance, as in \citet{heidari2020approximate}, considering non-binary categorical labels would be interesting.
Finally, as mentioned in the main text, there is still room for improvement in the bounds of probabilities we derived.

\bibliography{bibliofile}
\bibliographystyle{plainnat}

\clearpage
\appendix
\onecolumn

\section{Supplementary Material}

\subsection{Preliminaries}

We first introduce the notations, lemmas and facts which will be used throughout the proofs.

For any $\pmb{y} \in \mathcal{Y}=\{-1, 1\}^n$, we write $Y := \pmb{y}\pmb{y}^T = -\pmb{y}(-\pmb{y})^T$.
Also, for any $n\times n$ matrix $A$ and any subset $\mathcal{B}\subseteq\mathcal{E}$, 
we denote by $A_\mathcal{B}$ an upper triangular matrix such that $[A_\mathcal{B}]_{ij} = A_{ij}$ if $(i,j)\in \mathcal{B}$ and $[A_\mathcal{B}]_{ij} = 0$ otherwise. 
We define $\mathcal{Y}_\mathcal{B}:=\{Y_\mathcal{B} ~;~ Y= \pmb{y}\pmb{y}^T, \pmb{y}\in \mathcal{Y}\}$ for any $\mathcal{B}\subseteq\mathcal{E}$.

We note that every connected graph has a spanning tree subgraph, which is a tree containing every vertex of the graph \citep{bollobas2013modern}.
From this fact, we can prove the following lemma.

\begin{lemma}
\label{lemma1}
For any spanning tree subgraph $\mathcal{G}_1 = (\mathcal{V}, \mathcal{E}_1)$ of a connected graph $\mathcal{G}=(\mathcal{V}, \mathcal{E})$ with $n$ nodes,
the map $(Y_{\mathcal{E}} \in \mathcal{Y}_{\mathcal{E}}) \mapsto
([Y_{\mathcal{E}}]_{\mathcal{E}_1} = Y_{\mathcal{E}_1} \in \mathcal{Y}_{\mathcal{E}_1})$
is bijective.
This implies that the elements of $Y_{\mathcal{E}_2}$ can be determined by the elements of $Y_{\mathcal{E}_1}$ where $\mathcal{E}_2=\mathcal{E}\backslash \mathcal{E}_1$.
\end{lemma}

\begin{proof}
For any $Y_{\mathcal{E}_1} \in \mathcal{Y}_{\mathcal{E}_1}$, there exists $Y_{\mathcal{E}} \in \mathcal{Y}_{\mathcal{E}}$ such that $[Y_{\mathcal{E}}]_{\mathcal{E}_1} = Y_{\mathcal{E}_1}$, i.e., the map is onto.
Also, for any connected graph $\mathcal{G}=(\mathcal{V}, \mathcal{E})$, $\pmb{y}\pmb{y}^T =  -\pmb{y}(-\pmb{y})^T$ for all $\pmb{y}\in\mathcal{Y}$ and $\big(\pmb{y}\pmb{y}^T\big)_\mathcal{E} \neq  \big(\Tilde{\pmb{y}}\Tilde{\pmb{y}}^T\big)_\mathcal{E}$ if $\Tilde{\pmb{y}}\neq \pmb{y}$ or $-\pmb{y}$,
and thus the cardinality of $\mathcal{Y}_{\mathcal{E}}$ is $|\mathcal{Y}_{\mathcal{E}}| = \frac{1}{2}\cdot |\mathcal{Y}| = 2^{n-1}$.
This holds for both $\mathcal{Y}_{\mathcal{E}}$ and $\mathcal{Y}_{\mathcal{E}_1}$, that is, they have the same cardinality. Therefore, the map is bijective. 
\end{proof}

The above lemma results in the following fact:

\begin{fact}
\label{fact1}
Note that for any spanning tree subgraph $\mathcal{G}_1 = (\mathcal{V}, \mathcal{E}_1)$ of a connected graph $\mathcal{G}=(\mathcal{V}, \mathcal{E})$ with $n$ nodes, the cardinaltiy of $\mathcal{E}_1$ is $|\mathcal{E}_1| = n-1$ and $|\mathcal{Y}_{\mathcal{E}_1}| = 2^{n-1}$.
Thus, we derive that 
$\mathcal{Y}_{\mathcal{E}_1} = \{-1, 1\}^{n-1}$.
\end{fact}

Next, for any $\pmb{y}$ and $\pmb{y}' \in \mathcal{Y}$, we define $S_{\pmb{y}'}(\pmb{y}) := \{i\in \mathcal{V} ~;~ y'_i \neq y_i\}$.
We denote the Hamming distance by $d(\cdot, \cdot)$.
Now, we can easily check the following facts:

\begin{fact}
\label{fact2}
$d\Big((\pmb{y}'{\pmb{y}'}^T)_\mathcal{E}, (\pmb{y}{\pmb{y}}^T)_\mathcal{E}\Big) 
= \sum_{(i,j)\in\mathcal{E}} \mathbb{I}[y'_iy'_j\neq y_iy_j] = \big| \mathcal{E} \big(S_{\pmb{y}'}(\pmb{y}), S_{\pmb{y}'}(\pmb{y})^c \big) \big|$.
\end{fact}

\begin{fact}
\label{fact3}
$d(\pmb{y}', \pmb{y}) = \sum_{i\in\mathcal{V}}\mathbb{I}[y'_i\neq y_i] = |S_{\pmb{y}'}(\pmb{y})|$.
\end{fact}

\begin{fact}
\label{fact4}
For a fixed $\pmb{y}'$,
the map $S_{\pmb{y}'}(\cdot): \pmb{y}\in\mathcal{Y} \mapsto S_{\pmb{y}'}(\pmb{y}) \in \mathcal{P}(\mathcal{V})$ is bijective where $\mathcal{P}(\mathcal{V})$ is the power set of $\mathcal{V}$.
\end{fact}

Lastly, we present Assouad's lemma (Lemma 2 in \citep{yu1997assouad}) and Lemma \ref{lemma3} which will be used to prove Lemma \ref{theorem1_lemma1} and Lemma \ref{theorem2_lemma1}.

\begin{lemma}[Assouad's lemma]
\label{lemma2}
Consider a distance function $\delta(\cdot, \cdot)$ on $\mathcal{Y}\times \mathcal{Y}$. Suppose that
there are $n$ pseudo-distances $\{\delta_j(\cdot, \cdot)\}_{j=1,\cdots,n}$ such that for any $\pmb{y}, \pmb{y}' \in \mathcal{Y}$,
$$
\delta(\pmb{y}, \pmb{y}') = \sum_{j=1}^n \delta_j(\pmb{y}, \pmb{y}')
$$
and there exists $\alpha_n >0$ such that
for all $j=1,\cdots,n$,
$$
\delta_j(\pmb{y}, \pmb{y}') \geq \alpha_n
$$
when $y_j\neq y'_j$ and $y_i= y'_i$ for $\forall i \neq j$.
Let $P_{X|\pmb{y}}$ be the conditional probability distribution of $X \in \mathcal{X}$ given $\pmb{y}$, and let $p(X|\pmb{y})$ be its related probability mass function. Then for any estimator $\hat{\pmb{y}}(\cdot)$, we have
$$
\underset{\pmb{y}^*\in\mathcal{Y}}{\max}~\mathbb{E}_{X|\pmb{y}^*}\big[\delta(\hat{\pmb{y}}(X), \pmb{y}^*)\big]
\geq
n\cdot \frac{\alpha_n}{2} 
\min\bigg\{\|P_{X|\pmb{y}} \wedge P_{X|\pmb{y}'} \| : \sum_{i=1}^n \mathbb{I}[y_i \neq y'_i] = 1,~ \pmb{y}, \pmb{y}' \in \mathcal{Y} \bigg\}
$$
where
$\|P_{X|\pmb{y}} \wedge P_{X|\pmb{y}'} \| 
= \sum_{X\in\mathcal{X}} p(X|\pmb{y}) \wedge p(X|\pmb{y}')$.
\end{lemma}

\begin{lemma}
\label{lemma3}
$l(n) := \sum_{m=0}^{n}
\binom{n}{m}
\Big[
\big\{
p^{m} \cdot (1-p)^{n - m}
\cdot q \big\}
\wedge 
\big\{
(1-p)^{m} \cdot p^{n - m}
\cdot (1-q) \big\}\Big]
$,
$n\in\mathbb{N}$, is a non-increasing function for any $p, q\in[0, 1]$.
\end{lemma}

\begin{proof}
For any $n\geq 1$,
\begin{align*}
l(n) 
&= 
\sum_{m=0}^{n} \binom{n}{m} \Big[ \big\{ p^{m} \cdot (1-p)^{n - m} \cdot q \big\} \wedge \big\{ (1-p)^{m} \cdot p^{n - m} \cdot (1-q) \big\}\Big]
\\&= 
\sum_{m=0}^{n-1} \binom{n-1}{m} \Big[ \big\{ p^{m} \cdot (1-p)^{n - m} \cdot q \big\} \wedge \big\{ (1-p)^{m} \cdot p^{n - m} \cdot (1-q) \big\}\Big]
\\&~~~~+
\sum_{m=1}^{n} \binom{n-1}{m-1} \Big[ \big\{ p^{m} \cdot (1-p)^{n - m} \cdot q \big\} \wedge \big\{ (1-p)^{m} \cdot p^{n - m} \cdot (1-q) \big\}\Big]
\\&= 
\sum_{m=0}^{n-1} \binom{n-1}{m} \Big[ \big\{ p^{m} \cdot (1-p)^{n-1 - m} \cdot q (1-p)\big\} \wedge \big\{ (1-p)^{m} \cdot p^{n-1 - m} \cdot (1-q) p \big\}\Big]
\\&~~~~+
\sum_{m=0}^{n-1} \binom{n-1}{m} \Big[ \big\{ p^{m} \cdot (1-p)^{n -1 - m} \cdot q p \big\} \wedge \big\{ (1-p)^{m} \cdot p^{n -1 - m} \cdot (1-q) (1-p) \big\}\Big]
\\&\leq
\sum_{m=0}^{n-1} \binom{n-1}{m} \Big[ \big\{ p^{m} \cdot (1-p)^{n-1 - m} \cdot q (1-p) + p^{m} \cdot (1-p)^{n -1 - m} \cdot q p \big\}
\\&~~~~~~~~~~~~~~~~~~~~~~~~\wedge 
\big\{ (1-p)^{m} \cdot p^{n-1 - m} \cdot (1-q) p + (1-p)^{m} \cdot p^{n -1 - m} \cdot (1-q) (1-p) \big\}\Big]
\\&=
\sum_{m=0}^{n-1} \binom{n-1}{m} \Big[ \big\{ p^{m} \cdot (1-p)^{n-1 - m} \cdot q \big\} \wedge \big\{ (1-p)^{m} \cdot p^{n-1 - m} \cdot (1-q) \big\}\Big]
\\&= l(n-1)
\end{align*}
where the inequality holds because $a\wedge b + c\wedge d \leq (a+c)\wedge (b+d)$.
\end{proof}

\subsection{Proof of Theorem \protect\ref{theorem1}}
\label{proof_theorem1}

We can prove Theorem \ref{theorem1} by showing the following lemmas:

\begin{lemma}
\label{theorem1_lemma1}
Under the same conditions as in Theorem \ref{theorem1}, we have
$
\underset{\hat{\pmb{y}} : \mathcal{X} \to \mathcal{Y}}{\inf}
\underset{P \in\mathcal{P}}{\sup}
\mathbb{P}_{ (\pmb{y}^*, X) \sim P } \big[ \hat{\pmb{y}} (X) \neq \pmb{y}^* \big] 
\geq f_1.
$
\end{lemma}

\begin{lemma}
\label{theorem1_lemma2}
Under the same conditions as in Theorem \ref{theorem1}, we have
$
\underset{\hat{\pmb{y}} : \mathcal{X} \to \mathcal{Y}}{\inf}
\underset{P \in\mathcal{P}}{\sup}
\mathbb{P}_{ (\pmb{y}^*, X) \sim P } \big[ \hat{\pmb{y}} (X) \neq \pmb{y}^* \big]  
\geq \max\{g_1, g_1^*\}.
$
\end{lemma}

\subsubsection{Proof of Lemma \protect\ref{theorem1_lemma1}}

We apply Lemma \ref{lemma2} to our problem.
Here, we consider the zero-one distance 
$\delta(\pmb{y}, \pmb{y}') = \mathbb{I}[\pmb{y}\neq \pmb{y}']$
and the pseudo-distance $\delta_j(\pmb{y}, \pmb{y}') = \frac{1}{n}\mathbb{I}[\pmb{y}\neq \pmb{y}']$ for $j=1,\cdots,n$.
Then with $\alpha_n = \frac{1}{n}$, we can derive that
$$
\underset{\pmb{y}^*\in\mathcal{Y}}{\max}~\mathbb{P}_{X|\pmb{y}^*}\big[\hat{\pmb{y}}(X) \neq \pmb{y}^*\big]
\geq
\frac{1}{2} 
\min\bigg\{\|P_{X|\pmb{y}} \wedge P_{X|\pmb{y}'} \| : \sum_{i=1}^n \mathbb{I}[y_i \neq y'_i] = 1,~ \pmb{y}, \pmb{y}' \in \mathcal{Y} \bigg\}.
$$

To find the minimum value of the RHS, we consider $\pmb{y}, \pmb{y}'$ such that
$y_k\neq y'_k$ and $y_i= y'_i$ for $\forall i \neq k$. 
Denote the set of edges which are connected to the $k$th node by $\mathcal{E}_k$.
Then we can obtain that
\begin{align*}
p(X|\pmb{y}) \wedge p(X|\pmb{y}')
&= 
\bigg\{\prod_{(i,j) \in \mathcal{E}} p(X_{ij}|y_iy_j)\bigg\}\wedge
\bigg\{\prod_{(i,j) \in \mathcal{E}} p(X_{ij}|y'_iy'_j)\bigg\}
\\&=
\bigg\{\prod_{(i,j) \notin \mathcal{E}_k} p(X_{ij}|y_iy_j)
\cdot \prod_{(i,j) \in \mathcal{E}_k} p(X_{ij}|y_iy_j) \bigg\}
\wedge \bigg\{\prod_{(i,j) \notin \mathcal{E}_k} p(X_{ij}|y_iy_j) \cdot \prod_{(i,j) \in \mathcal{E}_k} p(X_{ij}|y'_iy'_j) \bigg\}
\\&=
\bigg\{\prod_{(i,j) \notin \mathcal{E}_k} p(X_{ij}|y_iy_j)\bigg\}
\times \Bigg[ \bigg\{ \prod_{(i,j) \in \mathcal{E}_k} p(X_{ij}|y_iy_j) \bigg\} \wedge \bigg\{ \prod_{(i,j) \in \mathcal{E}_k} p(X_{ij}|y'_iy'_j) \bigg\} \Bigg]
\\&=
\bigg\{\prod_{(i,j) \notin \mathcal{E}_k} p(X_{ij}|y_iy_j)\bigg\} \times \Bigg[ \bigg\{ \prod_{(i,j) \in \mathcal{E}_k} p(X_{ij}|y_iy_j) \bigg\} \wedge \bigg\{ \prod_{(i,j) \in \mathcal{E}_k} 1-p(X_{ij}|y_iy_j) \bigg\} \Bigg]
\\&=
\bigg\{\prod_{(i,j) \notin \mathcal{E}_k} p(X_{ij}|y_iy_j)\bigg\}
\\&~~~\times
\Bigg[ \bigg\{ p^{d(X_{\mathcal{E}_k}, Y_{\mathcal{E}_k})} \cdot (1-p)^{\Delta_k - d(X_{\mathcal{E}_k}, Y_{\mathcal{E}_k})} \bigg\} \wedge \bigg\{ (1-p)^{d(X_{\mathcal{E}_k}, Y_{\mathcal{E}_k})} \cdot p^{\Delta_k - d(X_{\mathcal{E}_k}, Y_{\mathcal{E}_k})} \bigg\} \Bigg].
\end{align*}
Accordingly,
\begin{align*}
\|P_{X|\pmb{y}} \wedge P_{X|\pmb{y}'} \| 
&= 
\sum_{X\in\mathcal{X}} p(X|\pmb{y}) \wedge p(X|\pmb{y}')
\\&=
\sum_{X\in\mathcal{X}} \bigg\{\prod_{(i,j) \notin \mathcal{E}_k} p(X_{ij}|y_iy_j)\bigg\}
\\&~~~~~~~~~~\times \Bigg[ \bigg\{ p^{d(X_{\mathcal{E}_k}, Y_{\mathcal{E}_k})} \cdot (1-p)^{\Delta_k - d(X_{\mathcal{E}_k}, Y_{\mathcal{E}_k})} \bigg\} \wedge \bigg\{ (1-p)^{d(X_{\mathcal{E}_k}, Y_{\mathcal{E}_k})} \cdot p^{\Delta_k - d(X_{\mathcal{E}_k}, Y_{\mathcal{E}_k})} \bigg\} \Bigg]
\\&=
\mathop{\sum_{X_{ij} \in \{-1,1\}}}_{(i,j)\notin\mathcal{E}_k} \bigg\{\prod_{(i,j) \notin \mathcal{E}_k} p(X_{ij}|y_iy_j)\bigg\}
\\&~~~~\times
\mathop{\sum_{X_{ij} \in \{-1,1\}}}_{(i,j)\in\mathcal{E}_k} \Bigg[ \bigg\{ p^{d(X_{\mathcal{E}_k}, Y_{\mathcal{E}_k})} \cdot (1-p)^{\Delta_k - d(X_{\mathcal{E}_k}, Y_{\mathcal{E}_k})}
\bigg\} \wedge \bigg\{ (1-p)^{d(X_{\mathcal{E}_k}, Y_{\mathcal{E}_k})} \cdot p^{\Delta_k - d(X_{\mathcal{E}_k}, Y_{\mathcal{E}_k})} \bigg\} \Bigg]
\\&=
\mathop{\sum_{X_{ij} \in \{-1,1\}}}_{(i,j)\in\mathcal{E}_k} \Bigg[ \bigg\{ p^{d(X_{\mathcal{E}_k}, Y_{\mathcal{E}_k})} \cdot (1-p)^{\Delta_k - d(X_{\mathcal{E}_k}, Y_{\mathcal{E}_k})} \bigg\} \wedge \bigg\{ (1-p)^{d(X_{\mathcal{E}_k}, Y_{\mathcal{E}_k})} \cdot p^{\Delta_k - d(X_{\mathcal{E}_k}, Y_{\mathcal{E}_k})} \bigg\} \Bigg].
\end{align*}
Note that
$d(X_{\mathcal{E}_k}, Y_{\mathcal{E}_k})$
has a value between $0$ and $\Delta_k$, 
and for each $m\in[0, \Delta_k]$, there exist $\binom{\Delta_k}{m}$ different $X_{\mathcal{E}_k}$'s satisfying $d(X_{\mathcal{E}_k}, Y_{\mathcal{E}_k})=m$.
Hence, we can write that
\begin{align*}
\|P_{X|\pmb{y}} \wedge P_{X|\pmb{y}'} \| 
&=
\sum_{m=0}^{\Delta_k} \binom{\Delta_k}{m} \bigg[ \Big\{ p^{m} \cdot (1-p)^{\Delta_k - m} \Big\} \wedge \Big\{ (1-p)^{m} \cdot p^{\Delta_k - m} \Big\} \bigg],
\end{align*}
and consequently,
\begin{align*}
&\min\bigg\{\|P_{X|\pmb{y}} \wedge P_{X|\pmb{y}'} \| : \sum_{i=1}^n \mathbb{I}[y_i \neq y'_i] = 1,~ \pmb{y}, \pmb{y}' \in \mathcal{Y} \bigg\}
\\&=
\underset{1\leq k \leq n}{\min} \Bigg( \sum_{m=0}^{\Delta_k} \binom{\Delta_k}{m} \bigg[ \Big\{ p^{m} \cdot (1-p)^{\Delta_k - m} \Big\} \wedge \Big\{ (1-p)^{m} \cdot p^{\Delta_k - m} \Big\} \bigg] \Bigg)
\\&=
\sum_{m=0}^{\Delta_{\max}} \binom{\Delta_{\max}}{m} \bigg[ \Big\{ p^{m} \cdot (1-p)^{\Delta_{\max} - m} \Big\} \wedge \Big\{ (1-p)^{m} \cdot p^{\Delta_{\max} - m} \Big\} \bigg].
\end{align*}
The last equality holds by Lemma \ref{lemma3}, where $q$ is set as $\frac{1}{2}$.

Finally, consider a joint probability distribution $P_0$ of $\pmb{y}^*$ and $X$ where the marginal probability of $\pmb{y}^*$ is $p_{\pmb{y}^*}(\pmb{y}_0) = 1$ for some $\pmb{y}_0\in\mathcal{Y}$, and where $X$ given $\pmb{y}^*$ follows the assumed conditional probability distribution.
Then we have that
\begin{align*}
\underset{P \in\mathcal{P}}{\sup}
\mathbb{P}_{ (\pmb{y}^*, X) \sim P } \big[ \hat{\pmb{y}} (X) \neq \pmb{y}^* \big] 
&\geq
\mathbb{P}_{ (\pmb{y}^*, X) \sim P_0 } \big[ \hat{\pmb{y}} (X) \neq \pmb{y}^* \big]     
\\&=
\mathbb{P}_{X|\pmb{y}_0}\big[\hat{\pmb{y}}(X) \neq \pmb{y}_0\big]
\end{align*}
which holds for any $\pmb{y}_0 \in \mathcal{Y}$.
Hence, we derive that
$$
\underset{P \in\mathcal{P}}{\sup}
\mathbb{P}_{ (\pmb{y}^*, X) \sim P } \big[ \hat{\pmb{y}} (X) \neq \pmb{y}^* \big] 
\geq
\underset{\pmb{y}^*\in\mathcal{Y}}{\max}~\mathbb{P}_{X|\pmb{y}^*}\big[\hat{\pmb{y}}(X) \neq \pmb{y}^*\big]
$$
for any estimator $\hat{\pmb{y}}(\cdot)$.

Therefore, we have
\begin{align*}
\underset{\hat{\pmb{y}} : \mathcal{X} \to \mathcal{Y}}{\inf}
\underset{P \in\mathcal{P}}{\sup}
\mathbb{P}_{ (\pmb{y}^*, X) \sim P } \big[ \hat{\pmb{y}} (X) \neq \pmb{y}^* \big] 
&\geq
\frac{1}{2} \min\bigg\{\|P_{X|\pmb{y}} \wedge P_{X|\pmb{y}'} \| : \sum_{i=1}^n \mathbb{I}[y_i \neq y'_i] = 1,~ \pmb{y}, \pmb{y}' \in \mathcal{Y} \bigg\}
\\&=
\frac{1}{2} \sum_{m=0}^{\Delta_{\max}} \binom{\Delta_{\max}}{m} \bigg[ \Big\{ p^{m} \cdot (1-p)^{\Delta_{\max} - m} \Big\} \wedge \Big\{ (1-p)^{m} \cdot p^{\Delta_{\max} - m} \Big\} \bigg]
= f_1.
\end{align*}

\subsubsection{Proof of Lemma \protect\ref{theorem1_lemma2}}

We will apply Fano's inequality (Theorem 2.10.1 in \citep{cover1999elements}) to our problem. When we observe that $\pmb{y}^* \rightarrow X \rightarrow \hat{\pmb{y}}$ forms a Markov chain, we have
$$
\mathbb{P}(\hat{\pmb{y}}\neq\pmb{y}^*) 
\geq 1-\frac{\mathbb{I}(\pmb{y}^*, X) + \log 2}{\log |\mathcal{Y}|}
$$
where $\mathbb{I}(\pmb{y}^*, X)$ is the mutual information of $\pmb{y}^*$ and $X$. 
To achieve our goal, we need to find an exact form or an upper bound of the mutual information $\mathbb{I}(\pmb{y}^*, X)$.
For simplicity, we write $\pmb{y}=\pmb{y}^*$ in the rest of the proof.

We can write the mutual information as follows:
\begin{align*}
\mathbb{I}(\pmb{y}, X) 
&= \frac{1}{2^n} \sum_{\pmb{y}\in \mathcal{Y}} \mathbb{KL}(P_{X|\pmb{y}}\|P_X)  
\\&=
\frac{1}{2^n} \sum_{\pmb{y}\in \mathcal{Y}} \sum_{X\in\mathcal{X}} p(X|\pmb{y}) \log \bigg[ \frac{p(X|\pmb{y})}{p(X)} \cdot \frac{q(X)}{q(X)} \bigg]
\\&=
\frac{1}{2^n} \sum_{\pmb{y}\in \mathcal{Y}} \sum_{X\in\mathcal{X}} p(X|\pmb{y}) \bigg[ \log  \frac{p(X|\pmb{y})}{q(X)} -\log \frac{p(X)}{q(X)} \bigg]
\\&=
\frac{1}{2^n} \sum_{\pmb{y}\in \mathcal{Y}} \mathbb{KL}(P_{X|\pmb{y}}\|Q_X)
- \sum_{X\in\mathcal{X}} \bigg[ \frac{1}{2^n} \sum_{\pmb{y}\in \mathcal{Y}} p(X|\pmb{y}) \bigg] \log\frac{p(X)}{q(X)}
\\&= 
\frac{1}{2^n} \sum_{\pmb{y}\in \mathcal{Y}} \mathbb{KL}(P_{X|\pmb{y}}\|Q_X)
- \sum_{X\in\mathcal{X}} p(X) \log\frac{p(X)}{q(X)}
\\&= \frac{1}{2^n} \sum_{\pmb{y}\in \mathcal{Y}} \mathbb{KL}(P_{X|\pmb{y}}\|Q_X) - \mathbb{KL}(P_X\|Q_X).
\end{align*}
The first equality holds because $\pmb{y}$ follows a uniform distribution, and the rest holds for any distribution $Q_X$ on $\mathcal{X}$.
If we set $Q_X$ to be a uniform distribution on $\mathcal{X}$, i.e., $q(x) = 2^{-|\mathcal{E}|}$, then we have
\begin{align*}
\mathbb{KL}(P_{X|\pmb{y}}\|Q_X)
&= \sum_{X\in\mathcal{X}} p(X|\pmb{y}) \log\frac{p(X|\pmb{y})}{2^{-|\mathcal{E}|}}
\\&= \sum_{X\in\mathcal{X}} \Big\{\prod_{(i,j)\in\mathcal{E}}p(X_{ij}|y_iy_j)\Big\}\cdot \bigg[
\log \Big(\prod_{(i,j)\in\mathcal{E}}p(X_{ij}|y_iy_j) \Big)
+ |\mathcal{E}|\log 2 \bigg]
\\&= \bigg[\sum_{(i,j)\in\mathcal{E}} \sum_{X_{ij}\in \{-1, 1\}}p(X_{ij}|y_iy_j) \log \Big( p(X_{ij}|y_iy_j) \Big) \bigg] + |\mathcal{E}|\log 2
\\&= |\mathcal{E}|\Big(p\log p + (1-p) \log (1-p)\Big) + |\mathcal{E}|\log 2
\\&= -|\mathcal{E}| H^*(p) + |\mathcal{E}|\log 2
\end{align*}
and
\begin{align*}
\mathbb{KL}(P_X\|Q_X)
&= \sum_{X\in\mathcal{X}} p(X)\log \frac{p(X)}{2^{-|\mathcal{E}|}}
\\&= \sum_{X\in\mathcal{X}} p(X)\log p(X) + |\mathcal{E}|\log 2
\\&= -H(X) + |\mathcal{E}|\log 2
\end{align*}
where $H(X) = - \sum_{X\in\mathcal{X}} p(X)\log p(X)$ is the entropy of $X$. 
Hence, we have the following exact form of the mutual information:
$$
\mathbb{I}(\pmb{y}, X) = H(X) -|\mathcal{E}| H^*(p).
$$
Note that the entropy of a discrete random variable is always positive and bounded above by the logarithm of the size of its domain, which implies that $0 \leq H(X) \leq |\mathcal{E}|\log 2$ in our case. 
By using this fact, we can simply find the upper bound of the mutual information as follows:
$$
\mathbb{I}(\pmb{y}, X) \leq |\mathcal{E}| \big(\log 2 -H^*(p)\big).
$$
Then we have
\begin{align}
\mathbb{P}(\hat{\pmb{y}}\neq\pmb{y}^*) 
&\geq 1-\frac{|\mathcal{E}| \big(\log 2 -H^*(p)\big) + \log 2}{n\log 2} \nonumber    
\\[1em]&=\frac{n-1}{n} - \frac{|\mathcal{E}|}{n}\cdot \bigg(1-\frac{H^*(p)}{\log2} \bigg) = g_1.
\label{eq:minimax1_fano_proof1}
\end{align}

Now, we want to find another upper bound of $H(X)$ instead of $|\mathcal{E}|\log 2$. 
For this, we can first write the probability mass function $p(X)$ as follows:
\begin{align*}
p(X)
&= \sum_{\pmb{y}\in\mathcal{Y}} p(X|\pmb{y})p(\pmb{y})
= \frac{1}{2^n} \sum_{\pmb{y}\in\mathcal{Y}} p(X|\pmb{y})
= \frac{1}{2^n} \sum_{\pmb{y}\in\mathcal{Y}} \Big\{\prod_{(i,j)\in\mathcal{E}}p(X_{ij}|y_iy_j)\Big\}
\\&=\frac{1}{2^n} \sum_{\pmb{y}\in\mathcal{Y}} 
\Big\{\prod_{(i,j)\in\mathcal{E}} p^{\mathbb{I}[X_{ij}\neq y_iy_j]} (1-p)^{1-\mathbb{I}[X_{ij}\neq y_iy_j]}\Big\}
\\&= \frac{1}{2^n} \sum_{\pmb{y}\in\mathcal{Y}} 
\Big\{ p^{d(X, Y_\mathcal{E})} (1-p)^{|\mathcal{E}|- d(X, Y_\mathcal{E})}\Big\}
\\&= \frac{1}{2^n} (1-p)^{|\mathcal{E}|} \sum_{\pmb{y}\in\mathcal{Y}} \bigg(\frac{p}{1-p}\bigg)^{d(X, Y_\mathcal{E})}.
\end{align*}

Here, we will partition the space $\mathcal{X} = \{-1, 1\}^{|\mathcal{E}|}$ into $\mathcal{Y}_{\mathcal{E}}$ and $\mathcal{X}\backslash\mathcal{Y}_{\mathcal{E}}$, and deal with
$-\sum_{X\in\mathcal{Y}_{\mathcal{E}}} p(X)\log p(X)$ and $-\sum_{X\in\mathcal{X}\backslash\mathcal{Y}_{\mathcal{E}}} p(X)\log p(X)$ separately.
The first thing to note is that
$\mathcal{Y}_\mathcal{E}=\{Y_\mathcal{E} ~;~ Y= \pmb{y}\pmb{y}^T, \pmb{y}\in \mathcal{Y}\}$
is the set of all matrices whose elements are feasible label products of the edges. 
Therefore, if $X \in \mathcal{Y}_{\mathcal{E}}$, we can find a vector $\pmb{y}'\in \mathcal{Y}$ which satisfies $X = Y'_\mathcal{E} := (\pmb{y}'{\pmb{y}'}^T)_\mathcal{E}$. 
On the other hand, if $X \in \mathcal{X}\backslash\mathcal{Y}_{\mathcal{E}}$, 
there does not exist such a vector $\pmb{y}'\in \mathcal{Y}$ satisfying $X = Y'_\mathcal{E} := (\pmb{y}'{\pmb{y}'}^T)_\mathcal{E}$.

First, consider the case that $X \in \mathcal{Y}_{\mathcal{E}}$. 
There exists $\pmb{y}'\in \mathcal{Y}$ such that $X = Y'_\mathcal{E} = (\pmb{y}'{\pmb{y}'}^T)_\mathcal{E}$,
and for each $\pmb{y}\in \mathcal{Y}$, we can write $d(X, Y_\mathcal{E}) = d\Big((\pmb{y}'{\pmb{y}'}^T)_\mathcal{E}, (\pmb{y}{\pmb{y}}^T)_\mathcal{E}\Big)$. 
Hence, we have that
\begin{align*}
\sum_{\pmb{y}\in\mathcal{Y}} \bigg(\frac{p}{1-p}\bigg)^{d(X, Y_\mathcal{E})}
&= 
\sum_{\pmb{y}\in\mathcal{Y}} \bigg(\frac{p}{1-p}\bigg)^{d((\pmb{y}'{\pmb{y}'}^T)_\mathcal{E}, (\pmb{y}{\pmb{y}}^T)_\mathcal{E})}
\\&= 
\sum_{\pmb{y}\in\mathcal{Y}} \bigg(\frac{p}{1-p}\bigg)^{| \mathcal{E} (S_{\pmb{y}'}(\pmb{y}), S_{\pmb{y}'}(\pmb{y})^c ) |}
\\&= \sum_{S\in \mathcal{P}(\mathcal{V})} \bigg(\frac{p}{1-p}\bigg)^{|\mathcal{E}(S, S^c)|}
= \sum_{S\subseteq \mathcal{V}} \bigg(\frac{p}{1-p}\bigg)^{|\mathcal{E}(S, S^c)|}    
\end{align*}
where the second and third equalities hold by the Facts \ref{fact2} and \ref{fact4}, respectively.
Accordingly, for any $X \in \mathcal{Y}_{\mathcal{E}}$,
$$
p(X) = \frac{1}{2^n} (1-p)^{|\mathcal{E}|} \sum_{S\subseteq \mathcal{V}} \bigg(\frac{p}{1-p}\bigg)^{|\mathcal{E}(S, S^c)|}.
$$
Note that the above does not depend on $X$.
Also, $|\mathcal{Y}_{\mathcal{E}}|= 2^{n-1}$ as shown in the proof of Lemma \ref{lemma1}, thus
\begin{align*}
-\sum_{X \in \mathcal{Y}_{\mathcal{E}}} p(X)\log p(X)
&= - 2^{n-1} \cdot \frac{1}{2^n} (1-p)^{|\mathcal{E}|} \sum_{S\subseteq \mathcal{V}} \bigg(\frac{p}{1-p}\bigg)^{|\mathcal{E}(S, S^c)|} \cdot
\log \Bigg[\frac{1}{2^n} (1-p)^{|\mathcal{E}|} \sum_{S\subseteq \mathcal{V}} \bigg(\frac{p}{1-p}\bigg)^{|\mathcal{E}(S, S^c)|}\Bigg].
\end{align*}

Next, consider the case that $X \in \mathcal{X}\backslash\mathcal{Y}_{\mathcal{E}}$.
By the Fact \ref{fact1}, there exists $\pmb{y}'\in \mathcal{Y}$ such that
$X_{\mathcal{E}_1} = Y'_{\mathcal{E}_1} = (\pmb{y}'{\pmb{y}'}^T)_{\mathcal{E}_1}$. 
Also, since $X \in \mathcal{X}\backslash\mathcal{Y}_{\mathcal{E}}$, 
we have that $X_{\mathcal{E}_2}$ should not be equal to $(\pmb{y}'{\pmb{y}'}^T)_{\mathcal{E}_2}$.
By using these facts and the triangle inequality, for each $\pmb{y}\in \mathcal{Y}$, we can derive the following:
\begin{align*}
d(X, Y_\mathcal{E})
&\geq
| d(Y'_{\mathcal{E}}, Y_{\mathcal{E}})
- d(X, Y'_\mathcal{E}) |
\\&=
| d(Y'_{\mathcal{E}}, Y_{\mathcal{E}})
- d(X_{\mathcal{E}_2}, Y'_{\mathcal{E}_2}) |    
\\&=
\Big| \big| \mathcal{E} \big(S_{\pmb{y}'}(\pmb{y}), S_{\pmb{y}'}(\pmb{y})^c \big) \big|
- d(X_{\mathcal{E}_2}, Y'_{\mathcal{E}_2}) \Big|.
\end{align*}
Therefore,
\begin{align*}
p(X)
&=
\frac{1}{2^n} (1-p)^{|\mathcal{E}|} 
\sum_{\pmb{y}\in\mathcal{Y}} \bigg(\frac{p}{1-p}\bigg)^{d(X, Y_\mathcal{E})}
\\&\leq
\frac{1}{2^n} (1-p)^{|\mathcal{E}|} 
\sum_{\pmb{y}\in\mathcal{Y}} \bigg(\frac{p}{1-p}\bigg)^{\big| | \mathcal{E} (S_{\pmb{y}'}(\pmb{y}), S_{\pmb{y}'}(\pmb{y})^c ) |
- d(X_{\mathcal{E}_2}, Y'_{\mathcal{E}_2}) \big|}    
\\&=
\frac{1}{2^n} (1-p)^{|\mathcal{E}|} 
\sum_{S\subseteq \mathcal{V}} \bigg(\frac{p}{1-p}\bigg)^{\big| | \mathcal{E} (S, S^c ) |
- d(X_{\mathcal{E}_2}, Y'_{\mathcal{E}_2}) \big|}
\end{align*}
where the inequality holds because $p\in[0, \frac{1}{2}]$.
Now, note that $-p\log p$ increases as $p$ grows from $0$ to $e^{-1}$.
Also, we can see that
\begin{align*}
p(X)
&\leq
\frac{1}{2^n} (1-p)^{|\mathcal{E}|} 
\sum_{S\subseteq \mathcal{V}} \bigg(\frac{p}{1-p}\bigg)^{\big| | \mathcal{E} (S, S^c ) |
- d(X_{\mathcal{E}_2}, Y'_{\mathcal{E}_2}) \big|}
\\&\leq
\frac{1}{2^n} (1-p)^{|\mathcal{E}|} \sum_{S\subseteq \mathcal{V}} 1 = \frac{1}{2^n} (1-p)^{|\mathcal{E}|} \cdot 2^n = (1-p)^{|\mathcal{E}|}.
\end{align*}
Hence, if we assume that $(1-p)^{|\mathcal{E}|} \leq \frac{1}{e}$, 
we have
$$
-p(X)\log p(X) \leq -\frac{1}{2^n} (1-p)^{|\mathcal{E}|} \sum_{S\subseteq \mathcal{V}} \bigg(\frac{p}{1-p}\bigg)^{| | \mathcal{E} (S, S^c ) |
- d(X_{\mathcal{E}_2}, Y'_{\mathcal{E}_2}) |} \cdot
\log \Bigg[\frac{1}{2^n} (1-p)^{|\mathcal{E}|} \sum_{S\subseteq \mathcal{V}} \bigg(\frac{p}{1-p}\bigg)^{| | \mathcal{E} (S, S^c ) |
- d(X_{\mathcal{E}_2}, Y'_{\mathcal{E}_2}) |}\Bigg].
$$

We define
$\rho_{\mathcal{G}}(p, m) 
:= \sum_{S\subseteq \mathcal{V}} \big(\frac{p}{1-p}\big)^{| | \mathcal{E} (S, S^c ) |
- m |}$
and
$m_X = d(X_{\mathcal{E}_2}, Y'_{\mathcal{E}_2})$.
We can check that
$m_X$ has a value between $1$ and $|\mathcal{E}_2|=|\mathcal{E}|-n+1$ for any $X \in \mathcal{X}\backslash\mathcal{Y}_{\mathcal{E}}$,
and
for each $m\in[1,|\mathcal{E}|-n+1]$, there exist $\binom{|\mathcal{E}|-n+1}{m}$ different $X_{\mathcal{E}_2}$'s satisfying $m_X=m$.
Also, $2^{n-1}$ different $X$'s in $\mathcal{X}\backslash\mathcal{Y}_{\mathcal{E}}$
have the same $X_{\mathcal{E}_2}$.
In sum,
for each $m\in[1,|\mathcal{E}|-n+1]$,
there are
$2^{n-1}\cdot\binom{|\mathcal{E}|-n+1}{m}$ different $X$'s in $\mathcal{X}\backslash\mathcal{Y}_{\mathcal{E}}$.
Therefore, we can derive that 
\begin{align*}
-\sum_{X\in \mathcal{X}\backslash\mathcal{Y}_{\mathcal{E}}} p(X)\log p(X)
&\leq
-\sum_{X\in \mathcal{X}\backslash\mathcal{Y}_{\mathcal{E}}} 
\frac{1}{2^n} (1-p)^{|\mathcal{E}|} 
\rho_{\mathcal{G}}(p, m_X)
\cdot
\log \bigg[\frac{1}{2^n} (1-p)^{|\mathcal{E}|} 
\rho_{\mathcal{G}}(p, m_X)\bigg]
\\&=
-2^{n-1}\cdot\sum_{m=1}^{|\mathcal{E}|-n+1}
\binom{|\mathcal{E}|-n+1}{m}
\frac{1}{2^n} (1-p)^{|\mathcal{E}|} 
\rho_{\mathcal{G}}(p, m)
\cdot
\log \bigg[\frac{1}{2^n} (1-p)^{|\mathcal{E}|} 
\rho_{\mathcal{G}}(p, m)\bigg].
\end{align*}

Now we can find the upper bound of the entropy $H(X)$.
Note that when $m=0$,
\begin{gather*}
\binom{|\mathcal{E}|-n+1}{0}
\frac{1}{2^n} (1-p)^{|\mathcal{E}|} 
\rho_{\mathcal{G}}(p, 0)
\cdot
\log \bigg[\frac{1}{2^n} (1-p)^{|\mathcal{E}|} 
\rho_{\mathcal{G}}(p, 0)\bigg]
\\=
\frac{1}{2^n} (1-p)^{|\mathcal{E}|} \sum_{S\subseteq \mathcal{V}} \bigg(\frac{p}{1-p}\bigg)^{|\mathcal{E}(S, S^c)|} \cdot
\log \Bigg[\frac{1}{2^n} (1-p)^{|\mathcal{E}|} \sum_{S\subseteq \mathcal{V}} \bigg(\frac{p}{1-p}\bigg)^{|\mathcal{E}(S, S^c)|}\Bigg].
\end{gather*}
Therefore,
\begin{align*}
H(X) 
&= - \sum_{X\in\mathcal{X}} p(X)\log p(X)
= - \sum_{X \in \mathcal{Y}_{\mathcal{E}}} p(X)\log p(X) - \sum_{X \in \mathcal{X}\backslash\mathcal{Y}_{\mathcal{E}}} p(X)\log p(X) 
\\&\leq 
-2^{n-1} \cdot \frac{1}{2^n} (1-p)^{|\mathcal{E}|} \sum_{S\subseteq \mathcal{V}} \bigg(\frac{p}{1-p}\bigg)^{|\mathcal{E}(S, S^c)|} \cdot
\log \Bigg[\frac{1}{2^n} (1-p)^{|\mathcal{E}|} \sum_{S\subseteq \mathcal{V}} \bigg(\frac{p}{1-p}\bigg)^{|\mathcal{E}(S, S^c)|}\Bigg]
\\&~~~~ 
-2^{n-1}\cdot\sum_{m=1}^{|\mathcal{E}|-n+1}
\binom{|\mathcal{E}|-n+1}{m}
\frac{1}{2^n} (1-p)^{|\mathcal{E}|} 
\rho_{\mathcal{G}}(p, m)
\cdot
\log \bigg[\frac{1}{2^n} (1-p)^{|\mathcal{E}|} 
\rho_{\mathcal{G}}(p, m)\bigg]
\\&=
-2^{n-1}\cdot\sum_{m=0}^{|\mathcal{E}|-n+1}
\binom{|\mathcal{E}|-n+1}{m}
\frac{1}{2^n} (1-p)^{|\mathcal{E}|} 
\rho_{\mathcal{G}}(p, m)
\cdot
\log \bigg[\frac{1}{2^n} (1-p)^{|\mathcal{E}|} 
\rho_{\mathcal{G}}(p, m)\bigg].
\end{align*}
The above inequality holds under the assumption that $(1-p)^{|\mathcal{E}|} \leq \frac{1}{e}$.

Now, to see the connection between the above bound and the Cheeger constant $\phi_{\mathcal{G}}$,
we can take a closer look at $\rho_{\mathcal{G}}(p, m)$. Note that $\phi_{\mathcal{G}} = \underset{S\subseteq \mathcal{V}, 1\leq|S|\leq \lfloor\frac{n}{2}\rfloor}{\min} \frac{|\mathcal{E}(S, S^c)|}{|S|}$. Then,
\begin{align*}
\rho_{\mathcal{G}}(p, m)
&= 
\sum_{S\subseteq \mathcal{V}} \bigg(\frac{p}{1-p}\bigg)^{| | \mathcal{E} (S, S^c ) |
- m |}
\\&= 
2\cdot\bigg(\frac{p}{1-p}\bigg)^m
+
\sum_{S\subseteq \mathcal{V}, 1\leq|S|\leq \lfloor\frac{n}{2}\rfloor}
\big(2-\mathbb{I}\big[|S|=n/2, n ~\text{is even}\big]\big)\bigg(\frac{p}{1-p}\bigg)^{| | \mathcal{E} (S, S^c ) |
- m |}
\\&\leq
2\cdot\bigg(\frac{p}{1-p}\bigg)^m
+
\sum_{S\subseteq \mathcal{V}, 1\leq|S|\leq \lfloor\frac{n}{2}\rfloor}
\big(2-\mathbb{I}\big[|S|=n/2, n ~\text{is even}\big]\big)\bigg(\frac{p}{1-p}\bigg)^{( \phi_{\mathcal{G}}|S|
- m )\vee 0}
\\&=
2\cdot\bigg(\frac{p}{1-p}\bigg)^m
+
\sum_{k=1}^{\lfloor\frac{n}{2}\rfloor} 
\big(2-\mathbb{I}\big[k=n/2, n ~\text{is even}\big]\big)
\binom{n}{k}
\bigg(\frac{p}{1-p}\bigg)^{( \phi_{\mathcal{G}}k
- m )\vee 0}
\\&=
2\cdot\bigg(\frac{p}{1-p}\bigg)^m
+
\sum_{k=1}^{n-1} 
\binom{n}{k}
\bigg(\frac{p}{1-p}\bigg)^{\big( \phi_{\mathcal{G}} \cdot\{k\wedge (n-k)\}
- m \big)\vee 0}
= \tau(m)
\end{align*}
where the inequality holds because $|\cdot| \geq (\cdot)\vee 0$, 
$|\mathcal{E}(S, S^c)| = \frac{|\mathcal{E}(S, S^c)|}{|S|}\cdot|S| \geq \phi_{\mathcal{G}} |S|$ when $1\leq|S|\leq \lfloor\frac{n}{2}\rfloor$, and
$a\vee 0 \geq b\vee 0$ if $a\geq b$.

Now, if we can derive that
\begin{align*}
\frac{1}{2^n} (1-p)^{|\mathcal{E}|} 
\rho_{\mathcal{G}}(p, m)
\leq
\frac{1}{2^n} (1-p)^{|\mathcal{E}|} 
\tau(m)
\leq
\frac{1}{e},
\end{align*}
then by using the fact that $-p\log p$ is increasing between $0$ and $\frac{1}{e}$, we have the inequality
$$
-\frac{1}{2^n} (1-p)^{|\mathcal{E}|} 
\rho_{\mathcal{G}}(p, m)
\log \bigg[\frac{1}{2^n} (1-p)^{|\mathcal{E}|} 
\rho_{\mathcal{G}}(p, m) \bigg]
\leq
-\frac{1}{2^n} (1-p)^{|\mathcal{E}|} 
\tau(m)
\log \bigg[\frac{1}{2^n} (1-p)^{|\mathcal{E}|} 
\tau(m)\bigg].
$$
Note that
\begin{gather*}
\tau(m) 
= 
2\cdot\bigg(\frac{p}{1-p}\bigg)^m +
\sum_{k=1}^{n-1} \binom{n}{k} \bigg(\frac{p}{1-p}\bigg)^{\big( \phi_{\mathcal{G}} \cdot\{k\wedge (n-k)\} - m \big)\vee 0} 
\leq 
2 + \sum_{k=1}^{n-1} \binom{n}{k}
= 2^n
\end{gather*}
and if the assumption $(1-p)^{|\mathcal{E}|} \leq \frac{1}{e}$ holds,
we can derive the desired inequality
\begin{align*}
\frac{1}{2^n} (1-p)^{|\mathcal{E}|} \tau(m) 
\leq
\frac{1}{2^n} (1-p)^{|\mathcal{E}|} \cdot 2^n
=
(1-p)^{|\mathcal{E}|}
\leq
\frac{1}{e}.
\end{align*}
Therefore, we have
\begin{align*}
H(X) 
&\leq 
-2^{n-1}\cdot\sum_{m=0}^{|\mathcal{E}|-n+1} \binom{|\mathcal{E}|-n+1}{m} \frac{1}{2^n} (1-p)^{|\mathcal{E}|} \rho_{\mathcal{G}}(p, m) \cdot \log \bigg[\frac{1}{2^n} (1-p)^{|\mathcal{E}|}  \rho_{\mathcal{G}}(p, m)\bigg]
\\&\leq 
-2^{n-1}\cdot\sum_{m=0}^{|\mathcal{E}|-n+1} \binom{|\mathcal{E}|-n+1}{m} \frac{1}{2^n} (1-p)^{|\mathcal{E}|} \tau(m) \cdot \log \bigg[\frac{1}{2^n} (1-p)^{|\mathcal{E}|} \tau(m)\bigg]
\\&=
-2^{n-1}\cdot2^{|\mathcal{E}|-n+1} (1-p)^{|\mathcal{E}|} \sum_{m=0}^{|\mathcal{E}|-n+1} \binom{|\mathcal{E}|-n+1}{m} \frac{1}{2^{|\mathcal{E}|-n+1}} \cdot \frac{\tau(m)}{2^n} \Big\{|\mathcal{E}|\log(1-p) + \log \frac{\tau(m)}{2^n} \Big\}
\\&=
-|\mathcal{E}|\log(1-p) \{2(1-p)\}^{|\mathcal{E}|}  \mathbb{E}_B \bigg[\frac{\tau(B)}{2^n}\bigg] + \{2(1-p)\}^{|\mathcal{E}|} \mathbb{E}_B \bigg[-\frac{\tau(B)}{2^n} \log\Big(\frac{\tau(B)}{2^n}\Big)\bigg]
\\&=
\kappa_1
\end{align*}
where $B\sim \text{Bin}(|\mathcal{E}|-n+1,~ \frac{1}{2})$.
Now, we can derive the upper bound of the mutual information as follows:
\begin{align*}
\mathbb{I}(\pmb{y}, X) 
&= H(X) -|\mathcal{E}| H^*(p)
\\&\leq \big\{|\mathcal{E}|\log 2 \wedge \kappa_1 \big\}\cdot \mathbb{I}\big[(1-p)^{|\mathcal{E}|} \leq e^{-1}\big]
+ |\mathcal{E}|\log 2 \cdot \big(1-\mathbb{I}\big[(1-p)^{|\mathcal{E}|} \leq e^{-1}\big]\big)
-|\mathcal{E}| H^*(p)
\\&=
\Big\{\big(\kappa_1-|\mathcal{E}|\log 2\big)\wedge 0 \Big\}\cdot \mathbb{I}\big[(1-p)^{|\mathcal{E}|} \leq e^{-1}\big]
+ |\mathcal{E}|\log 2
-|\mathcal{E}| H^*(p)
\\&=|\mathcal{E}|\log 2-|\mathcal{E}| H^*(p)
-\Big\{\big(|\mathcal{E}|\log 2-\kappa_1\big)\vee 0 \Big\}\cdot \mathbb{I}\big[(1-p)^{|\mathcal{E}|} \leq e^{-1}\big]
.
\end{align*}
Then, by Fano's inequality, we have
\begin{align}
\mathbb{P}(\hat{\pmb{y}}\neq \pmb{y}^*)
&\geq 1-\frac{|\mathcal{E}|\log 2-|\mathcal{E}| H^*(p) + \log 2}{n\log 2}
+\frac{\big(|\mathcal{E}|\log 2-\kappa_1\big)\vee 0 }{n\log 2}\cdot \mathbb{I}\big[(1-p)^{|\mathcal{E}|} \leq e^{-1}\big] \nonumber
\\&=g_1 + \frac{\big(|\mathcal{E}|\log 2-\kappa_1\big)\vee 0 }{n\log 2}\cdot \mathbb{I}\big[(1-p)^{|\mathcal{E}|} \leq e^{-1}\big]
=g_1^*.
\label{eq:minimax1_fano_proof2}
\end{align}
From (\ref{eq:minimax1_fano_proof1}) and (\ref{eq:minimax1_fano_proof2}), we obtain the desired result.

\subsection{Proof of Theorem \protect\ref{theorem2}}
\label{proof_theorem2}

In this section, we write $Y = Y_\mathcal{E}=[\pmb{y}\pmb{y}^T]_\mathcal{E}$ for each $\pmb{y}\in \mathcal{Y}$ and $Y^* = Y^*_\mathcal{E} = [\pmb{y}^*(\pmb{y}^*)^{T}]_\mathcal{E} = [-\pmb{y}^*(-\pmb{y}^*)^{T}]_\mathcal{E}$,
for simplicity.
Also, we write that 
$$
\pmb{y}^T X \pmb{y} 
= tr(X\pmb{y}\pmb{y}^T)
= tr(XY)
=: \langle X, Y \rangle.
$$
Now we define 
\begin{align*}
\Delta_X(Y) 
&:= \pmb{y}^{*T} X \pmb{y}^* - \pmb{y}^T X \pmb{y} 
= \langle X, Y^* - Y \rangle
= \langle \mathbb{E}[X], Y^* - Y \rangle + \langle X-\mathbb{E}[X], Y^* - Y \rangle.
\end{align*}
Then our goal becomes to find the upper bound of the probability
$\mathbb{P}\big[\exists Y\in{\mathcal{Y}_\mathcal{E}}\backslash\{Y^*\} ~\text{s.t.}~ \Delta_X(Y) \leq 0\big]$.

We will first find the upper bound of the probability $\mathbb{P} \big[\Delta_X(Y) \leq 0 \big]$ for a fixed $Y$.
Note that
$$
\mathbb{E}[X_{ij}] = Y^*_{ij}(1-p) - Y^*_{ij} p = Y^*_{ij} (1-2p),
$$
i.e., 
$\mathbb{E}[X] = Y^* (1-2p)$.
Then $\langle \mathbb{E}[X], Y^* - Y \rangle$ can be written as
\begin{align*}
\langle \mathbb{E}[X], Y^* - Y \rangle 
&= 
(1-2p) \langle Y^*, Y^* - Y \rangle
=
(1-2p) \sum_{(i,j)\in\mathcal{E}} Y^*_{ij} \cdot(Y^*_{ij} - Y_{ij})
\\&=
(1-2p) \Bigg[
\sum_{(i,j)\in\mathcal{E}:\tiny\begin{cases} Y^*_{ij}= 1 \\ Y_{ij}= -1 \end{cases}}
2
+ \sum_{(i,j)\in\mathcal{E}:\tiny\begin{cases} Y^*_{ij}= -1 \\ Y_{ij}= 1 \end{cases}}
(-1)\cdot (-2) \Bigg]
=
(1-2p) \cdot 2 \cdot 
\mathop{\sum_{(i,j)\in\mathcal{E}:}}_{Y^*_{ij}\neq Y_{ij}} 1
\\&= 2(1-2p)N(Y)
\end{align*}
where 
$N(Y) = \big|\big\{ (i,j)\in\mathcal{E} : Y^*_{ij} \neq Y_{ij}\big\}\big|$.

Also, $\langle X-\mathbb{E}[X], Y^* - Y \rangle
= \langle X-Y^* (1-2p), Y^* - Y \rangle$ can be represented by
\begin{align*}
\sum_{(i,j)\in\mathcal{E}} \big(X_{ij}-Y^*_{ij} (1-2p)\big) \cdot(Y^*_{ij} - Y_{ij})
&=
\mathop{\sum_{(i,j)\in\mathcal{E}:}}_{Y^*_{ij}\neq Y_{ij}}
\big(X_{ij}-Y^*_{ij} (1-2p)\big) \cdot(Y^*_{ij} - Y_{ij})
=: T_X(Y).
\end{align*}
Here, we can check that $T_X(Y)$ is the summation of $N(Y)$ i.i.d. binary random variables which have mean zero and variance $16p(1-p)$ and are bounded by $4(1-p)$.
Then, by Bernstein inequality (inequality (2.10) in \citep{boucheron2013concentration}),
\begin{align*}
\mathbb{P} \bigg[ -T_X(Y) \geq 2(1-2p) N(Y) \bigg]
&\leq 
\exp \bigg[ - \frac{2(1-2p)^2 \{N(Y)\}^2 }{ 16p(1-p)N(Y) + \frac{8}{3}(1-p)(1-2p)N(Y)} \bigg]
\\[0.5em]&= 
\exp \bigg[ - \frac{ (1-2p)^2N(Y) }{ \frac{4}{3}(1-p) (1 + 4p)} \bigg],
\end{align*}
that is,
\begin{align*}
\mathbb{P} \big[\Delta_X(Y) \leq 0 \big]
=
\mathbb{P} \bigg[ T_X(Y) + 2(1-2p) N(Y)  \leq 0 \bigg]
\leq 
\exp \bigg[ - \frac{ (1-2p)^2N(Y) }{ \frac{4}{3}(1-p) (1 + 4p)} \bigg]
=: h_1\big(p, N(Y)\big)
\end{align*}
for each $Y \neq Y^*$.

Note that $N(Y) = \big|\big\{ (i,j)\in\mathcal{E} : Y^*_{ij} \neq Y_{ij}\big\}\big| = \big|\mathcal{E}\big( S_{\pmb{y}^*}(\pmb{y}), S_{\pmb{y}^*}(\pmb{y})^c  \big)\big|$
for $\pmb{y}$ such that $Y = \pmb{y}\pmb{y}^T$ by the Fact \ref{fact2}.
Then, we can find the upper bound of the following probability:
\begin{align*}
\mathbb{P}\big[\exists Y\in{\mathcal{Y}_\mathcal{E}}\backslash\{Y^*\} ~\text{s.t.}~ \Delta_X(Y) \leq 0\big]
&\leq 
\sum_{Y\in{\mathcal{Y}_\mathcal{E}}} \mathbb{I}[Y \neq Y^*]\cdot \mathbb{P}\big[\Delta_X(Y) \leq 0 \big]
\\&\leq 
\sum_{Y\in{\mathcal{Y}_\mathcal{E}}} \mathbb{I}[Y \neq Y^*]\cdot h_1(p, N(Y))
\\&=
\frac{1}{2} \sum_{\pmb{y}\in \mathcal{Y}} \mathbb{I}[\pmb{y} \neq \pmb{y}^*, -\pmb{y}^*]\cdot
h_1\big(p, \big|\mathcal{E}\big( S_{\pmb{y}^*}(\pmb{y}), S_{\pmb{y}^*}(\pmb{y})^c  \big)\big|\big)
\\&= 
\frac{1}{2}\sum_{\pmb{y}\in \mathcal{Y}} \mathbb{I}\big[|S_{\pmb{y}^*}(\pmb{y})|\neq 0 ~\text{or}~ n\big] \cdot
h_1\big(p, \big|\mathcal{E}\big( S_{\pmb{y}^*}(\pmb{y}), S_{\pmb{y}^*}(\pmb{y})^c  \big)\big|\big)
\\&= 
\frac{1}{2}\mathop{\sum_{S\subseteq \mathcal{V}:}}_{1\leq|S|\leq n-1} h_1\big(p, \big|\mathcal{E}\big( S, S^c  \big)\big|\big)
\end{align*}
where the first inequality holds by the union bound and the last equality holds by the Fact \ref{fact4}.

Since $h_1(p, \cdot)$ is decreasing, we have
$$
h_1\big(p, \big|\mathcal{E}\big( S, S^c  \big)\big|\big) = h_1\Big(p,~ \frac{|\mathcal{E}( S, S^c )|}{|S|}\cdot |S|\Big) 
\leq h_1(p, \phi_{\mathcal{G}} |S|)
$$
for $1\leq|S|\leq \lfloor \frac{n}{2}\rfloor$.
Therefore,
\begin{align*}
\frac{1}{2}\mathop{\sum_{S\subseteq \mathcal{V}:}}_{1\leq|S|\leq n-1} h_1\big(p, \big|\mathcal{E}\big( S, S^c  \big)\big|\big)
&\leq 
\mathop{\sum_{S\subseteq \mathcal{V}:}}_{1\leq|S|\leq \lfloor \frac{n}{2}\rfloor}  
h_1\big(p, \big|\mathcal{E}\big( S, S^c  \big)\big|\big)
\leq 
\mathop{\sum_{S\subseteq \mathcal{V}:}}_{1\leq|S|\leq \lfloor \frac{n}{2}\rfloor} 
h_1(p, \phi_{\mathcal{G}} |S|)
=
\sum_{k=1}^{\lfloor \frac{n}{2}\rfloor} \binom{n}{k} 
h_1(p, \phi_{\mathcal{G}}k).
\end{align*}
In conclusion, we have the lower bound of the probability of success of the MLE approach as follows:
\begin{align*}
\mathbb{P}\big[\Delta_X(Y) > 0~\text{for all}~ Y\in{\mathcal{Y}_\mathcal{E}}\backslash\{Y^*\}\big]
&=
1 - \mathbb{P}\big[\exists Y\in{\mathcal{Y}_\mathcal{E}}\backslash\{Y^*\} ~\text{s.t.}~ \Delta_X(Y) \leq 0\big]
\\&\geq 
1 - \sum_{k=1}^{\lfloor \frac{n}{2}\rfloor} \binom{n}{k} h_1(p, \phi_{\mathcal{G}}k).
\end{align*}

\subsection{Detailed Description of Comparison to Tractable Algorithm in Section \protect\ref{subsec:mle1}}
\label{proof_comparison1}

Let $r = \exp \Big[ - \frac{ (1-2p)^2 \phi_{\mathcal{G}} }{ \frac{4}{3}(1-p) (1 + 4p)} \Big]$ and $s=\frac{r}{r+1}$ ($\Leftrightarrow r=\frac{s}{1-s}$.)
Then we can write that
\begin{align*}
\sum_{k=1}^{\lfloor \frac{n}{2}\rfloor} \binom{n}{k} h_1(p, \phi_{\mathcal{G}}k)
=
\sum_{k=0}^{\lfloor \frac{n}{2}\rfloor} \binom{n}{k} r^k - 1
= 
\sum_{k=0}^{\lfloor \frac{n}{2}\rfloor} \binom{n}{k} s^k (1-s)^{n-k} (1-s)^{-n} - 1
<
(1-s)^{-n} - 1
=
(r+1)^n -1.
\end{align*}

If $\phi_{\mathcal{G}}=\Omega(n)$, then $r\rightarrow 0$ as $n\rightarrow \infty$. From this fact, we can derive that 
$(r+1)^n -1 \approx e^{nr} - 1$ and thus
$nr \geq \log \big[ \sum_{k=1}^{\lfloor \frac{n}{2}\rfloor} \binom{n}{k} h_1(p, \phi_{\mathcal{G}}k) + 1 \big]$ for sufficiently large $n$.
Since $\sum_{k=1}^{\lfloor \frac{n}{2}\rfloor} \binom{n}{k} h_1(p, \phi_{\mathcal{G}}k) \underset{n\rightarrow \infty}{\longrightarrow} 0$ when $\phi_{\mathcal{G}}=\Omega(n)$, there exists a positive constant $C'$ such that
$$
nr 
\geq \log \Bigg[ \sum_{k=1}^{\lfloor \frac{n}{2}\rfloor} \binom{n}{k} h_1(p, \phi_{\mathcal{G}}k) + 1 \Bigg]
\geq C' \cdot \sum_{k=1}^{\lfloor \frac{n}{2}\rfloor} \binom{n}{k} h_1(p, \phi_{\mathcal{G}}k)
$$
for sufficiently large $n$.

Note that
\begin{align*}
nr 
&= n\cdot \exp \Bigg[ - \frac{ (1-2p)^2 \phi_{\mathcal{G}} }{ \frac{4}{3}(1-p) (1 + 4p)} \Bigg] \\
&= \frac{\epsilon_1}{2} \cdot \exp \Bigg[ \frac{3(1-2p)^2 \phi_{\mathcal{G}}^4}{1536\Delta_{\max}^3 p(1-p)+32(1-2p)(1-p)\phi_{\mathcal{G}}^2\Delta_{\max}} - \frac{ (1-2p)^2 \phi_{\mathcal{G}} }{ \frac{4}{3}(1-p) (1 + 4p)} \Bigg] \\
&= \frac{\epsilon_1}{2} \cdot \exp(-C_{n,p}\phi_{\mathcal{G}})
\end{align*}
where
$C_{n,p} = \frac{3(1-2p)^2}{4(1-p) (1 + 4p)} - \frac{3(1-2p)^2 \phi_{\mathcal{G}}^3}{1536\Delta_{\max}^3 p(1-p)+32(1-2p)(1-p)\phi_{\mathcal{G}}^2\Delta_{\max}}
= \frac{3(1-2p)^2}{4(1-p)}\cdot \bigg[\frac{1}{1+4p} - \frac{1}{8\cdot \frac{\Delta_{\max}}{\phi_{\mathcal{G}}} \cdot \big\{ 48p(\frac{\Delta_{\max}}{\phi_{\mathcal{G}}})^2 + (1-2p) \big\}} \bigg]$.
Since $\frac{\Delta_{\max}}{\phi_{\mathcal{G}}} \geq 1$, we can easily check that
$8\cdot \frac{\Delta_{\max}}{\phi_{\mathcal{G}}} \cdot \big\{ 48p(\frac{\Delta_{\max}}{\phi_{\mathcal{G}}})^2 + (1-2p) \big\} 
\geq 8(1+4p)$ and hence $C_{n,p} \geq \frac{21(1-2p)^2}{32(1-p)(1+4p)}$.
Now, we obtain
$$
\sum_{k=1}^{\lfloor \frac{n}{2}\rfloor} \binom{n}{k} h_1(p, \phi_{\mathcal{G}}k)
\leq C''\cdot nr
= \frac{C''}{2} \cdot \epsilon_1 \exp(-C_{n,p}\phi_{\mathcal{G}})
\leq \frac{C''}{2} \cdot \epsilon_1 \exp\bigg(-\frac{21(1-2p)^2}{32(1-p)(1+4p)}\cdot C''' n\bigg)
$$
for sufficiently large $n$ and some positive constants $C''$ and $C'''$.
Therefore, if $\phi_{\mathcal{G}}=\Omega(n)$,
$$
\frac{\sum_{k=1}^{ _\lfloor \frac{n}{2} _\rfloor} \binom{n}{k} h_1(p, \phi_{\mathcal{G}}k)}{\epsilon_1} = O\bigg(\exp\bigg(-\frac{C(1-2p)^2n}{(1-p)(1+4p)}\bigg)\bigg)
$$
for some positive constant $C$.

\subsection{Proof of Corollary \protect\ref{corollary2}}

Let $r = \exp \Big[ - \frac{ (1-2p)^2 \phi_{\mathcal{G}} }{ \frac{4}{3}(1-p) (1 + 4p)} \Big]$.
In the previous Section \ref{proof_comparison1}, we showed that 
\begin{align*}
\sum_{k=1}^{\lfloor \frac{n}{2}\rfloor} \binom{n}{k} h_1(p, \phi_{\mathcal{G}}k)
<
(r+1)^n -1.
\end{align*}

If the condition $\phi_{\mathcal{G}}\cdot \frac{(1-2p)^2}{(1-p) (1 + 4p)} \geq \frac{8}{3} \log n$ holds, we can derive that
\begin{align*}
\frac{ (1-2p)^2 \phi_{\mathcal{G}} }{ \frac{4}{3}(1-p) (1 + 4p)} 
&\geq 2\log n
\\ &\geq
\log n - \log (3^{\frac{1}{n}}-1)
\\& = -\log\bigg[\frac{3^{\frac{1}{n}}-1}{2}\cdot 2n^{-1}\bigg]
\\&\geq
-\log \big[ (2n^{-1} + 1)^{\frac{1}{n}} -1 \big]
\end{align*}
where the second inequality holds because for any $n\in\mathbb{N}$,
\begin{align*}
3 \geq e \geq (1+n^{-1})^{n}
~\Rightarrow~
3^{\frac{1}{n}} \geq 1+n^{-1}
~\Rightarrow~
\log(3^{\frac{1}{n}}-1) \geq -\log n,
\end{align*}
and the last inequality holds because the function $(x+1)^{\frac{1}{n}} - 1$ is concave and satisfies
\begin{align*}
(x+1)^{\frac{1}{n}} - 1
\geq
\frac{(2+1)^{\frac{1}{n}} - 1 - (0+1)^{\frac{1}{n}} + 1}{2-0}\cdot x
= \frac{3^{\frac{1}{n}}-1}{2}\cdot x
\end{align*}
on $x\in[0,2]$. By letting $x=2n^{-1}$, we obtain the last inequality.

Consequently,
$$
r 
= \exp \bigg[ - \frac{ (1-2p)^2 \phi_{\mathcal{G}} }{ \frac{4}{3}(1-p) (1 + 4p)} \bigg]
\leq
(2n^{-1} + 1)^{\frac{1}{n}} -1
~\Rightarrow~
(r+1)^n - 1 \leq 2n^{-1}.
$$
Therefore,
$$
1 - \sum_{k=1}^{\lfloor \frac{n}{2}\rfloor} \binom{n}{k} h_1(p, \phi_{\mathcal{G}}k)
>
1-2n^{-1}.
$$

\newpage
\subsection{Illustration of Bounds of Probabilities for Additional Examples of Graphs in Section \protect\ref{sec:case_that_only_edge}}
\label{appendix_illustration1}

\begin{figure}[h]
\vskip 0.2in
\begin{center}
\centerline{\includegraphics[width=\columnwidth]{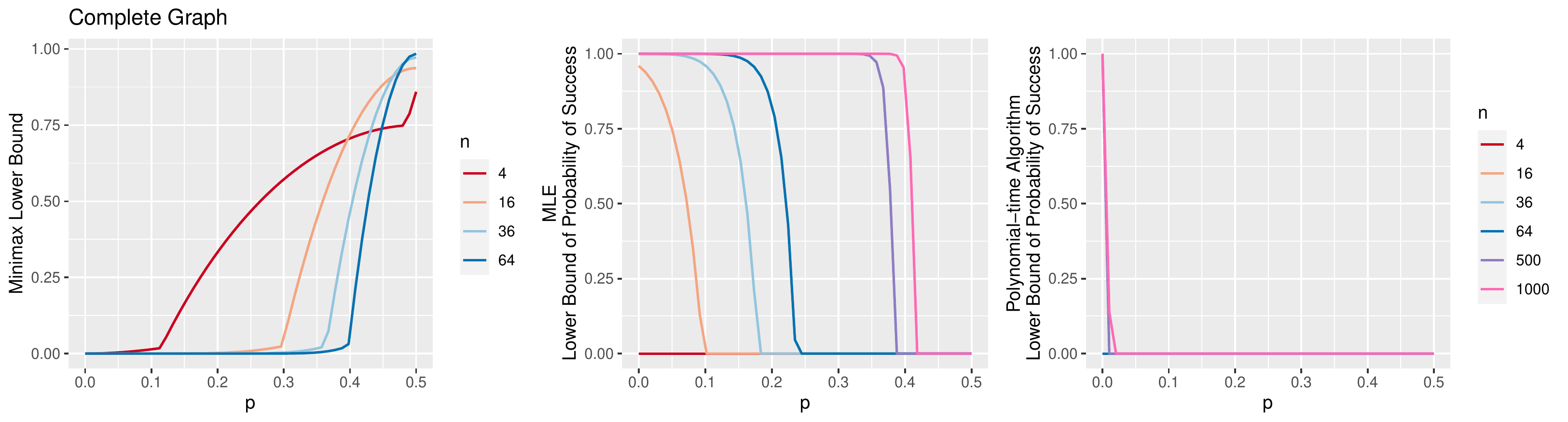}}
\caption{
Minimax lower bound (left), lower bound of the probability of success of the MLE algorithm (center), and lower bound of the probability of success of the polynomial-time algorithm (right) for complete graphs.
}
\label{fig:appendix_illustration1}
\end{center}
\vskip -0.2in
\end{figure}

\begin{figure}[h]
\vskip 0.2in
\begin{center}
\centerline{\includegraphics[width=4in]{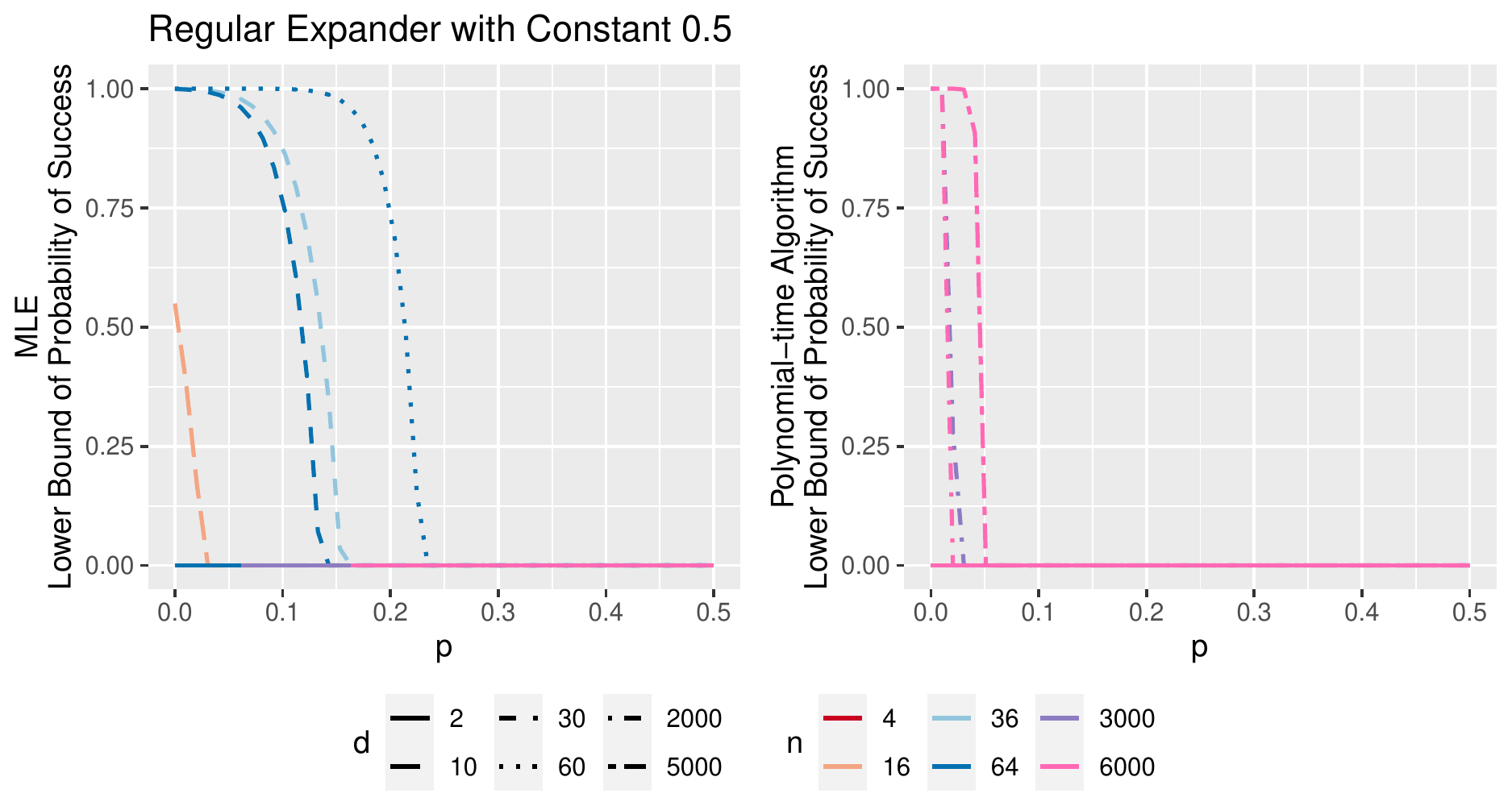}}
\caption{Lower bound of the probability of success of the MLE algorithm (left) and lower bound of the probability of success of the polynomial-time algorithm (right) for regular expanders.
}
\label{fig:appendix_illustration2}
\end{center}
\vskip -0.2in
\end{figure}

\subsection{Proof of Theorem \protect\ref{theorem3}}
\label{proof_theorem3}

In this section, we denote $(X, \pmb{c})$ by $Z$, where $Z\in\mathcal{Z}:=\mathcal{X}\times\mathcal{C}$.
We can prove Theorem \ref{theorem3} by showing the following lemmas:

\begin{lemma}
\label{theorem2_lemma1}
Under the same conditions as in Theorem \ref{theorem3}, we have
$
\underset{\hat{\pmb{y}}:\mathcal{Z} \to \mathcal{Y}}{\inf}
~\underset{ P \in\mathcal{P} }{\sup}
~\mathbb{P}_{ (\pmb{y}^*, Z) \sim P } \big[ \hat{\pmb{y}} (Z) \neq \pmb{y}^* \big] 
\geq f_2.
$
\end{lemma}

\begin{lemma}
\label{theorem2_lemma2}
Under the same conditions as in Theorem \ref{theorem3}, we have
$
\underset{\hat{\pmb{y}}:\mathcal{Z} \to \mathcal{Y}}{\inf}
~\underset{ P \in\mathcal{P} }{\sup}
~\mathbb{P}_{ (\pmb{y}^*, Z) \sim P } \big[ \hat{\pmb{y}} (Z) \neq \pmb{y}^* \big] 
\geq \max\{ g_2, g_2^*\}.
$
\end{lemma}

\subsubsection{{Proof of Lemma \protect\ref{theorem2_lemma1}}}

As in the proof of Lemma \ref{theorem1_lemma1}, we apply Lemma \ref{lemma2} and consider the zero-one distance 
$\delta(\pmb{y}, \pmb{y}') = \mathbb{I}[\pmb{y}\neq \pmb{y}']$
and the pseudo-distance $\delta_j(\pmb{y}, \pmb{y}') = \frac{1}{n}\mathbb{I}[\pmb{y}\neq \pmb{y}']$ for $j=1,\cdots,n$.
Then we can derive that 
$$
\underset{\pmb{y}^*\in\mathcal{Y}}{\max}~\mathbb{P}_{Z|\pmb{y}^*}[\hat{\pmb{y}}(Z)\neq \pmb{y}^*]
\geq
\frac{1}{2} 
\min\bigg\{\|P_{Z|\pmb{y}} \wedge P_{Z|\pmb{y}'} \| : \sum_{i=1}^n \mathbb{I}[y_i \neq y'_i] = 1,~ \pmb{y}, \pmb{y}' \in \mathcal{Y} \bigg\}.
$$

Denote the set of edges which are connected to the $k$th node by $\mathcal{E}_k$.
For $\pmb{y}, \pmb{y}'$ such that
$y_k\neq y'_k$ and $y_i= y'_i$ for $\forall i \neq k$,
we can obtain that
\begin{align*}
&p(Z|\pmb{y}) \wedge p(Z|\pmb{y}')
\\&= 
\bigg\{\prod_{(i,j) \in \mathcal{E}} p(X_{ij}|y_iy_j)
\cdot \prod_{i\in\mathcal{V}} p(c_i|y_i) \bigg\}\wedge
\bigg\{\prod_{(i,j) \in \mathcal{E}} p(X_{ij}|y'_iy'_j)
\cdot \prod_{i\in\mathcal{V}} p(c_i|y'_i) \bigg\}
\\&=
\bigg\{\prod_{(i,j) \notin \mathcal{E}_k} p(X_{ij}|y_iy_j)
\cdot
\prod_{(i,j) \in \mathcal{E}_k} p(X_{ij}|y_iy_j)
\cdot \prod_{i\neq k} p(c_i|y_i) \cdot p(c_k|y_k)
\bigg\}
\\&~~~~~\wedge
\bigg\{\prod_{(i,j) \notin \mathcal{E}_k} p(X_{ij}|y_iy_j)
\cdot
\prod_{(i,j) \in \mathcal{E}_k} p(X_{ij}|y'_iy'_j)
\cdot \prod_{i\neq k} p(c_i|y_i)\cdot p(c_k|y'_k)
\bigg\}
\\&=
\bigg\{\prod_{(i,j) \notin \mathcal{E}_k} p(X_{ij}|y_iy_j) \cdot \prod_{i\neq k} p(c_i|y_i)\bigg\}
\times
\Bigg[
\bigg\{
\prod_{(i,j) \in \mathcal{E}_k} 
p(X_{ij}|y_iy_j) \cdot p(c_k|y_k)
\bigg\}
\wedge 
\bigg\{
\prod_{(i,j) \in \mathcal{E}_k} p(X_{ij}|y'_iy'_j)
 \cdot p(c_k|y'_k)
\bigg\}
\Bigg]
\\&=
\bigg\{\prod_{(i,j) \notin \mathcal{E}_k} p(X_{ij}|y_iy_j) \cdot \prod_{i\neq k} p(c_i|y_i)\bigg\}
\\&~~~~~\times
\Bigg[
\bigg\{
\prod_{(i,j) \in \mathcal{E}_k} 
p(X_{ij}|y_iy_j) \cdot p(c_k|y_k)
\bigg\}
\wedge 
\bigg\{
\prod_{(i,j) \in \mathcal{E}_k} \Big(1- p(X_{ij}|y_iy_j)\Big)
 \cdot \Big(1-p(c_k|y_k)\Big)
\bigg\}
\Bigg]
\\&=
\bigg\{\prod_{(i,j) \notin \mathcal{E}_k} p(X_{ij}|y_iy_j) \cdot \prod_{i\neq k} p(c_i|y_i)\bigg\}
\times
\Bigg[
\bigg\{
p^{d(X_{\mathcal{E}_k}, Y_{\mathcal{E}_k})} \cdot (1-p)^{\Delta_k - d(X_{\mathcal{E}_k}, Y_{\mathcal{E}_k})}
\cdot q^{\mathbb{I}[c_k\neq y_k]} \cdot (1-q)^{1-\mathbb{I}[c_k\neq y_k]}
\bigg\}
\\&~~~~~~~~~~~~~~~~~~~~~~~~~~~~~~~~~~~~~~~~~~~~~~~~~~~~~~\wedge 
\bigg\{
(1-p)^{d(X_{\mathcal{E}_k}, Y_{\mathcal{E}_k})} \cdot p^{\Delta_k - d(X_{\mathcal{E}_k}, Y_{\mathcal{E}_k})}
\cdot (1-q)^{\mathbb{I}[c_k\neq y_k]} \cdot q^{1-\mathbb{I}[c_k\neq y_k]}
\bigg\}
\Bigg].
\end{align*}
Hence,
\begin{align*}
&\|P_{Z|\pmb{y}} \wedge P_{Z|\pmb{y}'} \| 
= 
\sum_{X\in\mathcal{X}}\sum_{\pmb{c}\in\mathcal{C}} p(Z|\pmb{y}) \wedge p(Z|\pmb{y}')
\\&=
\sum_{X\in\mathcal{X}}\sum_{\pmb{c}\in\mathcal{C}}
\bigg\{\prod_{(i,j) \notin \mathcal{E}_k} p(X_{ij}|y_iy_j) \cdot \prod_{i\neq k} p(c_i|y_i)\bigg\}
\\&~~~~~\times
\Bigg[
\bigg\{
p^{d(X_{\mathcal{E}_k}, Y_{\mathcal{E}_k})} \cdot (1-p)^{\Delta_k - d(X_{\mathcal{E}_k}, Y_{\mathcal{E}_k})}
\cdot q^{\mathbb{I}[c_k\neq y_k]} \cdot (1-q)^{1-\mathbb{I}[c_k\neq y_k]}
\bigg\}
\\&~~~~~~~~~~~~\wedge 
\bigg\{
(1-p)^{d(X_{\mathcal{E}_k}, Y_{\mathcal{E}_k})} \cdot p^{\Delta_k - d(X_{\mathcal{E}_k}, Y_{\mathcal{E}_k})}
\cdot (1-q)^{\mathbb{I}[c_k\neq y_k]} \cdot q^{1-\mathbb{I}[c_k\neq y_k]}
\bigg\}
\Bigg]
\\&=
\mathop{\sum_{X_{ij} \in \{-1,1\}}}_{(i,j)\in\mathcal{E}_k}
\sum_{c_k\in\{-1,1\}}
\Bigg[
\bigg\{
p^{d(X_{\mathcal{E}_k}, Y_{\mathcal{E}_k})} \cdot (1-p)^{\Delta_k - d(X_{\mathcal{E}_k}, Y_{\mathcal{E}_k})}
\cdot q^{\mathbb{I}[c_k\neq y_k]} \cdot (1-q)^{1-\mathbb{I}[c_k\neq y_k]}
\bigg\}
\\&~~~~~~~~~~~~~~~~~~~~~~~~~~~~~~~~\wedge 
\bigg\{
(1-p)^{d(X_{\mathcal{E}_k}, Y_{\mathcal{E}_k})} \cdot p^{\Delta_k - d(X_{\mathcal{E}_k}, Y_{\mathcal{E}_k})}
\cdot (1-q)^{\mathbb{I}[c_k\neq y_k]} \cdot q^{1-\mathbb{I}[c_k\neq y_k]}
\bigg\}
\Bigg]
\\&=
\mathop{\sum_{X_{ij} \in \{-1,1\}}}_{(i,j)\in\mathcal{E}_k}
\Bigg[
\bigg\{
p^{d(X_{\mathcal{E}_k}, Y_{\mathcal{E}_k})} \cdot (1-p)^{\Delta_k - d(X_{\mathcal{E}_k}, Y_{\mathcal{E}_k})}
\cdot q \bigg\}
\wedge 
\bigg\{
(1-p)^{d(X_{\mathcal{E}_k}, Y_{\mathcal{E}_k})} \cdot p^{\Delta_k - d(X_{\mathcal{E}_k}, Y_{\mathcal{E}_k})}
\cdot (1-q) \bigg\}
\Bigg]
\\&~~~~~~~~~~~~~~~~~~+
\Bigg[
\bigg\{
p^{d(X_{\mathcal{E}_k}, Y_{\mathcal{E}_k})} \cdot (1-p)^{\Delta_k - d(X_{\mathcal{E}_k}, Y_{\mathcal{E}_k})}
\cdot (1-q) \bigg\}
\wedge 
\bigg\{
(1-p)^{d(X_{\mathcal{E}_k}, Y_{\mathcal{E}_k})} \cdot p^{\Delta_k - d(X_{\mathcal{E}_k}, Y_{\mathcal{E}_k})}
\cdot q \bigg\}
\Bigg].
\end{align*}
Note that
$d(X_{\mathcal{E}_k}, Y_{\mathcal{E}_k})$
has a value between $0$ and $\Delta_k$, and for each $m\in[0, \Delta_k]$, there exist $\binom{\Delta_k}{m}$ different $X_{\mathcal{E}_k}$'s satisfying $d(X_{\mathcal{E}_k}, Y_{\mathcal{E}_k})=m$. Therefore, we can write that
\begin{align*}
\|P_{Z|\pmb{y}} \wedge P_{Z|\pmb{y}'} \| 
&=
\sum_{m=0}^{\Delta_k}
\binom{\Delta_k}{m}
\Bigg( \bigg[ \Big\{
p^{m} \cdot (1-p)^{\Delta_k - m}
\cdot q \Big\}
\wedge 
\Big\{
(1-p)^{m} \cdot p^{\Delta_k - m}
\cdot (1-q) \Big\}\bigg]
\\&~~~~~~~~~~~~~~~~~~~~~~+
\bigg[ \Big\{
p^{m} \cdot (1-p)^{\Delta_k -m}
\cdot (1-q) \Big\}
\wedge 
\Big\{
(1-p)^{m} \cdot p^{\Delta_k - m}
\cdot q \Big\} \bigg] \Bigg)
\end{align*}
and accordingly,
\begin{align*}
&\min\bigg\{\|P_{Z|\pmb{y}} \wedge P_{Z|\pmb{y}'} \| : \sum_{i=1}^n \mathbb{I}[y_i \neq y'_i] = 1,~ \pmb{y}, \pmb{y}' \in \mathcal{Y} \bigg\}
\\&=
\underset{1\leq k \leq n}{\min} 
\Bigg[
\sum_{m=0}^{\Delta_k}
\binom{\Delta_k}{m}
\Bigg(
\bigg[
\Big\{
p^{m} \cdot (1-p)^{\Delta_k - m}
\cdot q \Big\}
\wedge 
\Big\{
(1-p)^{m} \cdot p^{\Delta_k - m}
\cdot (1-q) \Big\}\bigg]
\\&~~~~~~~~~~~~~~~~~~~~~~~~~~~~~~~~+
\bigg[ \Big\{
p^{m} \cdot (1-p)^{\Delta_k -m}
\cdot (1-q) \Big\}
\wedge 
\Big\{
(1-p)^{m} \cdot p^{\Delta_k - m}
\cdot q \Big\} \bigg] \Bigg) \Bigg]
\\&=
\sum_{m=0}^{\Delta_{\max}}
\binom{\Delta_{\max}}{m}
\Bigg(
\bigg[
\Big\{
p^{m} \cdot (1-p)^{\Delta_{\max} - m}
\cdot q \Big\}
\wedge 
\Big\{
(1-p)^{m} \cdot p^{\Delta_{\max} - m}
\cdot (1-q) \Big\}\bigg]
\\&~~~~~~~~~~~~~~~~~~~~~~~~~+
\bigg[ \Big\{
p^{m} \cdot (1-p)^{\Delta_{\max} -m}
\cdot (1-q) \Big\}
\wedge 
\Big\{
(1-p)^{m} \cdot p^{\Delta_{\max} - m}
\cdot q \Big\} \bigg]
\Bigg) \Bigg].
\end{align*}
The last equality holds by Lemma \ref{lemma3}.

Finally, consider a joint probability distribution $P_0$ of $\pmb{y}^*$ and $Z$ where the marginal probability of $\pmb{y}^*$ is $p_{\pmb{y}^*}(\pmb{y}_0) = 1$ for some $\pmb{y}_0\in\mathcal{Y}$, and where $Z$ given $\pmb{y}^*$ follows the assumed conditional probability distribution.
Then we have that
\begin{align*}
\underset{P \in\mathcal{P}}{\sup}
\mathbb{P}_{ (\pmb{y}^*, Z) \sim P } \big[ \hat{\pmb{y}} (Z) \neq \pmb{y}^* \big] 
&\geq
\mathbb{P}_{ (\pmb{y}^*, Z) \sim P_0 } \big[ \hat{\pmb{y}} (Z) \neq \pmb{y}^* \big]     
\\&=
\mathbb{P}_{Z|\pmb{y}_0}\big[\hat{\pmb{y}}(Z) \neq \pmb{y}_0\big]
\end{align*}
which holds for any $\pmb{y}_0 \in \mathcal{Y}$.
Hence, we derive that
$$
\underset{ P \in\mathcal{P} }{\sup}
~\mathbb{P}_{ (\pmb{y}^*, Z) \sim P } \big[ \hat{\pmb{y}} (Z) \neq \pmb{y}^* \big] 
\geq 
\underset{\pmb{y}^*\in\mathcal{Y}}{\max}~\mathbb{P}_{Z|\pmb{y}^*}[\hat{\pmb{y}}(Z)\neq \pmb{y}^*]
$$
for any estimator $\hat{\pmb{y}}(\cdot)$.

Therefore, we have
\begin{align*}
\underset{\hat{\pmb{y}}:\mathcal{Z} \to \mathcal{Y}}{\inf}
~\underset{ P \in\mathcal{P} }{\sup}
~\mathbb{P}_{ (\pmb{y}^*, Z) \sim P } \big[ \hat{\pmb{y}} (Z) \neq \pmb{y}^* \big] 
&\geq
\frac{1}{2} \min\bigg\{\|P_{Z|\pmb{y}} \wedge P_{Z|\pmb{y}'} \| : \sum_{i=1}^n \mathbb{I}[y_i \neq y'_i] = 1,~ \pmb{y}, \pmb{y}' \in \mathcal{Y} \bigg\}
\\&=
\frac{1}{2}
\sum_{m=0}^{\Delta_{\max}}
\binom{\Delta_{\max}}{m}
\Bigg( \bigg[ \Big\{
p^{m} \cdot (1-p)^{\Delta_{\max} - m}
\cdot q \Big\}
\wedge 
\Big\{
(1-p)^{m} \cdot p^{\Delta_{\max} - m}
\cdot (1-q) \Big\}\bigg]
\\&~~~~~~~~~~~~~~~~~~~~~~~~~~~~+
\bigg[ \Big\{
p^{m} \cdot (1-p)^{\Delta_{\max} -m}
\cdot (1-q) \Big\}
\wedge 
\Big\{
(1-p)^{m} \cdot p^{\Delta_{\max} - m}
\cdot q \Big\} \bigg]
\Bigg) \Bigg]
\\&= f_2.
\end{align*}

\subsubsection{{Proof of Lemma \protect\ref{theorem2_lemma2}}}

As in the proof of Lemma \ref{theorem1_lemma2},
we can derive that
$$
\mathbb{I}(\pmb{y}, Z) = H(Z) -  |\mathcal{E}|H^*(p)
- n H^*(q)
$$
where $H(Z) = - \sum_{Z\in\mathcal{Z}} p(Z)\log p(Z)$ and $0 \leq H(Z) \leq (|\mathcal{E}|+n)\log 2$. 
Hence,
$$
\mathbb{I}(\pmb{y}, Z) \leq (|\mathcal{E}|+n)\log 2 -  |\mathcal{E}|H^*(p)
- n H^*(q)
$$
and by Fano's inequality, we have
\begin{align}
\mathbb{P}(\hat{\pmb{y}}\neq \pmb{y}^*)
&\geq
1- \frac{(|\mathcal{E}|+n)\log 2 -  |\mathcal{E}|H^*(p)
- n H^*(q) + \log 2}{n\log 2} \nonumber
\\&=
\frac{n-1}{n}-\frac{|\mathcal{E}|}{n}\bigg(1 - \frac{H^*(p)}{\log 2}\bigg)
-\bigg(1 - \frac{H^*(q)}{\log 2}\bigg)
= g_2.
\label{eq:minimax2_fano_proof1}
\end{align}

Next, we can write
\begin{align*}
p(Z)
&= \sum_{\pmb{y}\in\mathcal{Y}} p(Z|\pmb{y})p(\pmb{y})
= \frac{1}{2^n} \sum_{\pmb{y}\in\mathcal{Y}} p(X|\pmb{y})p(\pmb{c}|\pmb{y})
\\&= \frac{1}{2^n} \sum_{\pmb{y}\in\mathcal{Y}} 
\Big\{ p^{d(X_\mathcal{E}, Y_\mathcal{E})} (1-p)^{|\mathcal{E}|- d(X_\mathcal{E}, Y_\mathcal{E})}\Big\}\cdot
\Big\{ q^{d(\pmb{c}, \pmb{y})} (1-q)^{n- d(\pmb{c}, \pmb{y})}\Big\}
\\&= \frac{1}{2^n} (1-p)^{|\mathcal{E}|} (1-q)^n \sum_{\pmb{y}\in\mathcal{Y}} \bigg(\frac{p}{1-p}\bigg)^{d(X_\mathcal{E}, Y_\mathcal{E})}\cdot \bigg(\frac{q}{1-q}\bigg)^{d(\pmb{c}, \pmb{y})}.
\end{align*}
Since $\frac{p}{1-p}\leq 1$ and $\frac{q}{1-q}\leq 1$, we have the following inequalities:
\begin{align*}
p(Z)
&\leq \frac{1}{2^n} (1-p)^{|\mathcal{E}|} (1-q)^n \sum_{\pmb{y}\in\mathcal{Y}} \bigg(\frac{p}{1-p}\bigg)^{d(X_\mathcal{E}, Y_\mathcal{E})}
\leq \frac{1}{2^n} (1-p)^{|\mathcal{E}|} (1-q)^n \cdot 2^n = (1-p)^{|\mathcal{E}|} (1-q)^n,
\\
p(Z)
&\leq \frac{1}{2^n} (1-p)^{|\mathcal{E}|} (1-q)^n \sum_{\pmb{y}\in\mathcal{Y}} \bigg(\frac{q}{1-q}\bigg)^{d(\pmb{c}, \pmb{y})}
\leq \frac{1}{2^n} (1-p)^{|\mathcal{E}|} (1-q)^n \cdot 2^n = (1-p)^{|\mathcal{E}|} (1-q)^n.
\end{align*}
Then, if $(1-p)^{|\mathcal{E}|} (1-q)^n \leq \frac{1}{e}$ holds, we have
\begin{align*}
-p(Z)\log p(Z)
&\leq -\frac{1}{2^n} (1-p)^{|\mathcal{E}|} (1-q)^n \sum_{\pmb{y}\in\mathcal{Y}} \bigg(\frac{p}{1-p}\bigg)^{d(X_\mathcal{E}, Y_\mathcal{E})}\cdot
\log \Bigg[\frac{1}{2^n} (1-p)^{|\mathcal{E}|} (1-q)^n \sum_{\pmb{y}\in\mathcal{Y}} \bigg(\frac{p}{1-p}\bigg)^{d(X_\mathcal{E}, Y_\mathcal{E})}
\Bigg]
\\&=: f(X) ~~\text{and}
\\
-p(Z)\log p(Z)
&\leq -\frac{1}{2^n} (1-p)^{|\mathcal{E}|} (1-q)^n \sum_{\pmb{y}\in\mathcal{Y}} \bigg(\frac{q}{1-q}\bigg)^{d(\pmb{c}, \pmb{y})}
\cdot \log \Bigg[
\frac{1}{2^n} (1-p)^{|\mathcal{E}|} (1-q)^n \sum_{\pmb{y}\in\mathcal{Y}} \bigg(\frac{q}{1-q}\bigg)^{d(\pmb{c}, \pmb{y})}\Bigg]
\\&=: g(\pmb{c}),
\end{align*}
and consequently,
\begin{gather*}
-p(Z)\log p(Z)
\leq \frac{1}{2} f(X) + \frac{1}{2} g(\pmb{c})
\\
\Rightarrow H(Z) 
= - \sum_{Z\in\mathcal{Z}} p(Z)\log p(Z)
\leq \frac{1}{2}\cdot 2^n \sum_{X\in\mathcal{X}} f(X)
+ \frac{1}{2}\cdot 2^{|\mathcal{E}|} \sum_{\pmb{c}\in\mathcal{C}} g(\pmb{c}).
\end{gather*}

First, we can derive that
\begin{align*}
\sum_{X\in\mathcal{X}}f(X)
&\leq
-2^{n-1}\cdot\sum_{m=0}^{|\mathcal{E}|-n+1}
\binom{|\mathcal{E}|-n+1}{m}
\frac{1}{2^n} (1-p)^{|\mathcal{E}|} (1-q)^n
\tau(m)
\cdot
\log \bigg[\frac{1}{2^n} (1-p)^{|\mathcal{E}|} (1-q)^n
\tau(m)\bigg]
\\&=
-2^{|\mathcal{E}|} (1-p)^{|\mathcal{E}|} (1-q)^n \sum_{m=0}^{|\mathcal{E}|-n+1}
\binom{|\mathcal{E}|-n+1}{m}
\frac{1}{2^{|\mathcal{E}|-n+1}}
\cdot
\frac{\tau(m)}{2^n}
\bigg\{
|\mathcal{E}|\log(1-p) + n\log(1-q) + \log \frac{\tau(m)}{2^n}\bigg\}
\\&=
-\{2(1-p)\}^{|\mathcal{E}|} (1-q)^n \big\{
|\mathcal{E}|\log(1-p) + n\log(1-q) \big\} 
\mathbb{E}_B\bigg[\frac{\tau(B)}{2^n}\bigg]
\\&~~~~+
\{2(1-p)\}^{|\mathcal{E}|} (1-q)^n
\mathbb{E}_B\bigg[-\frac{\tau(B)}{2^n}\log\bigg(\frac{\tau(B)}{2^n}\bigg)\bigg]
\end{align*}
in a similar way to the proof of Lemma \ref{theorem1_lemma2}.
Note that 
$B\sim \text{Bin}(|\mathcal{E}|-n+1,~ \frac{1}{2})$.

Also, note that
$$
(1-q)^n \sum_{\pmb{y}\in\mathcal{Y}} \bigg(\frac{q}{1-q}\bigg)^{d(\pmb{c}, \pmb{y})}
= (1-q)^n \sum_{S\subseteq\mathcal{V}} \bigg(\frac{q}{1-q}\bigg)^{|S|}
= \sum_{k=0}^n \binom{n}{k} \bigg(\frac{q}{1-q}\bigg)^{k} (1-q)^n
= 1,
$$
that is,
$$
g(\pmb{c})
= -\frac{1}{2^n} (1-p)^{|\mathcal{E}|}
\cdot \log \bigg[
\frac{1}{2^n} (1-p)^{|\mathcal{E}|} \bigg].
$$

Therefore,
\begin{align*}
H(Z) 
&\leq \frac{1}{2}\cdot 2^n \sum_{X\in\mathcal{X}} f(X)
+ \frac{1}{2}\cdot 2^{|\mathcal{E}|} \sum_{\pmb{c}\in\mathcal{C}} g(\pmb{c})
\\&\leq
-\frac{1}{2}\{2(1-p)\}^{|\mathcal{E}|} \{2(1-q)\}^n \big\{
|\mathcal{E}|\log(1-p) + n\log(1-q) \big\} 
\mathbb{E}_B\bigg[\frac{\tau(B)}{2^n}\bigg]
\\&~~~+
\frac{1}{2}\{2(1-p)\}^{|\mathcal{E}|} \{2(1-q)\}^n
\mathbb{E}_B\bigg[-\frac{\tau(B)}{2^n}\log\bigg(\frac{\tau(B)}{2^n}\bigg)\bigg]
\\&~~~-
\frac{1}{2} \{2(1-p)\}^{|\mathcal{E}|}
\frac{1}{2^{n}}\log \bigg[
\frac{1}{2^n} (1-p)^{|\mathcal{E}|} \bigg]
\\&=
\{-|\mathcal{E}|\log(1-p) -n\log(1-q)\} \cdot \frac{1}{2}\{2(1-q)\}^{n}  \{2(1-p)\}^{|\mathcal{E}|} 
\mathbb{E}_B\bigg[\frac{\tau(B)}{2^n}\bigg]
\\&~~~+ \{-|\mathcal{E}|\log(1-p) +n\log 2\} \cdot\frac{1}{2}\{2(1-p)\}^{|\mathcal{E}|}\cdot \frac{1}{2^n}
\\&~~~+ \frac{1}{2}\{2(1-q)\}^{n}\{2(1-p)\}^{|\mathcal{E}|}
\mathbb{E}_B\bigg[-\frac{\tau(B)}{2^n}\log\bigg(\frac{\tau(B)}{2^n}\bigg)\bigg]
\\&=
\kappa_2
\end{align*}
when $(1-p)^{|\mathcal{E}|} (1-q)^n \leq \frac{1}{e}$.

Now, we can derive the upper bound of the mutual information as follows:
\begin{align*}
\mathbb{I}(\pmb{y}, Z) 
&= H(Z) -  |\mathcal{E}|H^*(p)
- n H^*(q)
\\&\leq \big\{(|\mathcal{E}|+n)\log 2 \wedge  \kappa_2 \big\}\cdot
\mathbb{I}\big[(1-p)^{|\mathcal{E}|} (1-q)^n \leq e^{-1}\big]
\\&~~~~+ (|\mathcal{E}|+n)\log 2\cdot \big( 1-\mathbb{I}\big[(1-p)^{|\mathcal{E}|} (1-q)^n \leq e^{-1}\big]\big)
-  |\mathcal{E}|H^*(p) - n H^*(q)
\\&=
\big\{\kappa_2-(|\mathcal{E}|+n)\log 2 \big\}\wedge 0 \cdot
\mathbb{I}\big[(1-p)^{|\mathcal{E}|} (1-q)^n \leq e^{-1}\big] + (|\mathcal{E}|+n)\log 2
- |\mathcal{E}|H^*(p) - n H^*(q)
\\&=
(|\mathcal{E}|+n)\log 2
- |\mathcal{E}|H^*(p) - n H^*(q)
- \big\{(|\mathcal{E}|+n)\log 2-\kappa_2 \big\}\vee 0 \cdot
\mathbb{I}\big[(1-p)^{|\mathcal{E}|} (1-q)^n \leq e^{-1}\big]
.
\end{align*}
Then, by Fano's inequality, we have
\begin{align}
\mathbb{P}(\hat{\pmb{y}}\neq \pmb{y}^*)
&\geq 1-\frac{(|\mathcal{E}|+n)\log 2
- |\mathcal{E}|H^*(p) - n H^*(q) + \log 2}{n\log 2} \nonumber
\\&~~~+\frac{\big\{(|\mathcal{E}|+n)\log 2-\kappa_2 \big\}\vee 0}{n\log 2}\cdot
\mathbb{I}\big[(1-p)^{|\mathcal{E}|} (1-q)^n \leq e^{-1}\big] \nonumber
\\&=g_2 +\frac{\big\{(|\mathcal{E}|+n)\log 2-\kappa_2 \big\}\vee 0}{n\log 2}\cdot
\mathbb{I}\big[(1-p)^{|\mathcal{E}|} (1-q)^n \leq e^{-1}\big] \nonumber
\\&=g_2^*.
\label{eq:minimax2_fano_proof2}
\end{align}

From (\ref{eq:minimax2_fano_proof1}) and (\ref{eq:minimax2_fano_proof2}), we obtain the desired result.

\subsection{Proof of Theorem \protect\ref{theorem4}}
\label{proof_theorem4}

As in the proof of Theorem \ref{theorem2}, we write $Y = Y_\mathcal{E}=[\pmb{y}\pmb{y}^T]_\mathcal{E}$ for each $\pmb{y}\in \mathcal{Y}$ and $Y^* = Y^*_\mathcal{E} = [\pmb{y}^*(\pmb{y}^*)^{T}]_\mathcal{E} = [-\pmb{y}^*(-\pmb{y}^*)^{T}]_\mathcal{E}$,
for simplicity.
Also, we define the following:
\begin{align*}
\Delta_{\pmb{c}}(\pmb{y}) 
&:= \alpha\pmb{c}^T \pmb{y}^* - \alpha\pmb{c}^T \pmb{y} 
= \alpha\langle \pmb{c}, \pmb{y}^* - \pmb{y} \rangle
= \alpha\langle \mathbb{E}[\pmb{c}], \pmb{y}^* - \pmb{y} \rangle
+ \alpha\langle \pmb{c}- \mathbb{E}[\pmb{c}], \pmb{y}^* - \pmb{y} \rangle
\\
\Delta(\pmb{y}) 
&:= 
\pmb{y}^{*T} X \pmb{y}^* + \alpha \pmb{c}^T \pmb{y}^* - \pmb{y}^T X \pmb{y} - \alpha \pmb{c}^T \pmb{y}
=\Delta_X(Y) + \Delta_{\pmb{c}}(\pmb{y}).
\end{align*}
Then our goal becomes to find the upper bound of the probability
$\mathbb{P}\big[\exists\pmb{y}\in\mathcal{Y}\backslash\{\pmb{y}^*\} ~\text{s.t.}~\Delta(\pmb{y}) \leq 0\big]$.

In a similar way to the proof of Theorem \ref{theorem2}, we can derive that
$$
\alpha\langle \mathbb{E}[\pmb{c}], \pmb{y}^* - \pmb{y} \rangle
= 2\alpha(1-2q)N(\pmb{y})
$$
where $N(\pmb{y}) := S_{\pmb{y}^*}(\pmb{y})=\{i\in\mathcal{V} : y_i^* \neq y_i\}$.
Also, $\alpha\langle \pmb{c}- \mathbb{E}[\pmb{c}], \pmb{y}^* - \pmb{y} \rangle
=
\alpha\langle \pmb{c}- \pmb{y}^* (1-2q), \pmb{y}^* - \pmb{y} \rangle$ can be represented by
\begin{align*}
\alpha\sum_{i\in\mathcal{V}} \big(c_i-y^*_i (1-2q)\big) \cdot(y^*_{i} - y_{i})
&=
\mathop{\sum_{i\in\mathcal{V}:}}_{y^*_i\neq y_i}
\alpha\big(c_i-y^*_i (1-2q)\big) \cdot(y^*_{i} - y_{i})
=:
T_{\pmb{c}}(\pmb{y})
\end{align*}
where $T_{\pmb{c}}(\pmb{y})$ is the summation of $N(\pmb{y})$ i.i.d. binary random variables which have mean zero and variance $16q(1-q)\alpha^2$ and are bounded by $4(1-q)\alpha$.

Now, we can find the upper bound of the probability $\mathbb{P}\big[\exists\pmb{y}\in\mathcal{Y}\backslash\{\pmb{y}^*\} ~\text{s.t.}~\Delta(\pmb{y}) \leq 0\big]$
by applying Bernstein inequality in two different ways. 

(1) Firstly, by applying Bernstein inequality on $T_{\pmb{c}}(\pmb{y})$, we can derive that
\begin{align*}
\mathbb{P} \big[\Delta_{\pmb{c}}(\pmb{y}) \leq 0 \big]
&=
\mathbb{P} \bigg[ -T_{\pmb{c}}(\pmb{y}) \geq 2\alpha(1-2q) N(\pmb{y}) \bigg]
\leq 
\exp \bigg[ - \frac{2\alpha^2(1-2q)^2 \{N(\pmb{y})\}^2 }{ 16q(1-q)\alpha^2N(\pmb{y}) + \frac{8}{3}(1-q)\alpha^2(1-2q)N(\pmb{y})} \bigg]
\\[0.5em]&= 
\exp \bigg[ - \frac{ (1-2q)^2N(\pmb{y}) }{ \frac{4}{3}(1-q) (1 + 4q)} \bigg]
= h_1\big(q, N(\pmb{y})\big) = h_1\big(q, S_{\pmb{y}^*}(\pmb{y})\big)
\end{align*}
for each $\pmb{y} \neq \pmb{y}^*$.

Note that for any $\pmb{y} \neq \pmb{y}^*,-\pmb{y}^*$,
\begin{align*}
\mathbb{P}\big[\Delta(\pmb{y}) \leq 0 ~\text{or}~ \Delta(-\pmb{y}) \leq 0 \big]
&= \mathbb{P}\big[\Delta_X(Y)+\Delta_{\pmb{c}}(\pmb{y}) \leq 0 ~\text{or}~ \Delta_X(Y)+\Delta_{\pmb{c}}(-\pmb{y}) \leq 0 \big]
\\&\leq \mathbb{P}\big[\Delta_X(Y) \leq 0 ~~\text{or}~~ \Delta_{\pmb{c}}(\pmb{y}) \leq 0 ~~\text{or}~~ \Delta_{\pmb{c}}(-\pmb{y}) \leq 0 \big]
\\&\leq \mathbb{P}\big[\Delta_X(Y) \leq 0 \big] +  \mathbb{P}\big[\Delta_{\pmb{c}}(\pmb{y}) \leq 0\big] +  \mathbb{P}\big[ \Delta_{\pmb{c}}(-\pmb{y}) \leq 0 \big],
\end{align*}
and for $\pmb{y} = -\pmb{y}^*$,
$
\mathbb{P}\big[\Delta(\pmb{y}) \leq 0 \big]
= \mathbb{P}\big[\Delta_{\pmb{c}}(-\pmb{y}^*) \leq 0\big].
$

Now, denote by $\mathcal{Y}^{\frac{1}{2}}$ the subset of $\mathcal{Y}$ which includes either $\pmb{y}$ or $-\pmb{y}$ for all $\pmb{y}\in\mathcal{Y}$.
Then, we can find the upper bound of the following probability:
\begin{align*}
&\mathbb{P}\big[\exists\pmb{y}\in\mathcal{Y}\backslash\{\pmb{y}^*\} ~\text{s.t.}~\Delta(\pmb{y}) \leq 0\big]
\\&\leq 
\mathbb{P}\big[\Delta(-\pmb{y}^*) \leq 0 \big] 
+ 
\sum_{\pmb{y}\in \mathcal{Y}^{\frac{1}{2}}} \mathbb{I}[\pmb{y} \neq \pmb{y}^*, -\pmb{y}^*] \cdot \mathbb{P}\big[\Delta(\pmb{y}) \leq 0 ~\text{or}~ \Delta(-\pmb{y}) \leq 0 \big]
\\&\leq 
\mathbb{P}\big[\Delta_{\pmb{c}}(-\pmb{y}^*) \leq 0 \big]
+ \sum_{\pmb{y}\in \mathcal{Y}^{\frac{1}{2}}} \mathbb{I}[\pmb{y} \neq \pmb{y}^*, -\pmb{y}^*]
\cdot \Big\{ \mathbb{P}\big[\Delta_X(Y) \leq 0 \big] +  \mathbb{P}\big[\Delta_{\pmb{c}}(\pmb{y}) \leq 0\big] +  \mathbb{P}\big[ \Delta_{\pmb{c}}(-\pmb{y}) \leq 0 \big] \Big\}
\\&= 
\mathbb{P}\big[\Delta_{\pmb{c}}(-\pmb{y}^*) \leq 0 \big]
+ \frac{1}{2}\sum_{\pmb{y}\in \mathcal{Y}} \mathbb{I}[\pmb{y} \neq \pmb{y}^*, -\pmb{y}^*] \cdot \mathbb{P}\big[\Delta_X(Y) \leq 0 \big]
+ 
\sum_{\pmb{y}\in \mathcal{Y}} \mathbb{I}[\pmb{y} \neq \pmb{y}^*, -\pmb{y}^*] \cdot \mathbb{P}\big[\Delta_{\pmb{c}}(\pmb{y}) \leq 0 \big]
\\&= 
\frac{1}{2}\sum_{\pmb{y}\in \mathcal{Y}} \mathbb{I}[\pmb{y} \neq \pmb{y}^*, -\pmb{y}^*] \cdot \mathbb{P}\big[\Delta_X(Y) \leq 0 \big]
+ 
\sum_{\pmb{y}\in \mathcal{Y}} \mathbb{I}[\pmb{y} \neq \pmb{y}^*] \cdot \mathbb{P}\big[\Delta_{\pmb{c}}(\pmb{y}) \leq 0 \big]
\\&\leq 
\frac{1}{2}\sum_{\pmb{y}\in \mathcal{Y}} \mathbb{I}[\pmb{y} \neq \pmb{y}^*, -\pmb{y}^*] \cdot 
h_1\big(p, \big|\mathcal{E}\big( S_{\pmb{y}^*}(\pmb{y}), S_{\pmb{y}^*}(\pmb{y})^c  \big)\big|\big)
+ \sum_{\pmb{y}\in \mathcal{Y}} \mathbb{I}[\pmb{y} \neq \pmb{y}^*] \cdot h_1\big(q, S_{\pmb{y}^*}(\pmb{y})\big)
\\&= 
\frac{1}{2}\sum_{\pmb{y}\in \mathcal{Y}} \mathbb{I}\big[|S_{\pmb{y}^*}(\pmb{y})|\neq 0 ~\text{or}~ n\big] \cdot h_1\big(p, \big|\mathcal{E}\big( S_{\pmb{y}^*}(\pmb{y}), S_{\pmb{y}^*}(\pmb{y})^c  \big)\big|\big)
+ \sum_{\pmb{y}\in \mathcal{Y}} \mathbb{I}[|S_{\pmb{y}^*}(\pmb{y})|\neq 0] \cdot h_1\big(q, S_{\pmb{y}^*}(\pmb{y})\big)
\\&= 
\frac{1}{2}\mathop{\sum_{S\subseteq \mathcal{V}:}}_{1\leq|S|\leq n-1} h_1\big(p, \big|\mathcal{E}\big( S, S^c  \big)\big|\big)
+ \mathop{\sum_{S\subseteq \mathcal{V}:}}_{|S|\geq 1} h_1\big(q, \big|S|\big)
\end{align*}
where the first inequality holds by the union bound and
the last equality holds by the Fact \ref{fact4}.

As shown in the proof of Theorem \ref{theorem2},
$$
\frac{1}{2}\mathop{\sum_{S\subseteq \mathcal{V}:}}_{1\leq|S|\leq n-1} h_1\big(p, \big|\mathcal{E}\big( S, S^c  \big)\big|\big)
\leq 
\sum_{k=1}^{\lfloor \frac{n}{2}\rfloor} \binom{n}{k} 
h_1(p, \phi_{\mathcal{G}}k),
$$
and we can easily derive that 
$$
\mathop{\sum_{S\subseteq \mathcal{V}:}}_{|S|\geq 1} h_1\big(q, \big|S|\big)
=
\sum_{k=1}^{n} \binom{n}{k} h_1\big(q, k\big).
$$
Therefore, 
\begin{align}
\label{theorem4_proof_eq1}
\mathbb{P}\big[\exists\pmb{y}\in\mathcal{Y}\backslash\{\pmb{y}^*\} ~\text{s.t.}~\Delta(\pmb{y}) \leq 0\big]
&\leq \sum_{k=1}^{\lfloor \frac{n}{2}\rfloor} \binom{n}{k} 
h_1(p, \phi_{\mathcal{G}}k)
+ \sum_{k=1}^{n} \binom{n}{k} h_1\big(q, k\big).
\end{align}

(2) Secondly, since $T_X(Y) + T_{\pmb{c}}(\pmb{y})$ is the summation of independent random variables which have zero mean and are bounded by $4(1-p)\vee 4(1-q)\alpha$, we can apply Bernstein inequality on $T_X(Y) + T_{\pmb{c}}(\pmb{y})$ and derive that
\begin{align*}
&\mathbb{P} \big[\Delta(\pmb{y}) \leq 0 \big]
=\mathbb{P} \bigg[ -T_X(Y) - T_{\pmb{c}}(\pmb{y}) \geq 2(1-2p) N(Y) + 2\alpha(1-2q)N(\pmb{y}) \bigg]
\\[0.5em]&\leq 
\exp \bigg[-\frac{2\big\{(1-2p)N(Y) + \alpha(1-2q) N(\pmb{y}) \big\}^2 }
{ 16p(1-p)N(Y) + 16q(1-q)\alpha^2N(\pmb{y}) + \frac{8}{3}[(1-p)\vee (1-q)\alpha]\cdot\big\{(1-2p)N(Y)+\alpha(1-2q)N(\pmb{y})\big\}} \bigg]
\\[0.5em]&= 
\exp \bigg[-\frac{\big\{(1-2p)N(Y) + \alpha(1-2q) N(\pmb{y}) \big\}^2}
{ 8p(1-p)N(Y) + 8q(1-q)\alpha^2N(\pmb{y}) + \frac{4}{3}[(1-p)\vee (1-q)\alpha]\cdot\big\{(1-2p)N(Y)+\alpha(1-2q)N(\pmb{y})\big\}} \bigg]
\\[0.5em]&=: 
h_2\big(N(Y), N(\pmb{y})\big)
\end{align*}
for each $\pmb{y} \neq \pmb{y}^*$.
Then, we can find the upper bound of the following probability:
\begin{align*}
\mathbb{P}\big[\exists\pmb{y}\in\mathcal{Y}\backslash\{\pmb{y}^*\} ~\text{s.t.}~\Delta(\pmb{y}) \leq 0\big]
&\leq 
\sum_{\pmb{y}\in{\mathcal{Y}}} \mathbb{I}[\pmb{y} \neq \pmb{y}^*]\cdot \mathbb{P}\big[\Delta(\pmb{y}) \leq 0 \big]
\\&\leq 
\sum_{\pmb{y}\in{\mathcal{Y}}} \mathbb{I}[\pmb{y} \neq \pmb{y}^*]\cdot 
h_2\big(N(Y), N(\pmb{y})\big)
\\&=
\sum_{\pmb{y}\in{\mathcal{Y}}} \mathbb{I}\big[|S_{\pmb{y}^*}(\pmb{y})|\neq 0\big]\cdot 
h_2\big(\big|\mathcal{E}\big( S_{\pmb{y}^*}(\pmb{y}), S_{\pmb{y}^*}(\pmb{y})^c  \big)\big|, \big|S_{\pmb{y}^*}(\pmb{y}) \big|\big)
\\&= 
\mathop{\sum_{S\subseteq \mathcal{V}:}}_{|S|\geq 1} 
h_2\big(|\mathcal{E}\big( S, S^c  \big)|, |S|\big)
\end{align*}
where the first inequality holds by the union bound and
the last equality holds by the Fact \ref{fact4}.

Since $h_2(\cdot, |S|)$ and $h_2(\cdot, n-|S|)$ are decreasing by Lemma \ref{theorem4_lemma}, which will be presented immediately, we can derive that
\begin{align*}
\mathop{\sum_{S\subseteq \mathcal{V}:}}_{|S|\geq 1} 
h_2\big(|\mathcal{E}\big( S, S^c  \big)|, |S|\big)
&=    
\mathop{\sum_{S\subseteq \mathcal{V}:}}_{1\leq|S|\leq \lfloor \frac{n}{2}\rfloor} 
h_2\big(|\mathcal{E}\big( S, S^c  \big)|, |S|\big)
+
\mathop{\sum_{S\subseteq \mathcal{V}:}}_{|S|\leq \lfloor \frac{n}{2}\rfloor} 
h_2\big(|\mathcal{E}\big( S, S^c  \big)|, n-|S|\big)
\\&\leq
\mathop{\sum_{S\subseteq \mathcal{V}:}}_{1\leq|S|\leq \lfloor \frac{n}{2}\rfloor} 
h_2\big(\phi_{\mathcal{G}}|S|, |S|\big)
+
\mathop{\sum_{S\subseteq \mathcal{V}:}}_{|S|\leq \lfloor \frac{n}{2}\rfloor} 
h_2\big(\phi_{\mathcal{G}}|S|, n-|S|\big)
\\&=
\sum_{k=1}^{\lfloor \frac{n}{2}\rfloor} 
\binom{n}{k} 
h_2\big(\phi_{\mathcal{G}}k, k\big)
+
\sum_{k=0}^{\lfloor \frac{n}{2}\rfloor} 
\binom{n}{k} 
h_2\big(\phi_{\mathcal{G}}k, n-k\big).
\end{align*}
Consequently,
\begin{align}
\label{theorem4_proof_eq2}
\mathbb{P}\big[\exists\pmb{y}\in\mathcal{Y}\backslash\{\pmb{y}^*\} ~\text{s.t.}~\Delta(\pmb{y}) \leq 0\big]
&\leq \sum_{k=1}^{\lfloor \frac{n}{2}\rfloor} 
\binom{n}{k} 
h_2\big(\phi_{\mathcal{G}}k, k\big)
+
\sum_{k=0}^{\lfloor \frac{n}{2}\rfloor} 
\binom{n}{k} 
h_2\big(\phi_{\mathcal{G}}k, n-k\big).
\end{align}

From (\ref{theorem4_proof_eq1}) and (\ref{theorem4_proof_eq2}), we can obtain the lower bound of the probability of success of the MLE approach as follows:
\begin{align*}
&\mathbb{P}\big[\Delta(\pmb{y}) > 0~\text{for all}~ \pmb{y}\in\mathcal{Y}\backslash\{\pmb{y}^*\}\big]
=
1- \mathbb{P}\big[\exists\pmb{y}\in\mathcal{Y}\backslash\{\pmb{y}^*\} ~\text{s.t.}~\Delta(\pmb{y}) \leq 0\big]
\\&\geq 1 - 
\min\Bigg\{
\sum_{k=1}^{\lfloor \frac{n}{2}\rfloor} \binom{n}{k} 
h_1(p, \phi_{\mathcal{G}}k)
+ \sum_{k=1}^{n} \binom{n}{k} h_1\big(q, k\big),~~
\sum_{k=1}^{\lfloor \frac{n}{2}\rfloor} 
\binom{n}{k} 
h_2\big(\phi_{\mathcal{G}}k, k\big)
+
\sum_{k=0}^{\lfloor \frac{n}{2}\rfloor} 
\binom{n}{k} 
h_2\big(\phi_{\mathcal{G}}k, n-k\big)
\Bigg\}.
\end{align*}

\begin{lemma}
\label{theorem4_lemma}
For any $p, q \in (0, \frac{1}{2})$ and $w> 0$,~
$h_2(z, w)$ decreases as $z$ increases on $[0,\infty)$. 
Likewise, for any $p, q \in (0, \frac{1}{2})$ and $z> 0$,~
$h_2(z, w)$ decreases as $w$ increases on $[0,\infty)$.
\end{lemma}

\begin{proof}
First, when $a, b, c, d > 0$, $\exp\big[-\frac{(az+b)^2}{cz+d}\big]$ decreases as $z$ increases on $[0, \infty)$ if $\frac{b}{a} - \frac{2d}{c}\leq 0$. Note that
\begin{align*}
&\frac{\alpha(1-2q)w}{(1-2p)}
-2\cdot\frac{8q(1-q)\alpha^2w + \frac{4}{3}[(1-p)\vee (1-q)\alpha]\cdot\alpha(1-2q)w}
{8p(1-p) + \frac{4}{3}[(1-p)\vee (1-q)\alpha]\cdot(1-2p)}
\\&=
\alpha w \frac{(1-2q)}{(1-2p)} \cdot \Bigg[1-2\cdot
\frac{\frac{6q(1-q)\alpha}{1-2q} + [(1-p)\vee (1-q)\alpha]}
{\frac{6p(1-p)}{1-2p} + [(1-p)\vee (1-q)\alpha]}
\Bigg]
\leq 0,
\end{align*}
which is shown numerically (see Figure \ref{fig:r_function}.)
Therefore, $h_2(z, w)$ decreases as $z$ increases on $[0,\infty)$.
The same argument holds with respect to $w$ from symmetry.

\begin{figure}[ht]
\vskip 0.2in
\begin{center}
\centerline{\includegraphics[width=3in]{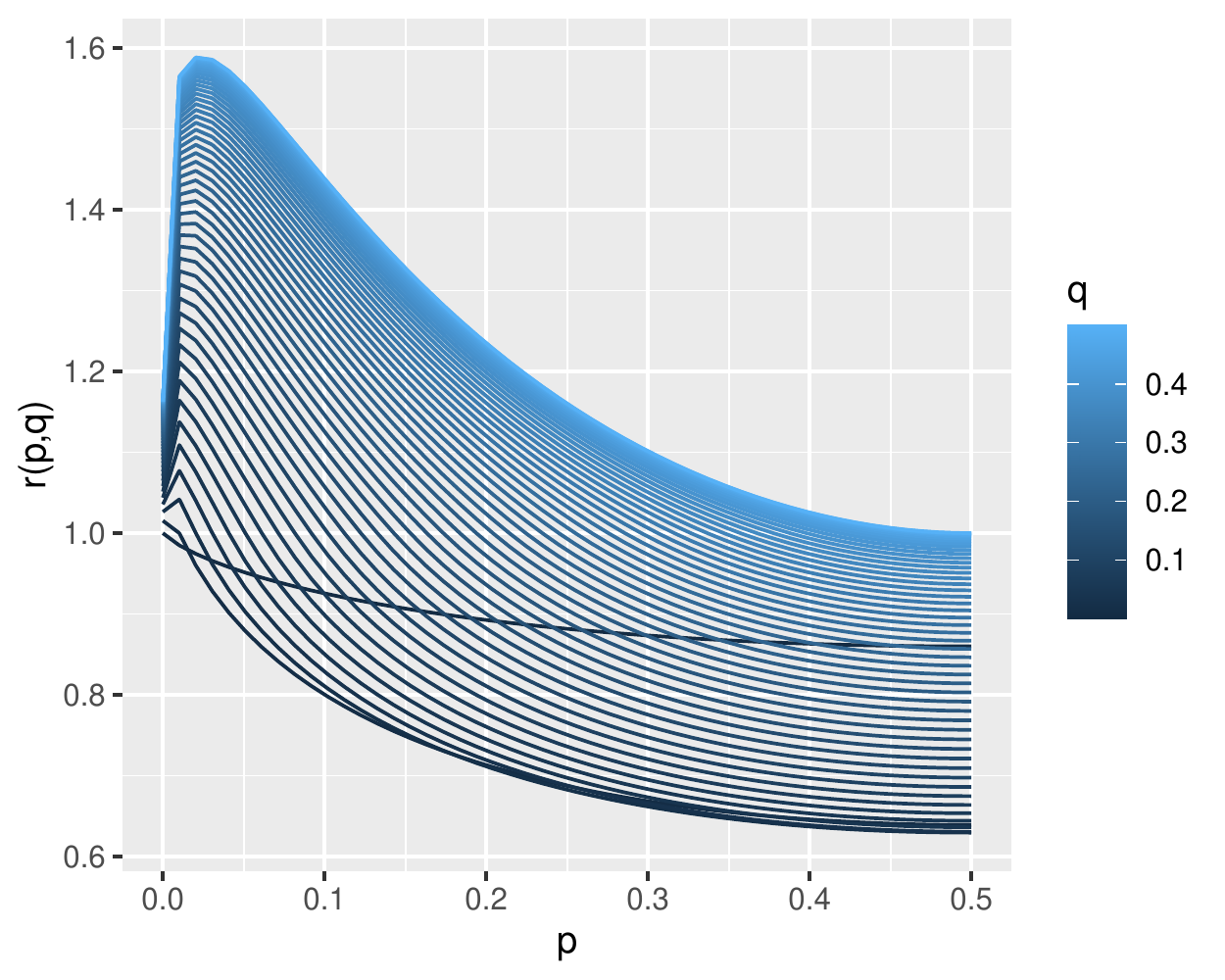}}
\caption{
Graph of the function $r(p,q):=\frac{\frac{6q(1-q)\alpha}{1-2q} + [(1-p)\vee (1-q)\alpha]}
{\frac{6p(1-p)}{1-2p} + [(1-p)\vee (1-q)\alpha]}$.~ It is always greater than $\frac{1}{2}$ when $p, q \in (0, \frac{1}{2})$.
}
\label{fig:r_function}
\end{center}
\vskip -0.2in
\end{figure}

\end{proof}

\subsection{Detailed Description of Comparison to Tractable Algorithm in Section \protect\ref{subsec:mle2}}
\label{proof_comparison2}

From Lemma \ref{theorem4_lemma}, we can obtain
\begin{align*}
\sum_{k=0}^{\lfloor \frac{n}{2}\rfloor} 
\binom{n}{k} 
h_2(\phi_{\mathcal{G}}k, n-k)
\leq
h_2(0, n) + \sum_{k=1}^{\lfloor \frac{n}{2}\rfloor} 
\binom{n}{k} 
h_2(\phi_{\mathcal{G}}k, k).
\end{align*}
Also, if $p\approx q$, we have
\begin{align*}
h_2(0, n)
= 
e^{-\frac{\{\alpha(1-2q)\}^2n}
{ 8q(1-q)\alpha^2 + \frac{4}{3}[(1-p)\vee (1-q)\alpha]\cdot\{\alpha(1-2q)\}}}
\approx
e^{-\frac{(1-2q)^2n}
{ 8q(1-q) + \frac{4}{3} (1-q)(1-2q)}}
=
e^{-\frac{(1-2q)^2n}{\frac{4}{3}(1-q) (1 + 4q)}}.
\end{align*}
Let $r 
= e^{-\frac{\{(1-2p)\phi_{\mathcal{G}} + \alpha(1-2q) \}^2}
{ 8p(1-p)\phi_{\mathcal{G}} + 8q(1-q)\alpha^2 + \frac{4}{3}[(1-p)\vee (1-q)\alpha]\cdot\{(1-2p)\phi_{\mathcal{G}}+\alpha(1-2q)\}}}$. 
Then we have
\begin{align*}
\sum_{k=1}^{\lfloor \frac{n}{2}\rfloor} 
\binom{n}{k} 
h_2(\phi_{\mathcal{G}}k, k)
+
\sum_{k=0}^{\lfloor \frac{n}{2}\rfloor} 
\binom{n}{k} 
h_2(\phi_{\mathcal{G}}k, n-k)
&\leq
2\cdot \sum_{k=1}^{\lfloor \frac{n}{2}\rfloor} 
\binom{n}{k} 
h_2(\phi_{\mathcal{G}}k, k)
+
e^{-\frac{(1-2q)^2n}{\frac{4}{3}(1-q) (1 + 4q)}} \\
&<
2 \cdot \{ (r+1)^n -1 \} + e^{-\frac{(1-2q)^2n}{\frac{4}{3}(1-q) (1 + 4q)}} \\
&= 2 \cdot \{ (r+1)^n -1 \} + \epsilon_2^{\frac{3}{2(1-q) (1 + 4q)}},
\end{align*}
where the second inequality holds as in Appendix \ref{proof_comparison1}.
We can check that $0.96 \leq \frac{3}{2(1-q) (1 + 4q)} \leq \frac{3}{2}$ for $q\in[0, \frac{1}{2}]$, and especially in the case that $q\leq \frac{1}{4}$, we can see that $\frac{3}{2(1-q) (1 + 4q)} \geq 1$ holds, i.e., $\epsilon_2^{\frac{3}{2(1-q) (1 + 4q)}}$ decays faster than $\epsilon_2$.

Next, if $\phi_{\mathcal{G}}=\Omega(n)$, then $r \underset{n\rightarrow\infty}{\longrightarrow} 0$. 
From this fact, we can derive that
$2 \cdot \{ (r+1)^n -1 \} \approx 2 (e^{nr} - 1)$ for sufficiently large $n$.
Let $A := 2 \cdot \{ (r+1)^n -1 \}$.
Then we have $nr \approx \log(\frac{A}{2} + 1)$.
Since $A\underset{n\rightarrow\infty}{\longrightarrow} 0$ when $\phi_{\mathcal{G}}=\Omega(n)$, there exists a positive constant $C$ such that
$nr \approx \log(\frac{A}{2} + 1) \geq C A$ holds
for sufficiently large $n$.

Also, by Lemma \ref{theorem4_lemma} and the condition that $p\approx q$, we have
$$
r
\leq
e^{-\frac{\{(1-2p)\phi_{\mathcal{G}} \}^2}
{ 8p(1-p)\phi_{\mathcal{G}} + \frac{4}{3}[(1-p)\vee (1-q)\alpha]\cdot\{(1-2p)\phi_{\mathcal{G}}\}}}
\approx
e^{-\frac{\{(1-2p)\phi_{\mathcal{G}} \}^2}
{ 8p(1-p)\phi_{\mathcal{G}} + \frac{4}{3} (1-q)\cdot\{(1-2p)\phi_{\mathcal{G}}\}}}
=
e^{-\frac{3(1-2p)^2\phi_{\mathcal{G}}}{4(1-p) (1 + 4p)}}.
$$
Hence,
$$
nr
\leq n\cdot e^{-\frac{3(1-2p)^2\phi_{\mathcal{G}}}{4(1-p) (1 + 4p)}}
=\frac{\epsilon_1}{2} \exp(-C_{n,p} \phi_{\mathcal{G}})
$$
where
$C_{n,p} = \frac{3(1-2p)^2}{4(1-p) (1 + 4p)} - \frac{3(1-2p)^2 \phi_{\mathcal{G}}^3}{1536\Delta_{\max}^3 p(1-p)+32(1-2p)(1-p)\phi_{\mathcal{G}}^2\Delta_{\max}}
\geq \frac{21(1-2p)^2}{32(1-p)(1+4p)}$
as shown in Appendix \ref{proof_comparison1}.
Therefore, we have
$$
A \leq C' \cdot \epsilon_1 \exp\bigg(-\frac{C''(1-2p)^2n}{(1-p)(1+4p)}\bigg)
$$
for sufficiently large $n$ and some positive constants $C'$ and $C''$,
which implies that $A = 2 \cdot \{ (r+1)^n -1 \}$ decays much faster than $\epsilon_1$.

Consequently, $\sum_{k=1}^{\lfloor \frac{n}{2}\rfloor} 
\binom{n}{k} 
h_2(\phi_{\mathcal{G}}k, k)
+
\sum_{k=0}^{\lfloor \frac{n}{2}\rfloor} 
\binom{n}{k} 
h_2(\phi_{\mathcal{G}}k, n-k)$ decays much faster than $\epsilon_1+ \epsilon_2$
if $p\approx q \leq \frac{1}{4}$ and $\phi_{\mathcal{G}}=\Omega(n)$.

\subsection{Proof of Corollary \protect\ref{corollary4}}

Let $r 
= e^{-\frac{\{(1-2p)\phi_{\mathcal{G}} + \alpha(1-2q) \}^2}
{ 8p(1-p)\phi_{\mathcal{G}} + 8q(1-q)\alpha^2 + \frac{4}{3}[(1-p)\vee (1-q)\alpha]\cdot\{(1-2p)\phi_{\mathcal{G}}+\alpha(1-2q)\}}}$.\\
If the condition $\frac{\{(1-2p)\phi_{\mathcal{G}} + \alpha(1-2q) \}^2}
{ 6p(1-p)\phi_{\mathcal{G}} + 6q(1-q)\alpha^2 + [(1-p)\vee (1-q)\alpha]\cdot\{(1-2p)\phi_{\mathcal{G}}+\alpha(1-2q)\}} \geq \frac{8}{3}\log n$ holds, we can derive that
\begin{align*}
\sum_{k=1}^{\lfloor \frac{n}{2}\rfloor} 
\binom{n}{k} 
h_2(\phi_{\mathcal{G}}k, k)
<
(r+1)^n -1
\leq
2n^{-1}
\end{align*}
exactly the same way as in the proof of Corollary \ref{corollary2}.

From Lemma \ref{theorem4_lemma}, we can obtain
\begin{align*}
\sum_{k=0}^{\lfloor \frac{n}{2}\rfloor} 
\binom{n}{k} 
h_2(\phi_{\mathcal{G}}k, n-k)
\leq
h_2(0, n) + \sum_{k=1}^{\lfloor \frac{n}{2}\rfloor} 
\binom{n}{k} 
h_2(\phi_{\mathcal{G}}k, k)
\end{align*}
where
\begin{align*}
h_2(0, n)
= 
e^{-\frac{\{\alpha(1-2q)\}^2n}
{ 8q(1-q)\alpha^2 + \frac{4}{3}[(1-p)\vee (1-q)\alpha]\cdot\{\alpha(1-2q)\}}}.
\end{align*}
It is easy to check that 
if the condition $\frac{\alpha\{(1-2q)\}^2n}
{ 6q(1-q)\alpha + \big[(1-p)\vee (1-q)\alpha\big]\cdot(1-2q)}
\geq \frac{4}{3}\log n$ holds,
$h_2(0, n) \leq n^{-1}$.

Consequently, we have
\begin{align*}
\sum_{k=1}^{\lfloor \frac{n}{2}\rfloor} 
\binom{n}{k} 
h_2(\phi_{\mathcal{G}}k, k)
+
\sum_{k=0}^{\lfloor \frac{n}{2}\rfloor} 
\binom{n}{k} 
h_2(\phi_{\mathcal{G}}k, n-k)
<
5n^{-1}
\end{align*}
under the assumptions.
Therefore,
$$
1 - \sum_{k=1}^{\lfloor \frac{n}{2}\rfloor} 
\binom{n}{k} 
h_2(\phi_{\mathcal{G}}k, k)
-
\sum_{k=0}^{\lfloor \frac{n}{2}\rfloor} 
\binom{n}{k} 
h_2(\phi_{\mathcal{G}}k, n-k)
>
1-5n^{-1}.
$$

\subsection{Illustration of Bounds of Probabilities for Additional Examples of Graphs in Section \protect\ref{sec:case_that_edge_and_node}}
\label{appendix_illustration2}

\begin{figure}[ht]
\vskip 0.2in
\begin{center}
\centerline{\includegraphics[width=\columnwidth]{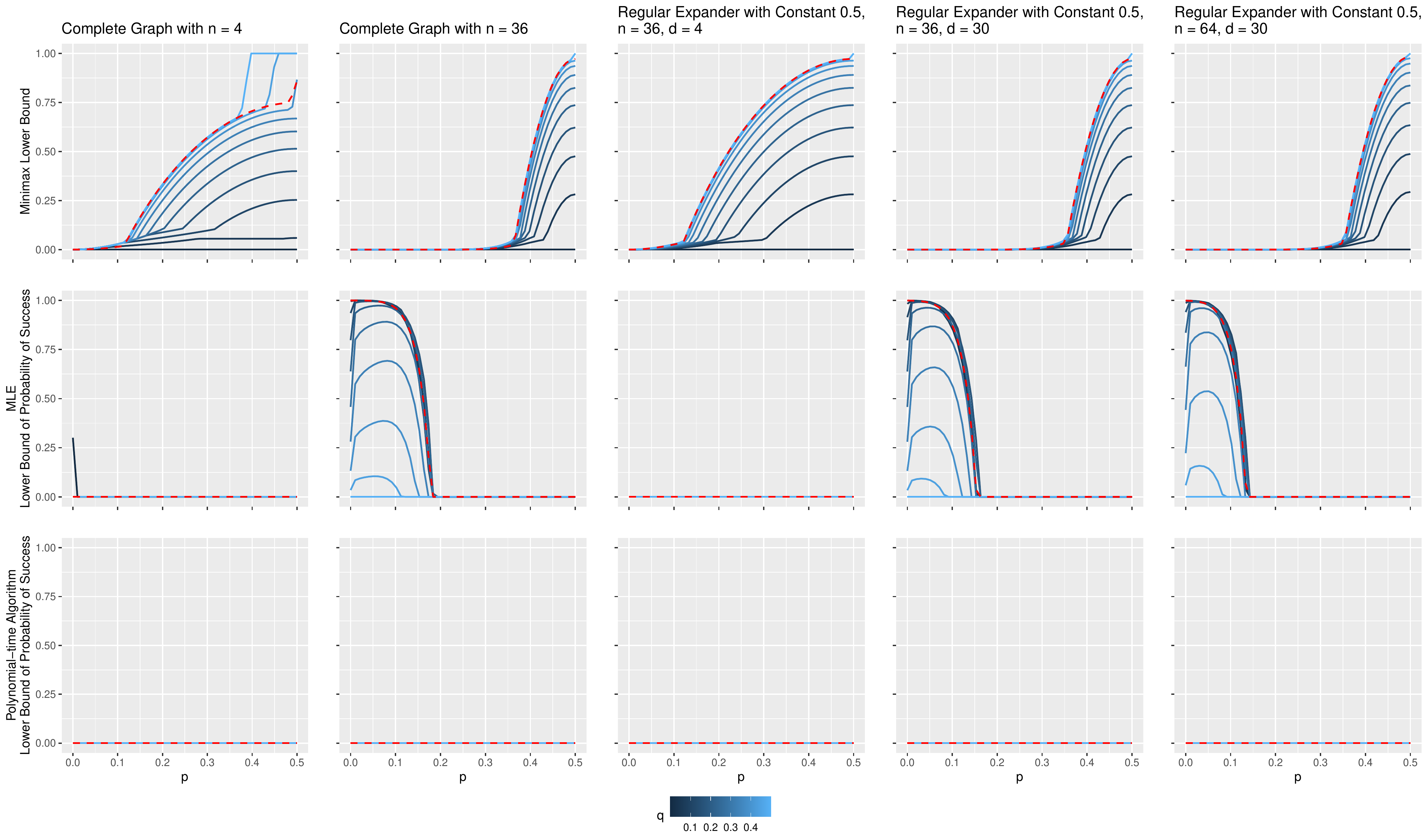}}
\caption{
Minimax lower bound (top), lower bound of the probability of success of the MLE algorithm (middle), and lower bound of the probability of success of the polynomial-time algorithm (bottom) for complete graphs and regular expanders.
}
\label{fig:appendix_illustration3}
\end{center}
\vskip -0.2in
\end{figure}

\begin{figure}[ht]
\vskip 0.2in
\begin{center}
\centerline{\includegraphics[width=6in]{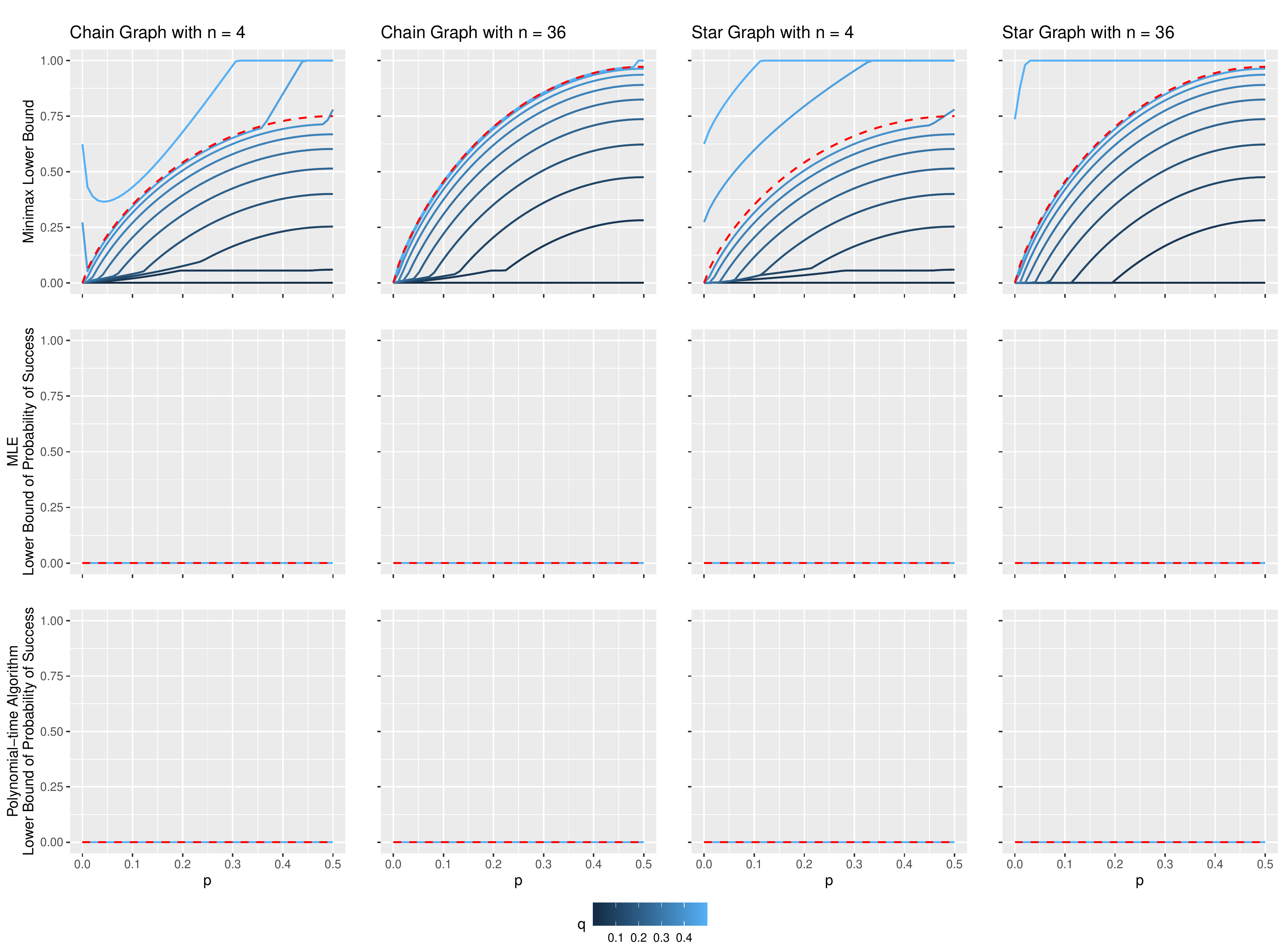}}
\caption{
Minimax lower bound (top), lower bound of the probability of success of the MLE algorithm (middle), and lower bound of the probability of success of the polynomial-time algorithm (bottom) for chain graphs and star graphs.
}
\label{fig:appendix_illustration4}
\end{center}
\vskip -0.2in
\end{figure}

\end{document}